\documentclass{article}

\pdfoutput=1

\usepackage{microtype}
\usepackage{graphicx}
\usepackage{subfigure}
\usepackage{booktabs} 

\usepackage{hyperref}
\usepackage{url}            
\usepackage{booktabs}       
\usepackage{amsfonts}       
\usepackage{nicefrac}       
\usepackage{microtype}      
\usepackage{mkolar_definitions}
\usepackage{xspace}
\usepackage{amsmath}
\usepackage{algorithm}
\usepackage{algorithmic}
\usepackage{color}
\usepackage{enumitem}
\usepackage{comment}
\usepackage{bm}
\usepackage{tablefootnote}

\usepackage{csquotes}
\newtheorem*{theorem*}{Theorem}
\theoremstyle{definition}
\newtheorem*{runex*}{Running Example}
\newtheorem*{ex*}{Example}
\newtheorem*{remark*}{Remark}

\newcommand{\dkl}{D_{\mathrm{KL}}}

\usepackage[accepted]{icml2020}

\icmltitlerunning{Provably Efficient Model-based Policy Adaptation}

\begin{document}

\twocolumn[
\icmltitle{Provably Efficient Model-based Policy Adaptation}

\begin{icmlauthorlist}
\icmlauthor{Yuda Song}{ucsd}
\icmlauthor{Aditi Mavalankar}{ucsd}
\icmlauthor{Wen Sun}{microsoft}
\icmlauthor{Sicun Gao}{ucsd}
\end{icmlauthorlist}

\icmlaffiliation{ucsd}{Department of Computer Science and Engineering, University of California, San Diego, La Jolla, USA}
\icmlaffiliation{microsoft}{Microsoft Research NYC, New York , USA}

\icmlcorrespondingauthor{Yuda Song}{yus167@ucsd.edu}

\icmlkeywords{Reinforcement Learning, Machine Learning, ICML}

\vskip 0.3in
]
\printAffiliationsAndNotice{}

\begin{abstract}
The high sample complexity of reinforcement learning challenges its use in practice. A promising approach is to quickly adapt pre-trained policies to new environments. Existing methods for this policy adaptation problem typically rely on domain randomization and meta-learning, by sampling from some distribution of target environments during pre-training, and thus face difficulty on out-of-distribution target environments. We propose new model-based mechanisms that are able to make online adaptation in unseen target environments, by combining ideas from no-regret online learning and adaptive control. We prove that the approach learns policies in the target environment that can recover trajectories from the source environment, and establish the rate of convergence in general settings. We demonstrate the benefits of our approach for policy adaptation in a diverse set of continuous control tasks, achieving the performance of state-of-the-art methods with much lower sample complexity. Our project website, including code, can be found at \url{https://yudasong.github.io/PADA}.
\end{abstract}

\section{Introduction}
\label{sec:intro}
Deep Reinforcement Learning (RL) methods typically require a very large number of interactions with environments, making them difficult to be used on practical systems~\cite{tan2018sim}. A promising direction is to adapt policies trained in one environment to similar but unseen environments, such as from simulation to real robots. Existing approaches for policy adaptation mostly focus on pre-training the policies to be robust to predefined distributions of disturbances in the environment, by increasing the sample diversity during training~\cite{peng2018sim,tobin2017domain,mordatch2015ensemble}, or meta-learn policies or models that can be quickly adapted to in-distribution environments~\cite{finn2017model,nagabandi2019learning,nagabandi2018deep,yu2018one}. 
A key assumption for these approaches is that the distribution of the target environments is known, and that it can be efficiently sampled during training. 
On out-of-distribution target environments, these methods typically do not deliver good performance, reflecting common challenges in generalization~\cite{Na2020Learning}. 
If we observe how humans and animals adapt to environment changes, clearly there is an online adaptation process in addition to memorization~\cite{staddon2016adaptive}. 
We can quickly learn to walk with a slightly injured leg even if we have not experienced the situation. We draw experiences from normal walking and adapt our actions online, based on how their effects differ from what we are familiar with in normal settings. 
Indeed, this intuition has recently led to practical approaches for policy adaptation. The work in \cite{christiano2016transfer} uses a pre-trained policy and model of the training environment, and learns an inverse dynamics model from scratch in the new environment by imitating the behaviors of the pre-trained policy. However, it does not involve mechanisms for actively reducing the divergence between the state trajectories of the two environments, which leads to inefficiency and distribution drifting, and does not fully capture the intuition above. The work in \cite{zhu2018reinforcement} uses Generative Adversarial Imitation Learning (GAIL)~\cite{ho2016generative} to imitate the source trajectories in the new environment, by adding the GAIL discriminator to the reward to reduce divergence, but relies on generic policy optimization methods with high sample complexity. In general, these recent approaches show the feasibility of policy adaptation, but are not designed for optimizing sample efficiency. There has been no theoretical analysis of whether policy adaptation methods can converge in general, or their benefits in terms of sample complexity. 

In this paper, we propose a new model-based approach for the policy adaptation problem that focuses on efficiency with theoretical justification. In our approach, the agent attempts to predict the effects of its actions based on a model of the training environment, and then adapts the actions to minimize the divergence between the state trajectories in the new (target) environment and in the training (source) environment. This is achieved by iterating between two steps: a modeling step learns the divergence between the source environment and the target environment, and a planning step that uses the divergence model to plan actions to reduce the divergence over time. Under the assumption that the target environment is close to the source environment, the divergence modeling and policy adaption can both be done locally and efficiently. We give the first theoretical analysis of policy adaptation by establishing the rate of convergence of our approaches under general settings. Our methods combine techniques from model-based RL~\cite{DBLP:journals/corr/abs-1907-02057} and no-regret online learning~\cite{ross2011reduction}. We demonstrate that the approach is empirically efficient in comparison to the state-of-the-art approaches~\cite{christiano2016transfer,zhu2018reinforcement}. 
The idea of recovering state trajectories from the source environment in the target environment suggests a strong connection between policy adaptation and imitation learning, such as Learning from Observation (LfO)~\cite{torabi2018behavioral, torabi2019imitation,sun2019provably, yang2019imitation}. A key difference is that in policy adaptation, the connection between the source and target environments and their difference provide both new challenges and opportunities for more efficient learning. By actively modeling the divergence between the source and target environments, the agent can achieve good performance in new environments by only making local changes to the source policies and models. On the other hand, because of the difference in the dynamics and the action spaces, it is not enough to merely imitate the experts~\cite{bain1999framework, ross2011reduction,sun2017deeply}. 
Traditionally, adaptive control theory~\cite{aastrom1983theory} studies how to adapt to disturbances by stabilizing the error dynamics. Existing work in adaptive control assumes closed-form dynamics and does not apply to the deep RL setting~\cite{nagabandi2019learning}. In comparison to domain randomization and meta-learning approaches, our proposed approach does not require sampling of source environments during pre-training, and makes it possible to adapt in out-of-distribution environments. Note that the two approaches are complementary, and we demonstrate in experiments that our methods can be used in conjunction with domain randomization and meta-learning to achieve the best results.

\begin{figure*}[t!]

    \begin{tabular}{p{0.02\linewidth}p{0.17\linewidth}p{0.17\linewidth}p{0.17\linewidth}p{0.17\linewidth}p{0.17\linewidth}}
    (a)&
    \includegraphics[width=0.8\linewidth]{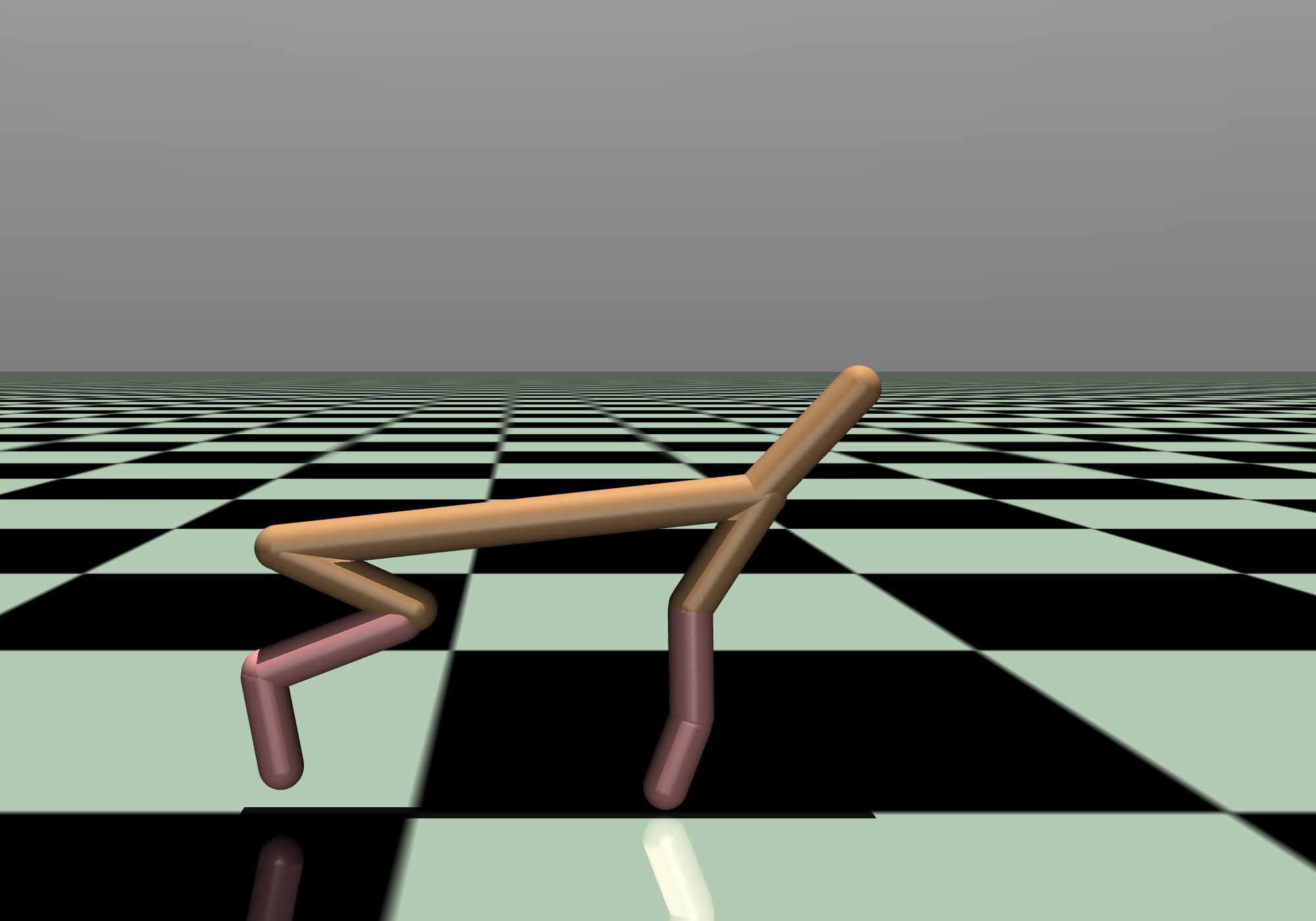}&
    \includegraphics[width=0.8\linewidth]{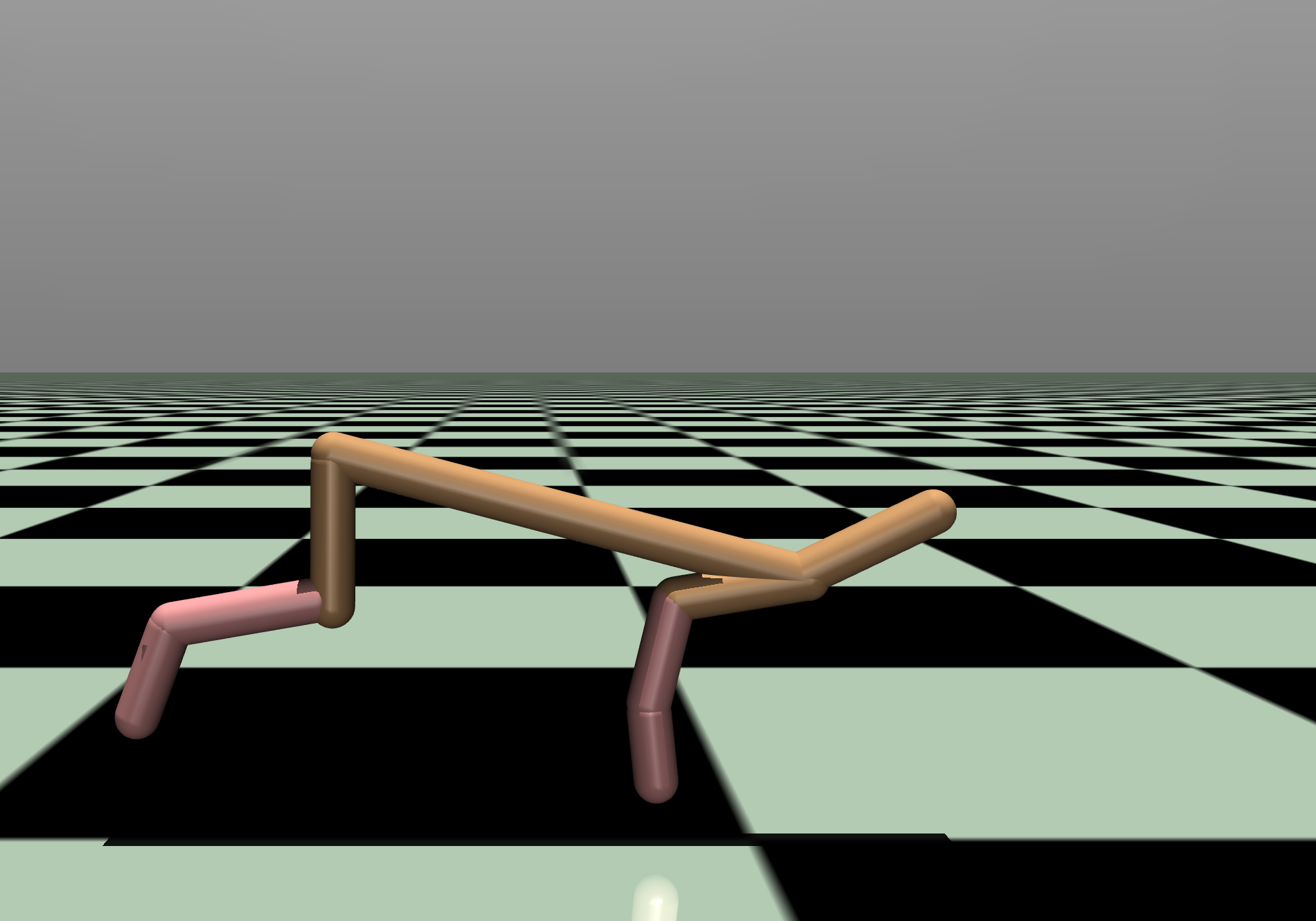}&
    \includegraphics[width=0.8\linewidth]{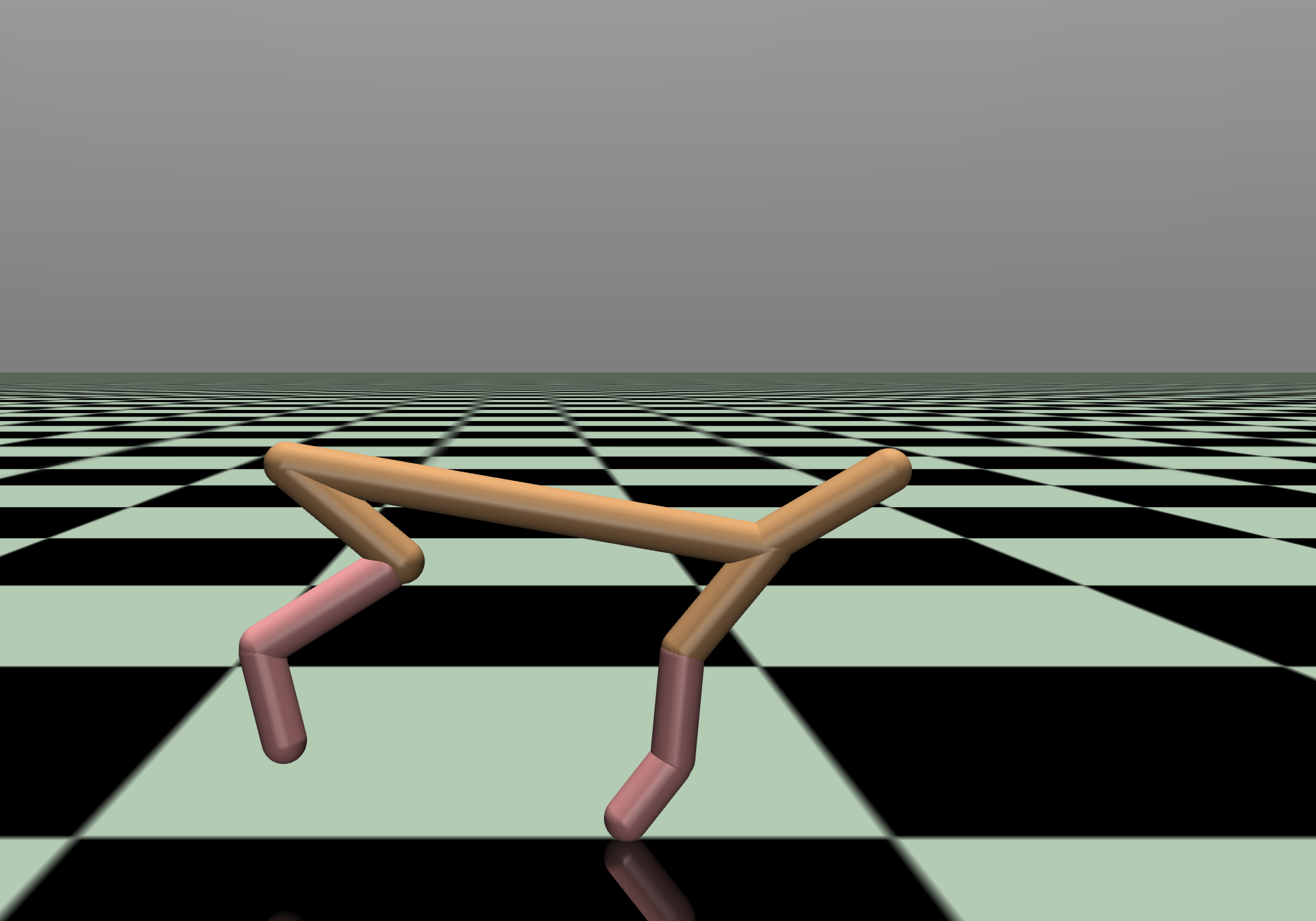}&
    \includegraphics[width=0.8\linewidth]{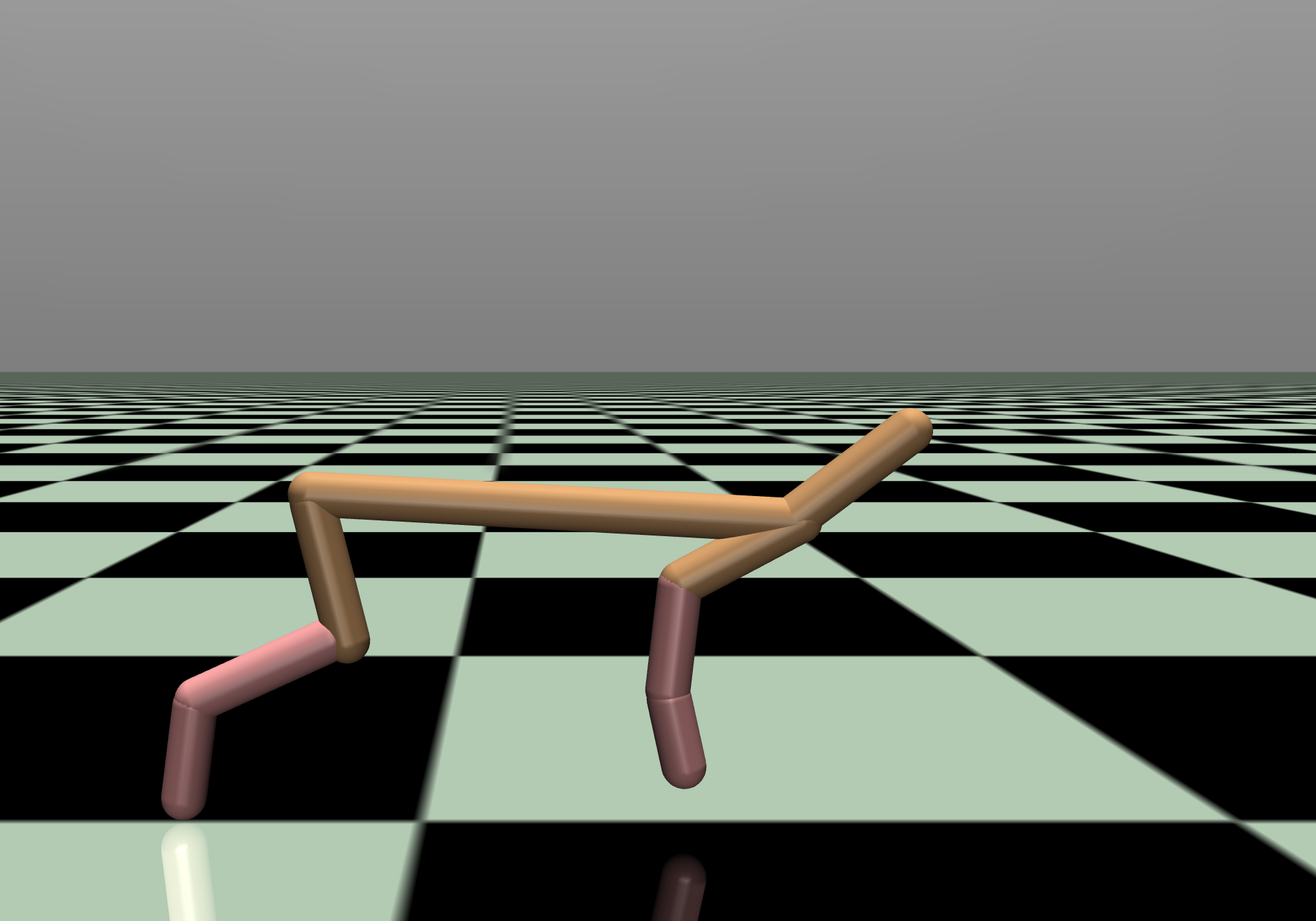}&
    \includegraphics[width=0.8\linewidth]{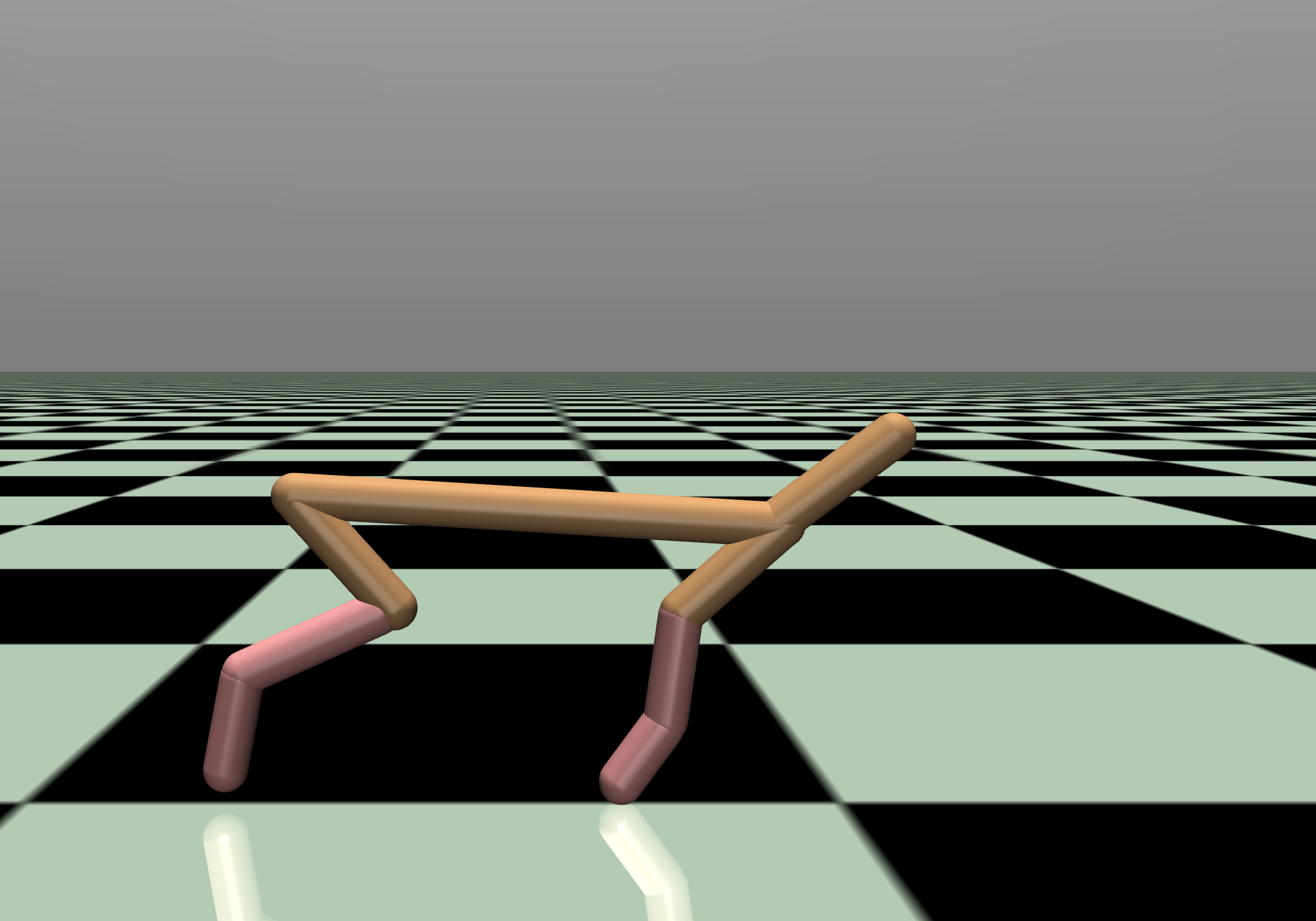}\\
    & & & & \\
    
    (b)&
    \includegraphics[width=0.8\linewidth]{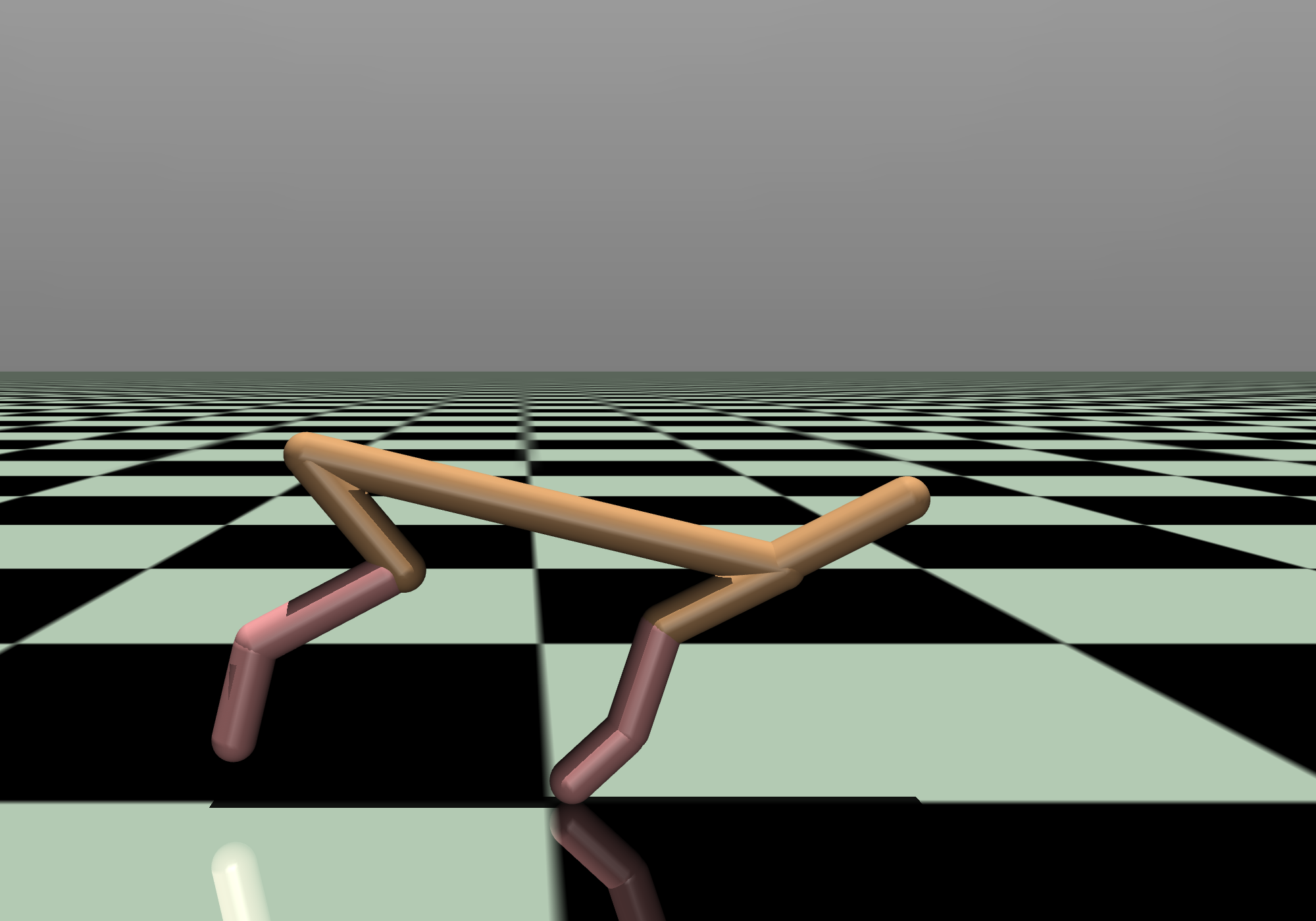}&
    \includegraphics[width=0.8\linewidth]{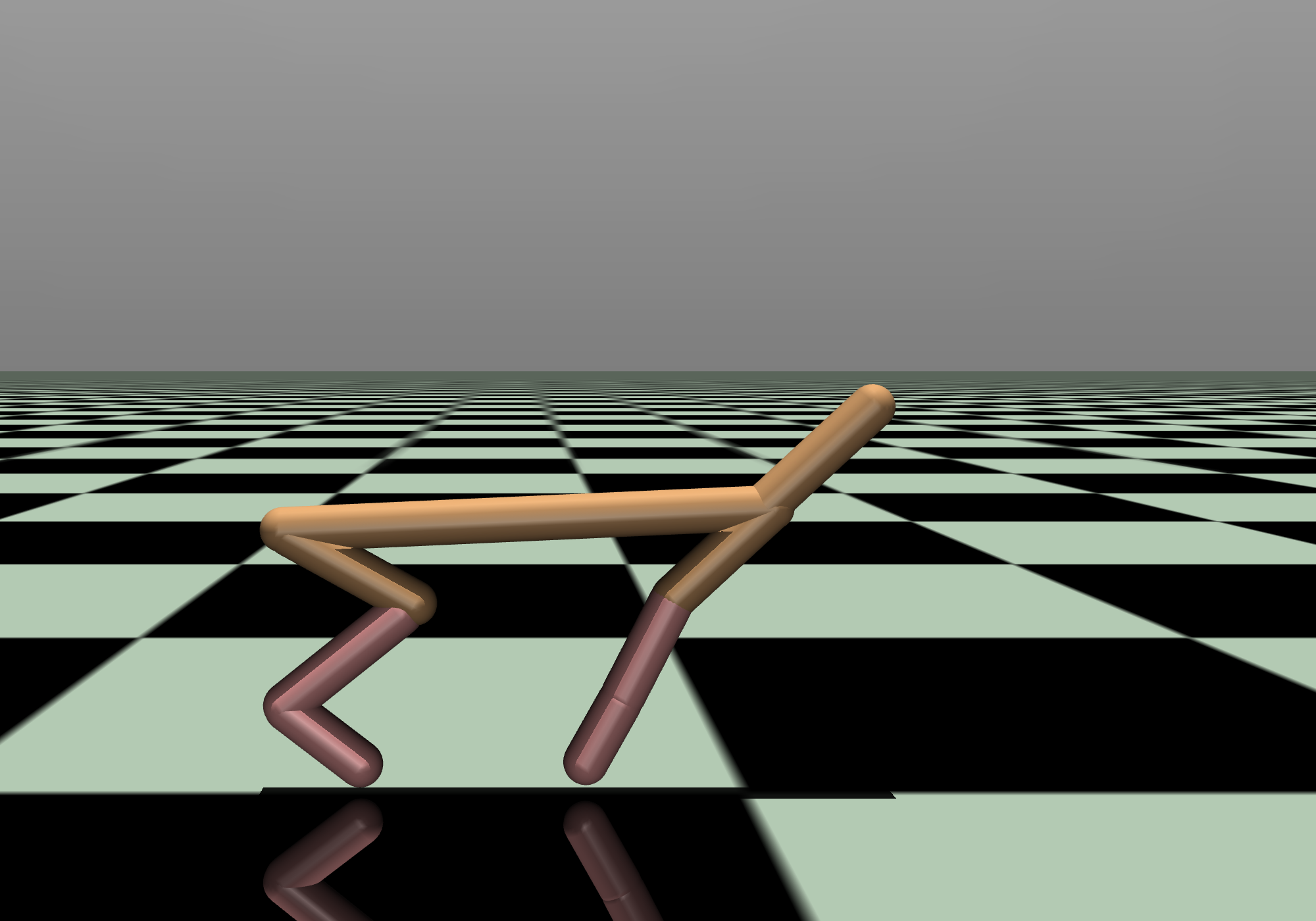}&
    \includegraphics[width=0.8\linewidth]{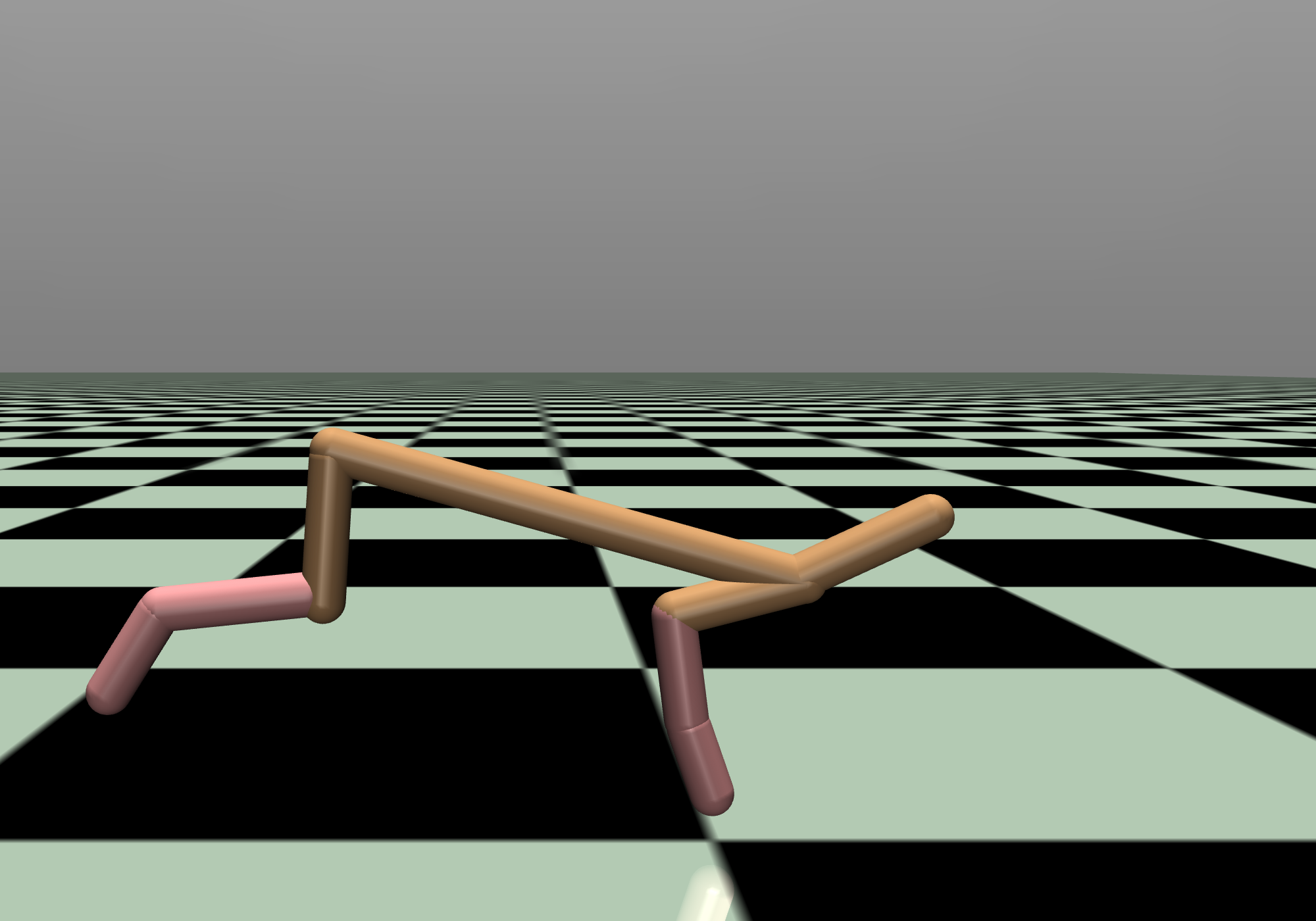}&
    \includegraphics[width=0.8\linewidth]{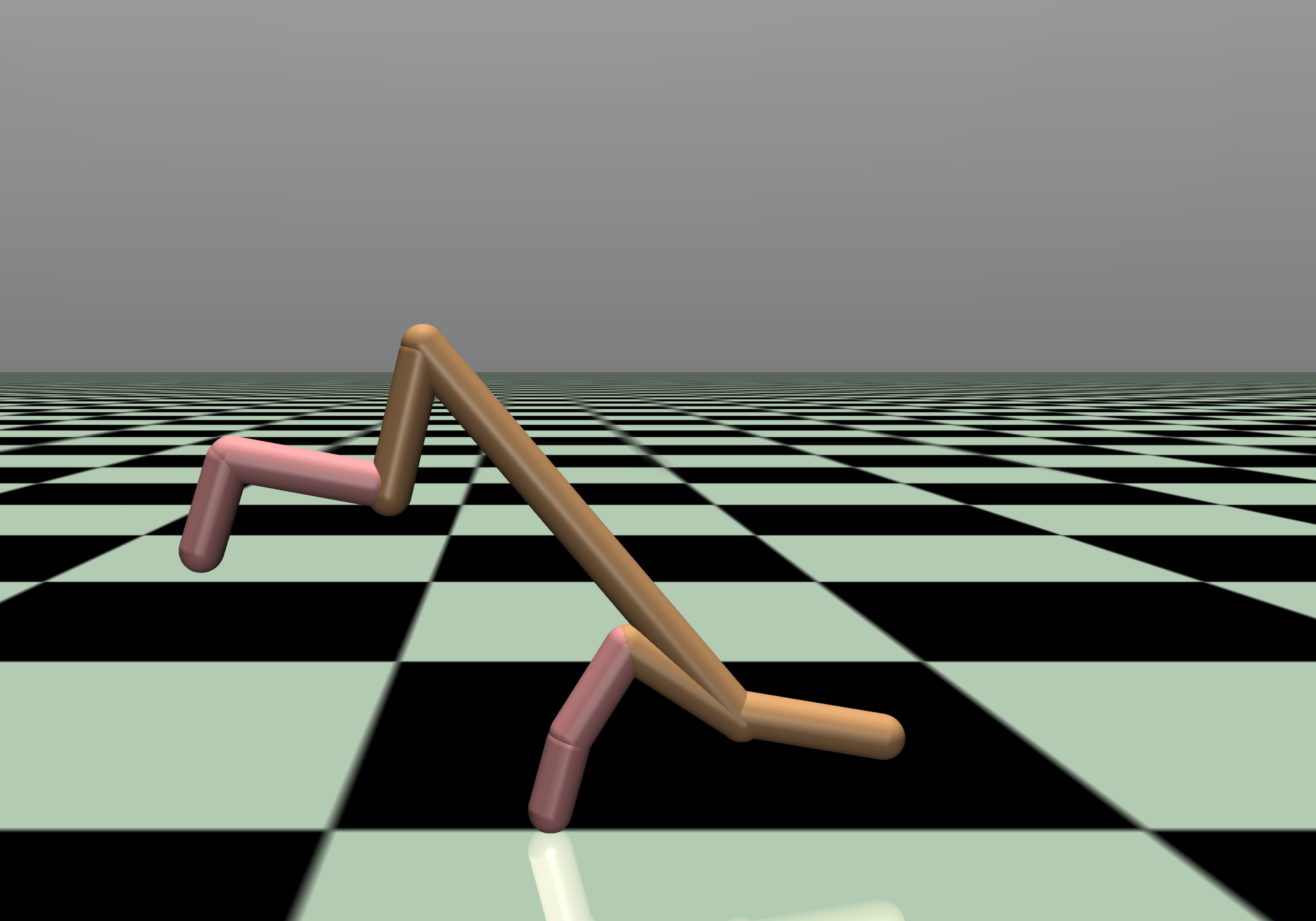}&
    \includegraphics[width=0.8\linewidth]{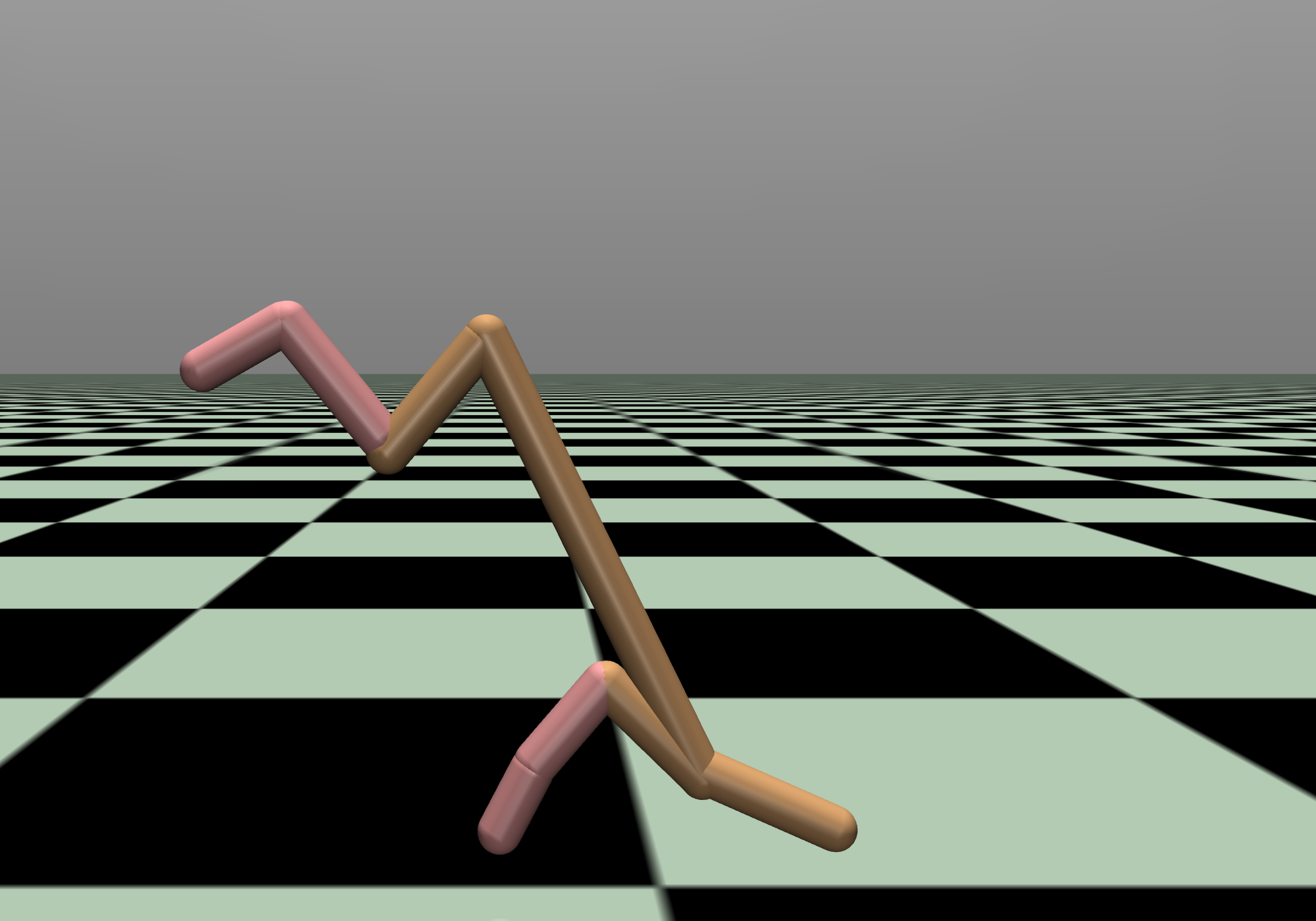}\\
    & & & & \\
     
    (c)&
    \includegraphics[width=0.8\linewidth]{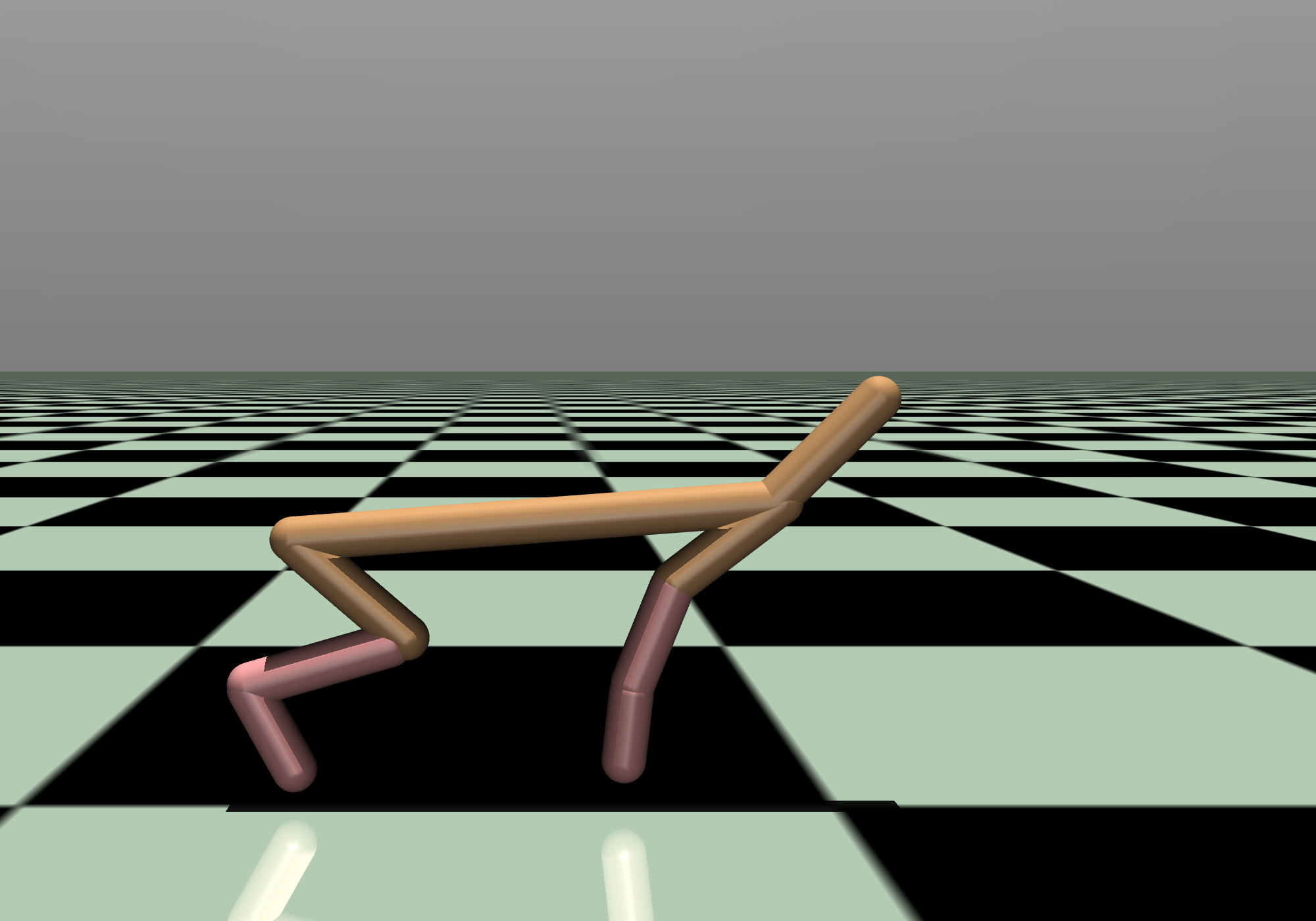}&
    \includegraphics[width=0.8\linewidth]{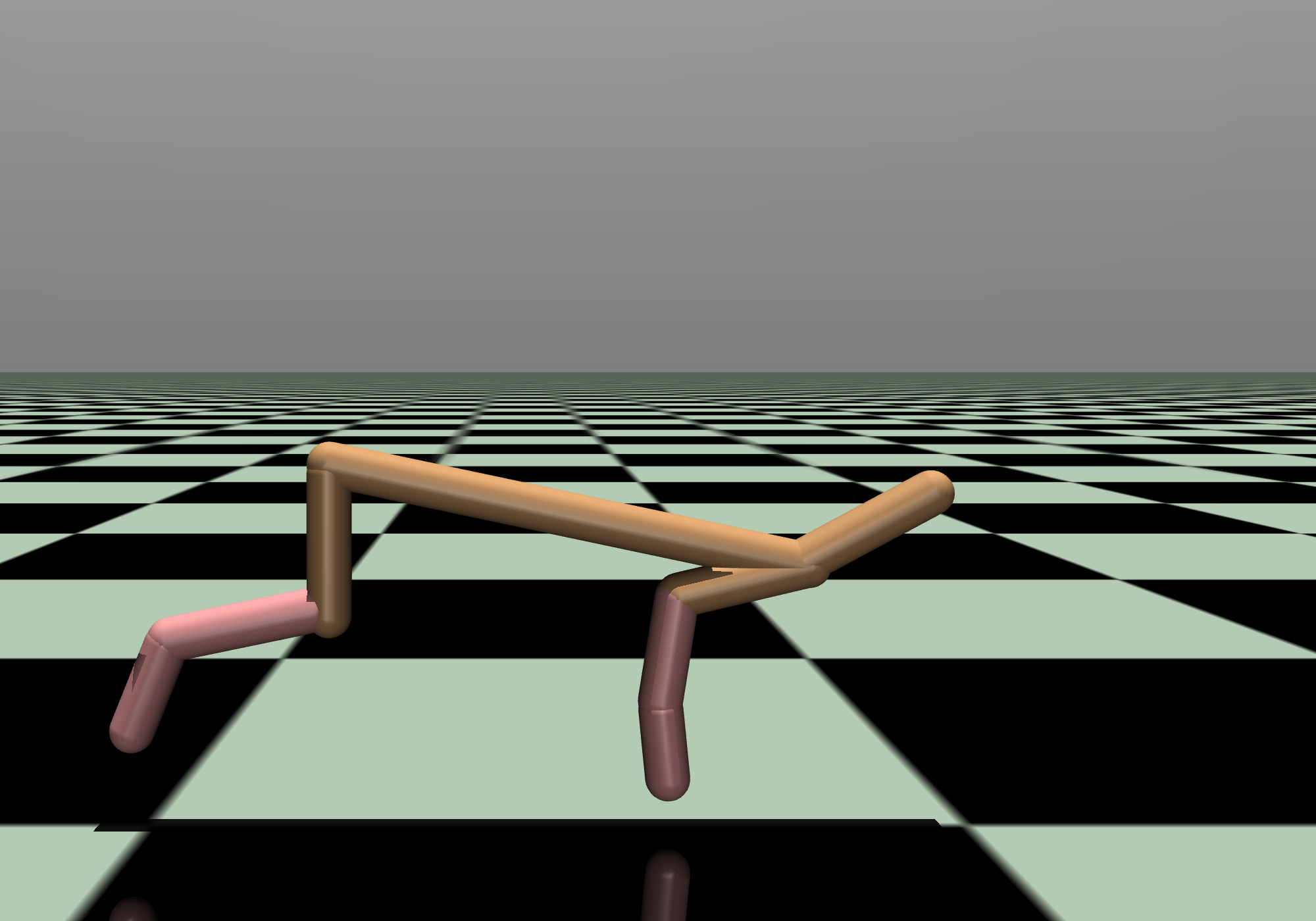}&
    \includegraphics[width=0.8\linewidth]{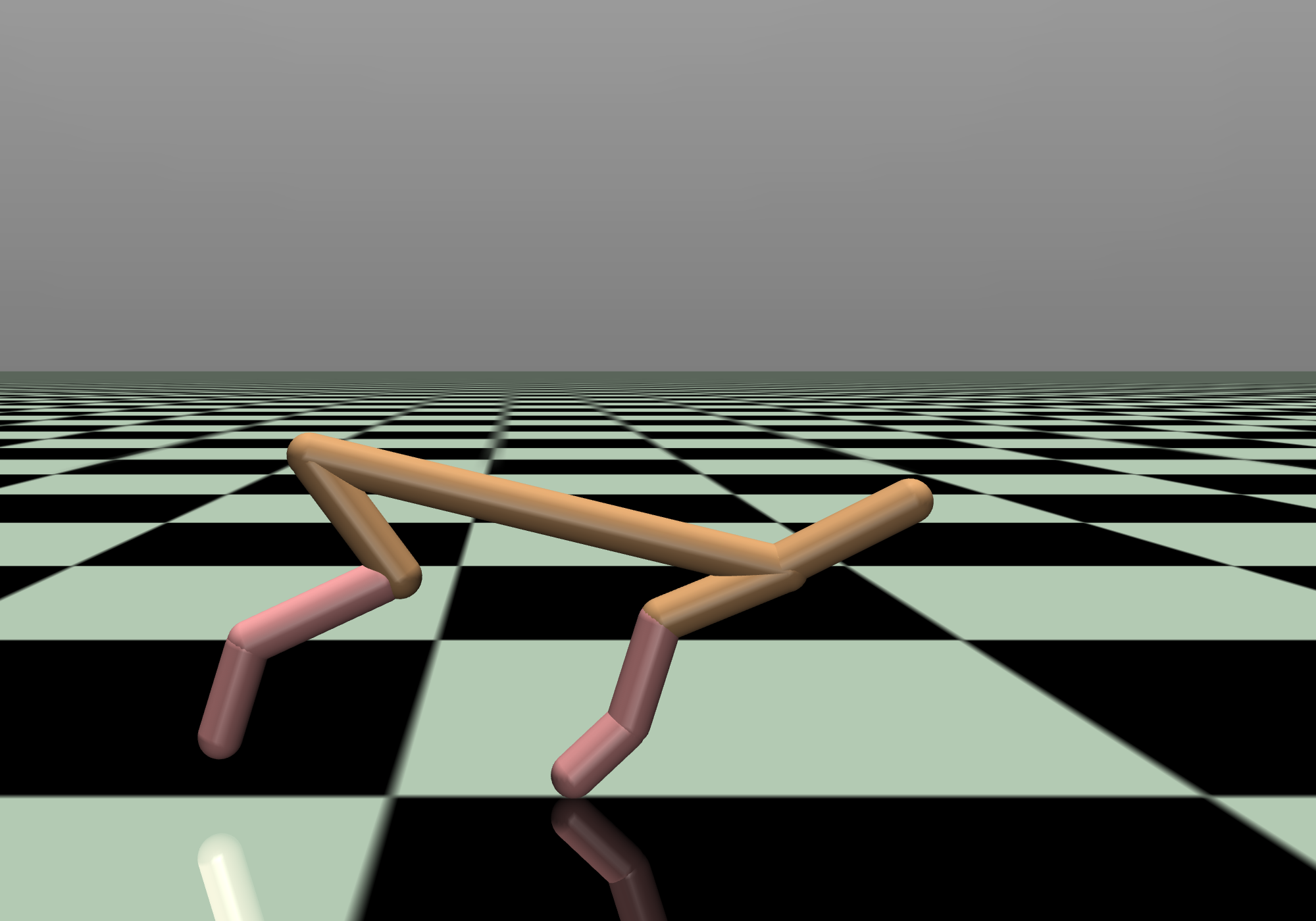}&
    \includegraphics[width=0.8\linewidth]{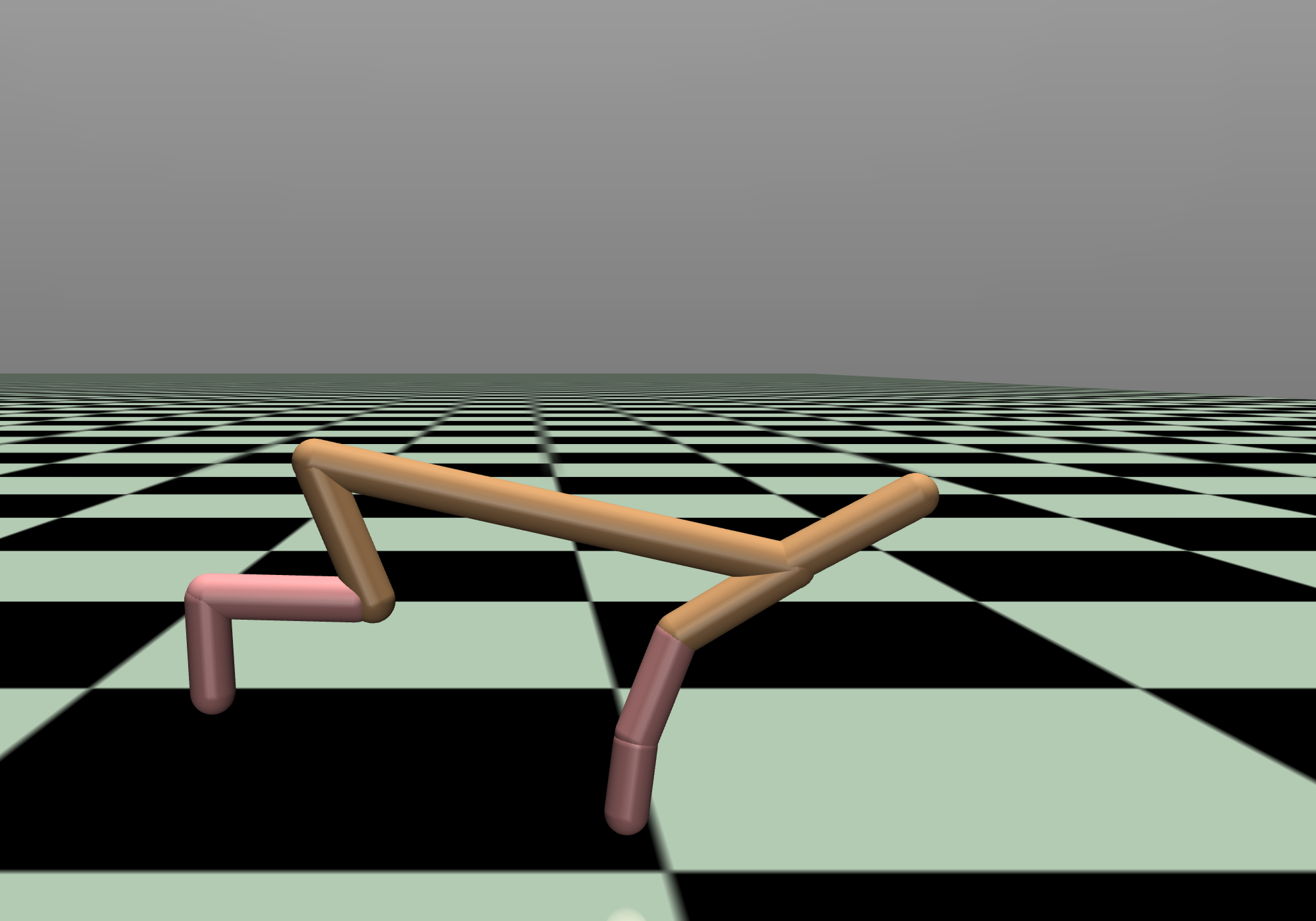}&
    \includegraphics[width=0.8\linewidth]{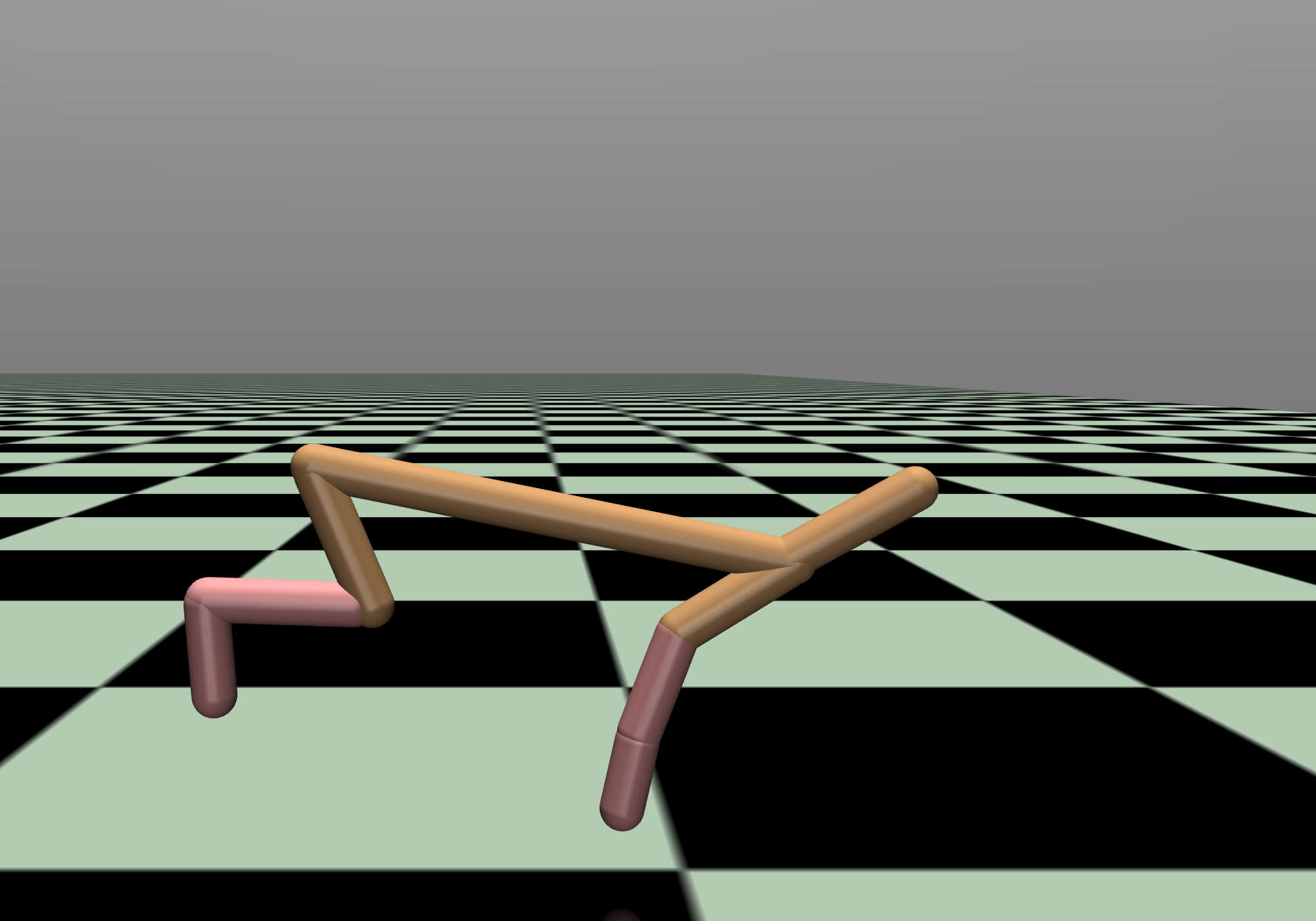}\\
    
     &
    $\qquad t = 0$&
    $\qquad t = 1$&
    $\qquad t = 2$&
    $\qquad t = 3$&
    $\qquad t = 4$\\
    \end{tabular}
  \caption{(a) Halfcheetah of mass $m$ using the source policy $\pi\sups$ in the source environment $\calM\sups$. (b) Halfcheetah of mass $1.5 \times m$ using the source policy $\pi\sups$ in the target environment $\calM\supt$. (c) Halfcheetah of mass $1.5 \times m$ using the learned target policy $\pi\supt$ in $\calM\supt$ with out method. Using the policy trained in the source environment without adapting it to the target environment yields suboptimal results. The adapted policy $\pi\supt$ can recover gaits close to the source trajectory.}
  \label{traj}
\end{figure*}

The paper is organized as follows. We review related work in Section 2 and the preliminaries in Section 3. In Section 4, we propose the theoretical version of the adaptation algorithm and prove its rate of convergence. In section 5 we describe the practical implementation using deviation models and practical optimization methods. We show detailed comparison with competing approaches in Section 6.
\section{Related Work}
\label{sec:related}
Our work connects to existing works on imitation learning, online adaptation, domain randomization and meta-learning, and model-based reinforcement learning. 

\noindent{\bf Imitation Learning.} In imitation learning, there is typically no separation between training environments and test environments. Existing imitation learning approaches aim to learn a policy that generates state distributions \cite{tobin2017domain,torabi2019recent,sun2019provably, yang2019imitation} or state-action distributions \cite{ho2016generative,fu2017learning,ke2019imitation,ghasemipour2019divergence} that are similar to the ones given by the expert policy. 
The difference in the policy adaptation setting is that the expert actions do not work in the first place in the new environment, and we need to both model the divergence and find a new policy for the target environment. In light of this difference, the work \cite{zhu2018reinforcement} considers a setting that is the closest to ours (which we will compare with in the experiments). It uses a state-action-imitation (GAIL\cite{ho2016generative}) approach to learn a policy in the target environment, to generate trajectories that are similar to the trajectories of the expert from the source environments. It also relies on using the true reward signals in the target environment to train the policy besides state imitation. In recent work, \cite{liu2020state} approaches a similar problem by using Wasserstein distance between the states as the reward. It uses adversarial training and model-free policy optimization methods. Our approach is model-based and relies on reduction to Data Aggregation~\cite{ross2011reduction} for efficiency. The reduction allows us to derive provable convergence guarantee with the rate of convergence. Experimentally we show that our approach is more sample efficient than model-free and minmax-based imitation approaches in general.

\noindent{\bf Online Adaptation.} Online adaptation methods transfer policies by learning a mapping between the source and target domains~\cite{daftry2016learning,tzeng2015adapting}. Such methods have achieved success in vision-based robotics but require extra kernel functions or learning feature spaces. In contrast, we focus on control problems where the policy adaptation is completely autonomous. 
\cite{christiano2016transfer} trains an inverse dynamics model (IDM) which can serve as the target policy by inquiring the source policy online. However, the approach does not focus on optimizing sample efficiency, which is crucial for the use of policy adaptation. In \cite{yu2018policy}, the agent selects from a range of pre-trained policies online, and does not perform further adaptation, and thus experiences problems similar to domain randomization approaches.

\noindent{\bf Domain Randomization and Meta-Learning.} 
Domain randomization and meta-learning methods are popular ways of transferring pre-trained policies to new environments. These methods rely on the key assumption that the training environments and test environments are sampled from the same predefined distribution. Domain randomization methods train robust agents on diverse samples of target environments~\cite{tobin2017domain,mordatch2015ensemble,antonova2017reinforcement,chebotar2019closing}. 
When the configuration of the target environment lies outside the training distribution, there is no performance guarantee. In meta-learning, such as Model Agnostic Meta-Learning (MAML) \cite{finn2017model,finn2017one,nagabandi2018learning}, meta-learned dynamics policies and models can adapt to perturbed environments with notable success including in physical robots. However, similar to domain randomization based approaches, they experience difficulties on new environments that are not covered by the training distribution. Our proposed approach focuses on online adaption in unseen environments. It is orthogonal to domain randomization and meta-learning approaches. We show in experiments that these different approaches can be easily combined. 

\noindent{\bf Model-based Reinforcement Learning.} 
Model-based reinforcement learning (MBRL) provides a paradigm that learns the environment dynamics and optimizes the control actions at the same time. Recent work has shown that MBRL has much better sample efficiency compared to model-free approaches both theoretically and empirically~\cite{tu2018gap,chua2018deep,sun2019model}. 
Our setting is different from the traditional MBRL setting. We consider test environments that are different from the training environment, and adapt the policy from the training environment to the test environment.

\section{Preliminaries}
\label{sec:prelim}
We consider finite-horizon Markov Decision Processes (MDP) $\calM =\langle\calS, \calA, f, H, R\rangle$ with the following components. $\mathcal{S}$ denotes the state space, and $\mathcal{A}$ the action space. The transition function $f\colon \mathcal{S} \times \mathcal{S} \times \mathcal{A} \rightarrow [0,1]$ determines the probability $f(s' |s,a) $ of transitioning into state $s'$ from state $s$ after taking action $a$. The reward function $R: \mathcal{S} \rightarrow \mathbb{R}$ is defined \emph{only} on states. We write $\pi_\theta$ to denote a stochastic policy $\pi_\theta \colon \mathcal{S} \times \mathcal{A} \rightarrow [0,1]$ parameterized by $\theta$. Each policy $\pi_\theta$ determines a distribution over trajectories $\{(s_i,a_i, r_i)\}_{i=1}^H$ under a fixed dynamics $f$. 
The goal of the agent is to maximize the expected cumulative reward $J(\theta) = {\Exp}_{\pi_{\theta}, f}\left[\sum_{h=1}^{H} R(s_h)\right]$ over all possible trajectories that can be generated by $\pi_{\theta}$.
Without loss of generality, in the theoretical analysis we always assume the normalized total reward is in the $[0,1]$ range, i.e., $\max_{s_1,\dots s_H}\sum_{h=1}^H R(s_h) \in [0,1]$.

We write the set $\{1, 2, \dots, N\}$ as $[N]$, and the uniform distribution over set $A$ as $U(A)$ throughout the paper. For any two distributions $d_1$ and $d_2$, we use $\|d_1-d_2\|$ to denote the total variation distance between the two distributions.

\section{Policy Adaptation with Data Aggregation}
\label{sec:theory}

\subsection{Basic Definitions}

In policy adapation, we consider a pair of MDPs and call them  a source MDP and a target MDP. We define the source MDP as $\calM\sups := \{ \calS, \calA\sups, f\sups, H, R \}$ and the target MDP as $\calM\supt := \{ \calS, \calA\supt, f\supt, H, R \}$. Note that the two MDPs share the same state space and reward functions, but can have different action spaces and transition dynamics. Fig. \ref{traj} demonstrates the problem of adapting a policy from a source environment to a target environment. Because of the difference in the action space and dynamics, directly using good policies from the source environment (Fig.\ref{traj}(a)) does not work in the target environment (Fig.\ref{traj}(b)). The objective is to adapt the policy from the source to the target environment to achieve good performance (Fig.\ref{traj}(c)). 

We focus on minimizing the samples needed for adaptation in the target MDP, by leveraging $\calM\sups$ to quickly learn a policy in $\calM\supt$. To achieve this, we assume that a pre-trained policy $\pi\sups$ from $\calM\sups$ achieves high rewards in $\calM\sups$. We wish to adapt $\pi\sups$ to a policy $\pi\supt$ that works well in $\calM\supt$. For ease of presentation, we consider $\pi^{(s)}$ and $\pi^{(t)}$ as deterministic throughout the theoretical analysis. 

Given a policy $\pi$, we write $\Exp\sups_{\pi}(\cdot)$ for the expectation over random outcomes induced by $\pi$ and $\calM\sups$. We write $d_{\pi;h}\sups$ to denote the state distribution induced by $\pi$ at time step $h$ under $\calM\sups$, and $d_{\pi}\sups = \sum_{h=1}^H d_{\pi;h}\sups / H$ as the average state distribution of $\pi$ under $\calM\sups$. We write $\rho_{\pi}\sups$ to represent the distribution of the state trajectories from $\pi$: for $\tau = \{s_h\}_{h=0}^H$,  $\rho_{\pi}\sups(\tau) = \prod_{h=1}^{H} f\sups(s_{h}|s_{h-1}, \pi(s_{h-1}))$. For the target MDP $\calM\supt$, we make the same definitions but drop the superscript $(t)$ for ease of presentation. 
Namely, $\Exp_{\pi}(\cdot)$ denotes the expectation over the randomness from $\pi$ and $\calM\supt$, $d_{\pi}$ denotes the induced state distribution of $\pi$ under $\calM\supt$, and $\rho_\pi$ denotes the state trajectory distribution. 

\subsection{Algorithm}

We now introduce the main algorithm Policy Adaptation with Data Aggregation (PADA). Note that this is the theoretical version of the algorithm, and the practical implementation will be described in detail in Section~\ref{sec:practical_alg}. To adapt a policy from a source environment to a target environment, PADA learns a model $\hat{f}$ to approximate the target environment dynamics $f\supt$. Based on the learned model, the algorithm generates actions that attempt to minimize the divergence between the trajectories in the target environment and those in the source environment generated by $\pi\sups$ at $\calM\sups$. Namely, the algorithm learns a policy $\pi\supt$ that reproduces the behavior of $\pi\sups$ on $\calM\sups$ in the target MDP $\calM\supt$. Since the state space $\calS$ is often large, learning a model $\hat{f}$ that can accurately approximate $f\supt$ globally is very costly. Instead, we only aim to iteratively learn a locally accurate model, i.e., a model that is accurate near the states that are generated by $\pi\supt$. This is the key to efficient adaptation.

\begin{algorithm}[t]
\caption{Policy Adaptation with Data Aggregation}
\label{alg:alg_theory}
\begin{algorithmic}[1]
\REQUIRE Source domain policy $\pi\sups$, source dynamics $f\sups$,  model class $\mathcal{F}$
\STATE Initialize dataset $\mathcal{D} = \emptyset$
\STATE Initialize $\hat{f}_1$
\FOR{e = 1, \ldots {\text{T}}}
\STATE Define policy $\pi\supt_e$ as in Eq.~\ref{eq:policy_f}

\FOR{n = 1, \ldots {\text{N}}} 
\STATE Reset $\calM\supt$ to a random initial state
\STATE Uniformly sample a time step $h\in [H]$
\STATE Execute $\pi\supt_e$ in $\calM\supt$ for $h$ steps to get state $s$
\STATE Sample exploration action $a\sim U({\calA}\supt)$
\STATE Take $a$ in $\calM\supt$ and get next state $s'$ 
\STATE Add $(s,a,s')$ into $\mathcal{D}$ (Data Aggregation)
\ENDFOR
\STATE Update to $\hat{f}_{e+1}$ via MLE by Eq.~\ref{eq:mle}
\ENDFOR 
\STATE \textbf{Output}: $\{\pi\supt_e\}_{e=1}^T$
\end{algorithmic}
\end{algorithm}

The detailed algorithm is summarized in Alg.~\ref{alg:alg_theory}. Given a model $\hat{f}_e$ at the $e$-th iteration, we define the ideal policy $\pi\supt_e$
\begin{align}
{\pi\supt_e(s) \triangleq \argmin_{a\in\mathcal{A}\supt} \| \hat{f}_e(\cdot|s,a) - f\sups(\cdot | s, \pi\sups(s)) \|.}
\label{eq:policy_f}
\end{align} 

The intuition is that, assuming $\hat{f}_e$ is accurate in terms of modelling $f\supt$ at state $s$, $\pi\supt_e(s)$ aims to pick an action such that the resulting next state distribution under $\hat{f}_e$ is similar to the next state distribution resulting from $\pi\sups$ under the source dynamics $f\sups$. We then execute $\pi\supt_e$ in the target environment $\calM\supt$ to generate a batch of data $(s,a,s')$. We further aggregate the newly generated data to the dataset $\mathcal{D}$ (i.e., data aggregation). We update model to $\hat{f}_{e+1}$ via Maximum Likelihood Estimation (MLE) on $\mathcal{D}$: 
\begin{align}
{\hat{f}_{e+1} = \argmax_{f\in\mathcal{F}} \sum_{s,a,s'\in\mathcal{D}}\log f(s'| s,a).}
\label{eq:mle}
\end{align}

Note that Algorithm \ref{alg:alg_theory} relies on two black-box offline computation oracles: (1) a \emph{one-step} minimization oracle (Eq.~\ref{eq:policy_f}) and (2) a Maximum Likelihood Estimator (Eq.~\ref{eq:mle}). In Section~\ref{sec:practical_alg}, we will introduce practical methods to implement these two oracles. We emphasize here that these two oracles are offline computation oracles and the computation itself does not require any fresh samples from the target environment $\calM\supt$.

\subsection{Analysis}

We now prove the performance guarantee of Alg.\ref{alg:alg_theory} for policy adaptation and establish its rate of convergence. At a high level, our analysis of Alg.~\ref{alg:alg_theory} is inspired from the analysis of DAgger \cite{ross2011reduction,ross2012agnostic} which leverages a reduction to no-regret online learning \cite{shalev2012online}. We will first make the connection with the Follow-the-Leader (FTL) algorithm, a classic no-regret online learning algorithm, on a sequence of loss functions. We then show that we can transfer the no-regret property of FTL to performance guarantee on the learned policy $\pi\supt$. Our analysis uses the FTL regret bound $\tilde{O}(1/T)$ where $T$ is the number of iterations~\cite{shalev2012online}. Since our analysis is a reduction to general no-regret online learning, in theory we can also replace FTL by other no-regret online learning algorithms as well (e.g., Online Gradient Descent \cite{zinkevich2003online} and AdaGrad \cite{duchi2011adaptive}). 
Intuitively, for fast policy adaptation to succeed, one should expect that there is similarity between the source environment and the target environment. We formally introduce the following assumption to quantify this. 
\begin{assum}[Adaptability]
\label{assum:similarity}
For any state action pair $(s,a)$ with source action $a\in\calA\sups$ , there exists a target action $a'\in \calA\supt$ in target environment, such that: 
\begin{align*}
\| f\sups(\cdot |s, a) - f\supt(\cdot |s, a') \| \leq \epsilon_{s,a},
\end{align*} for some small $\epsilon_{s,a} \in \mathbb{R}^+$. 
\end{assum} 
\begin{remark} When $\epsilon_{s,a} \to 0$ in the above assumption, the target environment can perfectly recover the dynamics of the source domain at $(s,a)$. However, $\epsilon_{s,a} = 0$ does not mean the two transitions are the same, i.e., $f\supt(s,a) = f\sups(s,a)$. First the two action spaces can be widely different. Secondly, there may exist states $s$, where one may need to take completely different target actions from $\calA\supt$ in order to match the source transition $f\sups(\cdot|s,a)$, i.e., $\exists a'\in\calA\supt$ such that $f\supt(\cdot|s,a') = f\sups(\cdot|s,a)$, but $a\neq a'$.
\end{remark}

\begin{assum}[Realizability]
\label{assum2:Realizability}
Let the model class $\mathcal{F}$ be a subset of $\{f : \calS\times\calS\times\calA \to [0,1]\}$. We assume $f\supt\in \mathcal{F}$. 
\end{assum} 
Here we assume that our model class $\mathcal{F}$ is rich enough to include $f\supt$. Note that the assumption on realizability is just for analysis simplicity. Agnostic results can be achieved with more refined analysis similar to~\cite{ross2011reduction,ross2012agnostic}.

We define the following loss function: 
\begin{align*}
\ell_e(f) \triangleq \mathbb{E}_{s\sim d_{\pi\supt_e} ,a\sim U(\mathcal{A}\supt)}\left[ D_{KL}\left(  f\supt(\cdot|s,a), f(\cdot |s,a)  \right)\right], 
\end{align*} for all $e\in [T]$. 
The loss function $\ell_e(f)$ measures the difference between $f\supt$ and $f$ under the state distribution induced by $\pi\supt_e$ under $\calM\supt$ and the uniform distribution over the action space. This definition matches the way we collect data inside each episode. We generate $(s,a,s')$ triples via sampling $s$ from $d_{\pi\supt_e}$, $a$ from $U(\calA\supt)$, and then $s'\sim f\supt(\cdot|s,a)$.
At the end of the iteration $e$, the learner uses FTL to compute $\hat{f}_{e+1}$ as:
\begin{align*}
\hat{f}_{e+1} = \argmin_{f\in\mathcal{F}} \sum_{i=1}^e \ell_i(f).
\end{align*}
Using the definition of KL-divergence, it is straightforward to show that the above optimization is equivalent to the following Maximum Likelihood Estimation:
\[\argmax_{f\in\mathcal{F}}\sum_{i=1}^e\mathbb{E}_{s\sim d_{\pi\supt_e}, a\sim U(\calA\supt), s'\sim f\supt(\cdot|s,a)}\left[\log f(s'|s,a)\right].\]
At the end of the episode $e$, the aggregated dataset $\calD$ contains triples that are sampled based on the above procedure from the first to the $e$-th episode.

With no-regret learning on $\hat{f}_e$, assumptions~\ref{assum:similarity}, and \ref{assum2:Realizability}, we can obtain the following main results. We first assume that the target environment $\calM\supt$ has a discrete action space, i.e., $\calA\supt$ is discrete, and then show that the result can be easily extended to continuous action spaces. 

\begin{theorem}[Main Theorem]
Assume  $\calM\supt$ has a discrete action space $\mathcal{A}\supt$ and denote $A \triangleq |\calA\supt|$. Among the sequence of policies computed in Alg.~\ref{alg:alg_theory}, there exists a policy $\hat{\pi}$ such that:
\begin{align*}
\mathbb{E}_{s\sim d_{\hat\pi}} \| f\supt(\cdot | s, \hat\pi(s)) - f\sups(\cdot | s, \pi\sups(s))    \| \\\leq O\left({A} T^{-1/2} + \mathbb{E}_{s\sim d_{\hat\pi}}\left[ \epsilon_{s, \pi\sups(s)} \right]\right),
\end{align*} which implies that:
\begin{align*}
\left\| \rho_{\hat\pi} - \rho\sups_{\pi\sups} \right\| \leq O\left(H A T^{-1/2} + H\mathbb{E}_{s\sim d_{\hat\pi}}\left[ \epsilon_{s, \pi\sups(s)} \right] \right),
\end{align*} where we recall that $\rho_{\pi}$ stands for the state-trajectory distribution of policy $\pi$ under $\mathcal{M}\supt$ and $\rho\sups_{\pi\sups}$ stands for the state-trajectory distribution of $\pi\sups$ under $\mathcal{M}\sups$.
\label{them:main}
\end{theorem}
The full proof is in the Appendix \ref{app:proof_of_main_theorem}. The theorem shows that our algorithm can provide a policy in the target environment that induces trajectories close to those induced by the experts in the source environment. For instance, if the target and source MDPs are completely adaptable (i.e., $\epsilon_{s,a} = 0$ in Assumption~\ref{assum:similarity} for all $(s,a)$)  and the number of iterations approach to infinity, then we can learn a  policy $\hat{\pi}$ that generates state trajectories in $\calM\supt$ that match the state trajectories generated via the source policy $\pi\sups$ at the source MDP $\calM\sups$. 

\begin{remark} The error $\mathbb{E}_{s\sim d_{\hat\pi}}\left[ \epsilon_{s, \pi\sups(s)}\right]$ is averaged over the state distribution induced by the learned policy rather than in an $\ell_{\infty}$ form, i.e., $\max_{s,a} \epsilon_{s,a}$. 
\end{remark}

Although the analysis is done on discrete action space, the algorithm can be naturally applied to compact continuous action space as follows. The proof of the following corollary and its extension to the $d$-dimensional continuous action spaces are in the Appendix. 
\begin{corollary}[Continuous Action Space]
Assume ${\calA}\supt = [0,1]$, $f\supt$ and functions $f\in\mathcal{F}$ are Lipschitz continuous with (and only with) actions in $\calA\supt$. Among policies returned from Alg.~\ref{alg:alg_theory}, there exists a policy $\hat{\pi}$ such that:
\[\| \rho_{\hat{\pi}} - \rho\sups_{\pi\sups} \| \leq O\left(H  T^{-1/4} + H\mathbb{E}_{s\sim d_{\hat\pi}}\left[ \epsilon_{s, \pi\sups(s)} \right]  \right).\]
\label{cor:continuous_action}
\end{corollary}
\vspace{-.8cm}
\begin{remark} As we assume the reward function only depends on states,  $\|\rho_{\hat{\pi}} - \rho\sups_{\pi\sups} \| \leq \delta$ implies $|J\supt( \hat{\pi}) - J\sups(\pi\sups)|\leq \| \rho_{\hat\pi} - \rho\sups_{\pi\sups} \| \left(\max_{s_1\dots s_H} \sum_h^H R(s_h)\right) \leq \delta$ due to the normalization assumption on the rewards. 
Thus, though our algorithm runs without rewards, when $\pi\sups$ achieves high reward in the source MDP $\calM\sups$, the algorithm is guaranteed to learn a policy  $\hat\pi$ that achieves high rewards in the target environment $\calM\supt$.
\end{remark}

\section{Practical Implementation}
\label{sec:practical_alg}
In Algorithm~\ref{alg:alg_theory} we showed the theoretical version of our approach, which takes an abstract model class $\calF$ as input and relies on two offline computation oracles 
(Eq.~\ref{eq:policy_f} and Eq.~\ref{eq:mle}). We now design the practical implementation by specifying the parameterization of the model class $\calF$ and the optimization oracles. Algorithm~\ref{alg:practical} shows the practical algorithm, and we explain the details in this section.

\subsection{Model Parameterization} \label{model_param}
In the continuous control environments, we focus on 
stochastic transitions with Gaussian noise, where $f\supt (s,a)= \bar{f}\supt(s,a) + \epsilon$, $f\sups(s,a) = \bar{f}\sups(s,a) + \epsilon'$, with $\epsilon$ and $\epsilon'$ from $\mathcal{N}(0,\Sigma)$ and $\bar{f}\supt$ and $\bar{f}\sups$ being nonlinear deterministic functions. In this case, we 
consider the following model class with parameterization $\theta$:  
{ \begin{align*}
\calF = \{\delta_{\theta}(s,a) + \hat f^{(s)}(s,\pi\sups(s)),\forall s,a: \theta\in\Theta \}.
\end{align*}}
where $\hat f^{(s)}$ is a pre-trained model of the source dynamics $f^{(s)}$ and we assume $\hat f\sups$ well approximates $f\sups$ (and one has full control to the source environment such as the ability to reset). Then for each state $s$, $\hat{f}\sups(s,\pi\sups(s))$ is a fixed distribution of the next state in the source environment by following the source policy. 
Define $\Delta^{\pi\sups}(s,a) \triangleq \hat{f}\sups(s,\pi\sups(s)) - f\supt(s,a)$, which captures the deviation from taking action $a$ in the target environment to following $\pi\sups$ in the source environment. 
So $\delta_{\theta}(s,a)$ is trained to approximate the deviation $\Delta^{\pi\sups}(s,a)$. Note that learning $\Delta^{\pi\sups}$ is just an alternative way to capture the target dynamics since we know $\hat{f}\sups(s,\pi(s))$ foresight, thus it should be no harder than learning $f\supt$ directly.

\subsection{Model Predictive Control}

For deterministic transition, Eq~\ref{eq:policy_f} reduces to one-step minimization $\argmin_{a\in{\calA}\supt} \| \hat{f}_e(\cdot|s,a) - \hat{f}\sups( s, \pi\sups(s)) \|_2.$ Since $\hat{f}_e\in \calF$, we have $\hat{f}_e(s,a) = \delta_{\theta_e}(s,a) + \hat f\sups(s,\pi\sups(s))$, and the optimization can be further simplfied to:
$\argmin_{a\in\mathcal{A}\supt} \| \delta_{\theta_e}(s,a) \|_2.$
We use the Cross Entropy Method (CEM)~\cite{botev2013cross} which iteratively repeats: randomly draw $N$ actions, evaluate them in terms of the objective value $\|\delta_{\theta_e}(s,a)\|_2$, pick the top $K$ actions in the increasing order of the objective values, and then refit a new Gaussian distribution using the empirical mean and covariance of the top $K$ actions.

We emphasize here we only need to solve a \emph{one-step} optimization problem without unrolling the system for multiple steps. We write the CEM oracle as $\mathrm{CEM}(\calA\supt, s, \delta_{\theta})$ which outputs an action $a$ from $\calA\supt$ that approximately minimizes $\|\delta_{\theta}(s,a)\|_2$. Here, $\mathrm{CEM}(\calA\supt, s, \delta_{\theta}): \mathcal{S}\to\calA\supt$ can be considered as a policy that maps state $s$ to a target action $a$. 

\begin{algorithm}[t]
\caption{Policy Adaptation with Data Aggregation via Deviation Model}
\label{alg:practical}
\begin{algorithmic}[1]
\REQUIRE $\pi_s$, $\hat{f}\sups$, deviation model class $\{\delta_{\theta}: \theta\in\Theta\}$, explore probability $\epsilon$, 
replay buffer $\calD$, learning rate $\eta$
\STATE Randomly initialize divergence model $\delta_{\theta}$ 
\FOR{$T$ Iterations}
\FOR{$n$ steps} 
\STATE $s \gets \text{Reset $\calM\supt$}$ 
\WHILE{current episode does not terminate}
\STATE With probability $\epsilon$: $a \sim U(\calA\supt)$ \label{line:epsilon_greedy}
\STATE Otherwise: $a \gets \mathrm{CEM}(\calA\supt, s, \delta_{\theta})$
\STATE Execute $a$ in $\calM\supt$: $s' \gets f\supt(s,a)$
\STATE Update replay buffer: $\calD \gets \calD \cup \{(s, a, s')\}$ \label{line:replay_buffer}
\STATE $s  \gets s'$ 
\ENDWHILE
\ENDFOR
\STATE Update $\theta$ with Eq. \ref{eq:sgd}
\ENDFOR
\end{algorithmic}
\end{algorithm}

\subsection{Experience Replay for Model Update}\label{ssec:exp}

Note that Alg.~\ref{alg:alg_theory} requires to solve a batch optimization problem (MLE in Eq.~\ref{eq:mle}) in every iteration, which could be computationally expensive in practice. We use Experience Replay \cite{adam2011experience,mnih2013playing}, which is more suitable to optimize rich non-linear function approximators ($\delta_{\theta}$ is a deep neural network in our experiments). Given the current divergence model $\delta_{\theta}$ and the aggregated dataset $\calD = \{s, a, s'\}$ (aka, replay buffer) with $s' = f\supt(s,a)$, we randomly sample a mini-batch $B\subset \calD$ and perform a stochastic gradient descent step with learning rate $\eta$:
\begin{align} 
\resizebox{1.\hsize}{!}{$\theta\leftarrow \theta - \frac{\eta}{|B|}\nabla_{\theta}\left( \sum_{i=1}^{|B|} \| \hat{f}\sups(s_i,\pi\sups(s_i)) +\delta_{\theta}(s_i,a_i) - s'_i \|_2^2  \right)$}.
\label{eq:sgd}
\end{align}
\subsection{Policy Adaptation with Data Aggregation}
\label{ssec:summ}
As show in Algorithm~\ref{alg:practical}, we maintain a reply buffer that stores all experiences from the target model $\calM\supt$ (Line~\ref{line:replay_buffer}) and constantly update the model $\delta_{\theta}$ using mini-batch SGD (Eq.~\ref{eq:sgd}).  Alg~\ref{alg:practical} performs local exploration in an $\epsilon$-greedy way.
We refer our method as Policy Adaptation with Data Aggregation via Deviation Model (\textbf{PADA-DM}). 

\begin{remark}  Even being one-step, $\mathrm{CEM}(\calA\supt, s, \delta_{\theta})$ may be computationally expensive, we could obtain a MPC-free policy (target policy) by training an extra parameterized policy to mimic $\mathrm{CEM}(\calA\supt, s, \delta_{\theta})$ via techniques of Behavior Cloning \cite{bain1999framework}. 
When we train this extra parameterized policy, we name the method as \textbf{PADA-DM with target policy} and we will show it does not affect the performance of the overall algorithm during training. However, during test time, such parameterized policy runs faster than CEM and thus is more suitable to be potentially deployed on real-world systems.
\end{remark}

\begin{figure*}[th!]
    \centering
    \includegraphics[width = 0.245\textwidth, trim=0cm 0cm 0cm 0cm, clip=true]{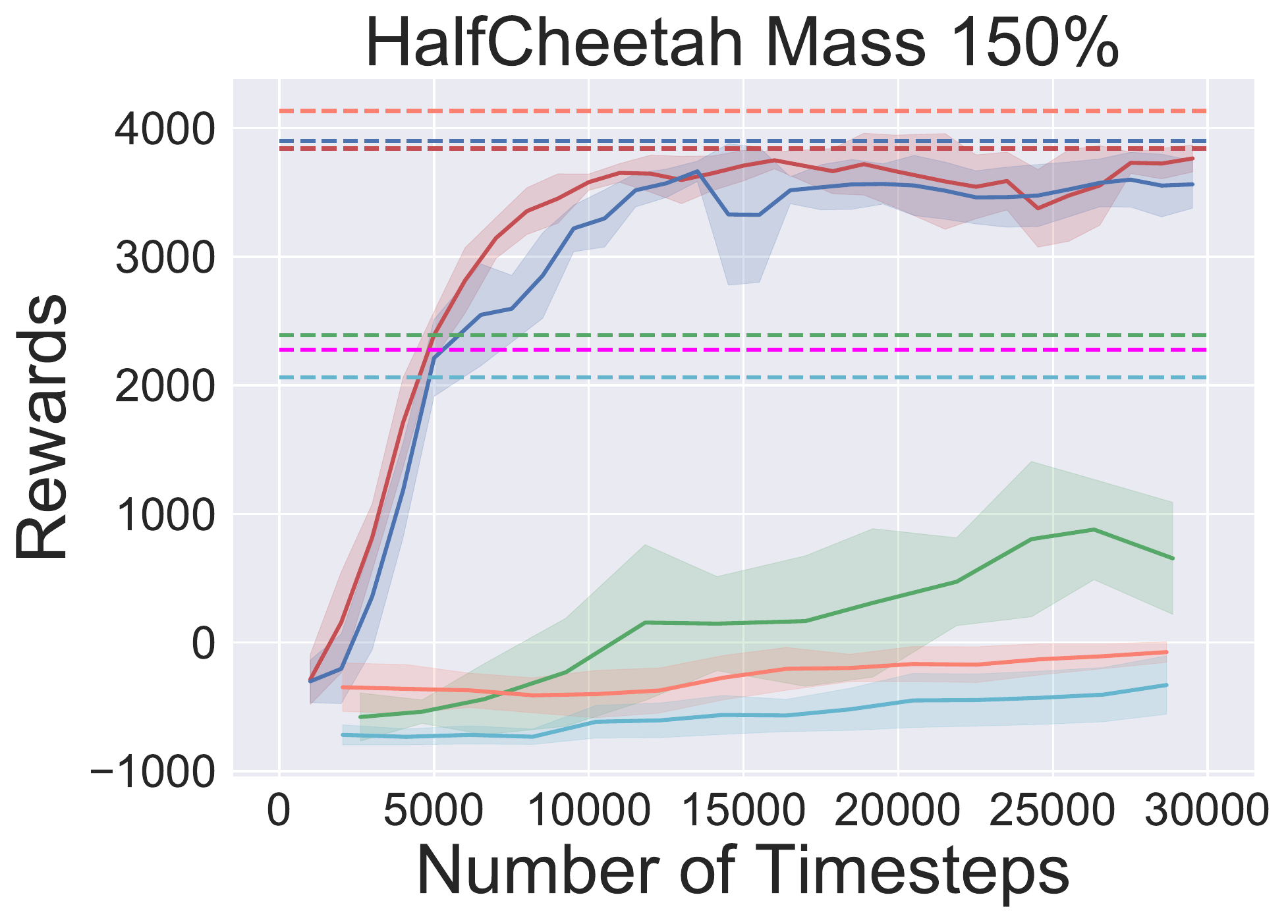}
	\includegraphics[width = 0.245\textwidth, trim=0cm 0cm 0cm 0cm, clip=true]{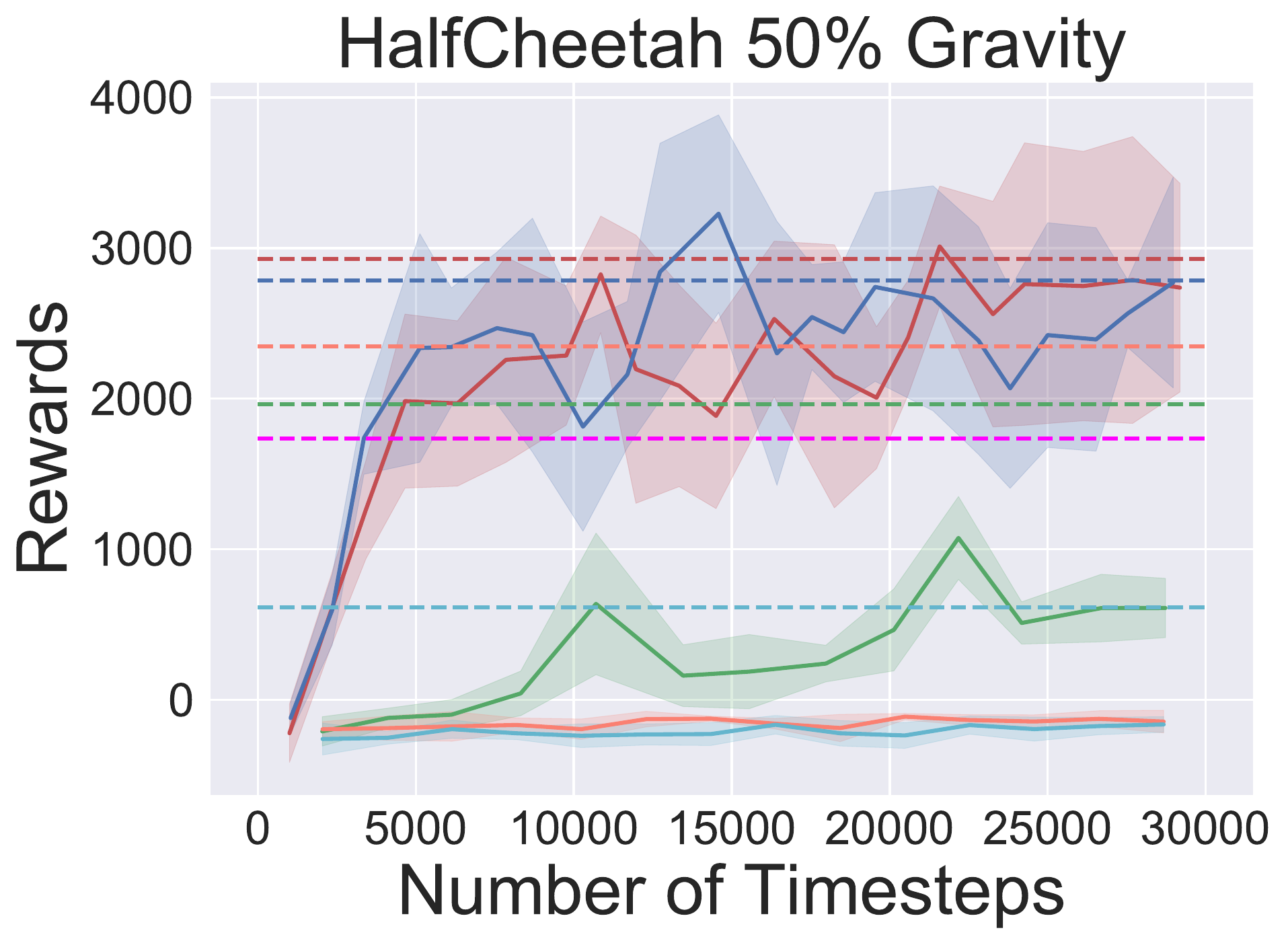}
	\includegraphics[width = 0.245\textwidth, trim=0cm 0cm 0cm 0cm, clip=true]{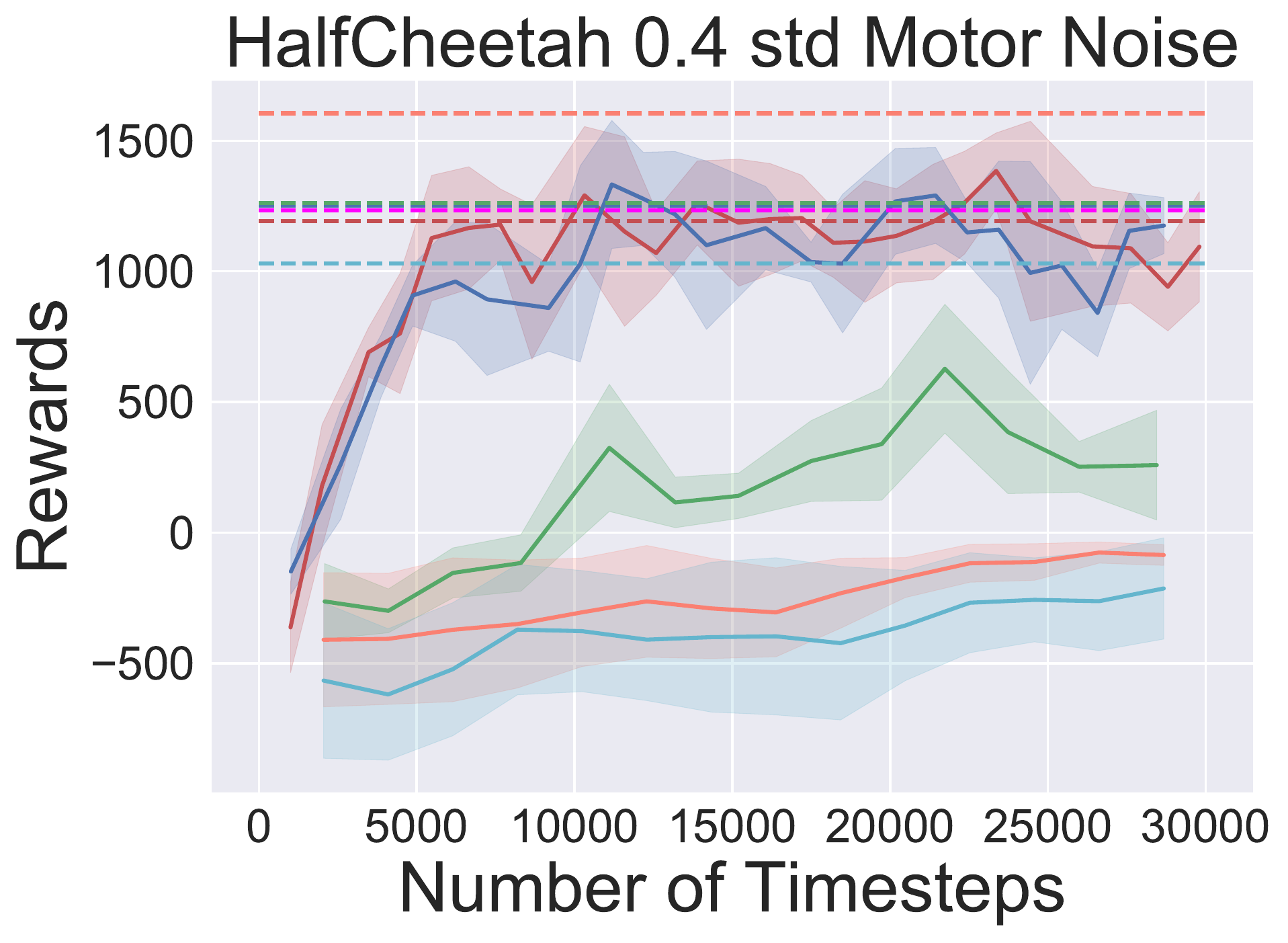}
	\includegraphics[width = 0.245\textwidth, trim=0cm 0cm 0cm 0cm, clip=true]{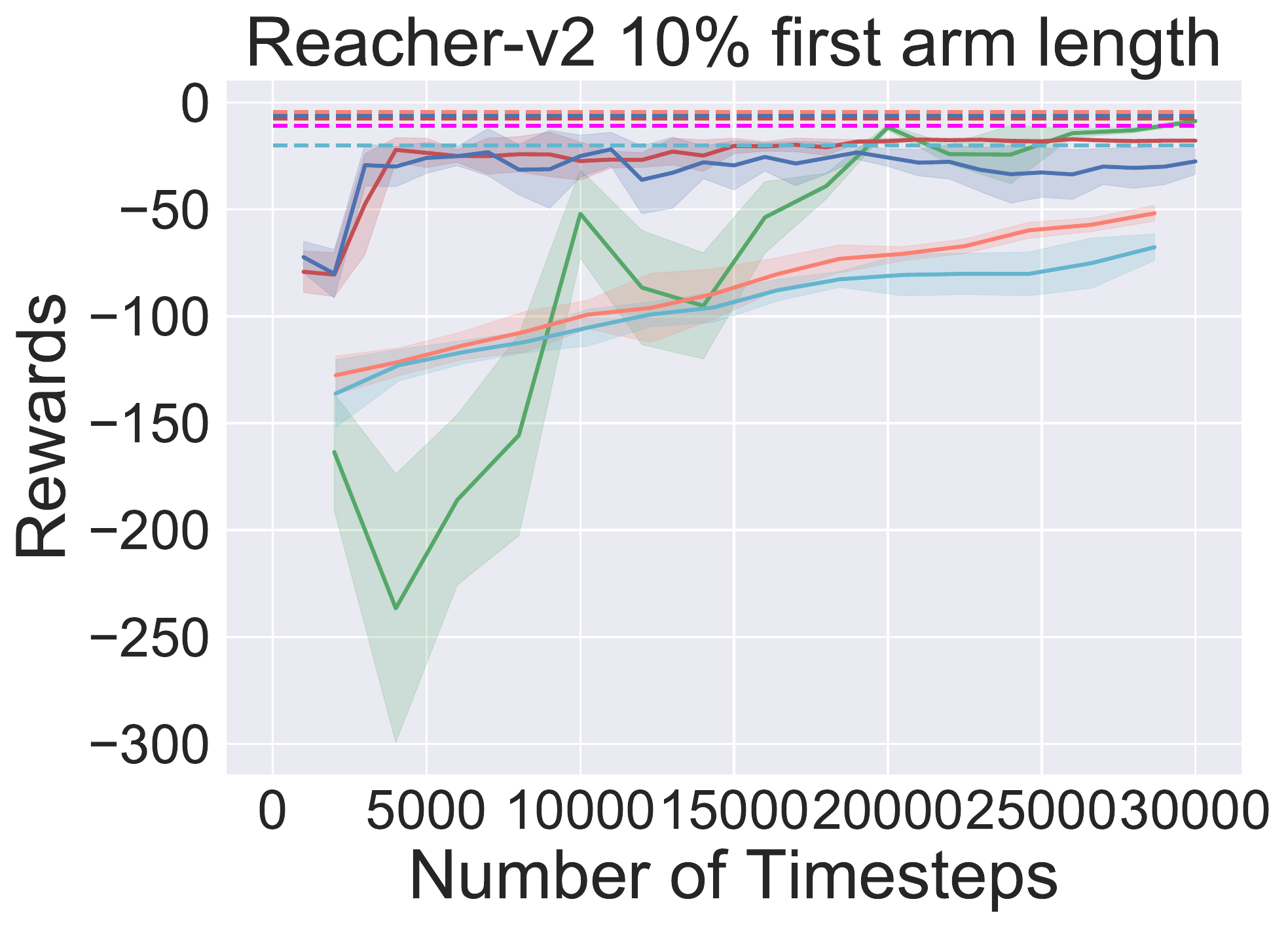}

	\centering
	\includegraphics[width = 0.245\textwidth, trim=0cm 0cm 0cm 0cm, clip=true]{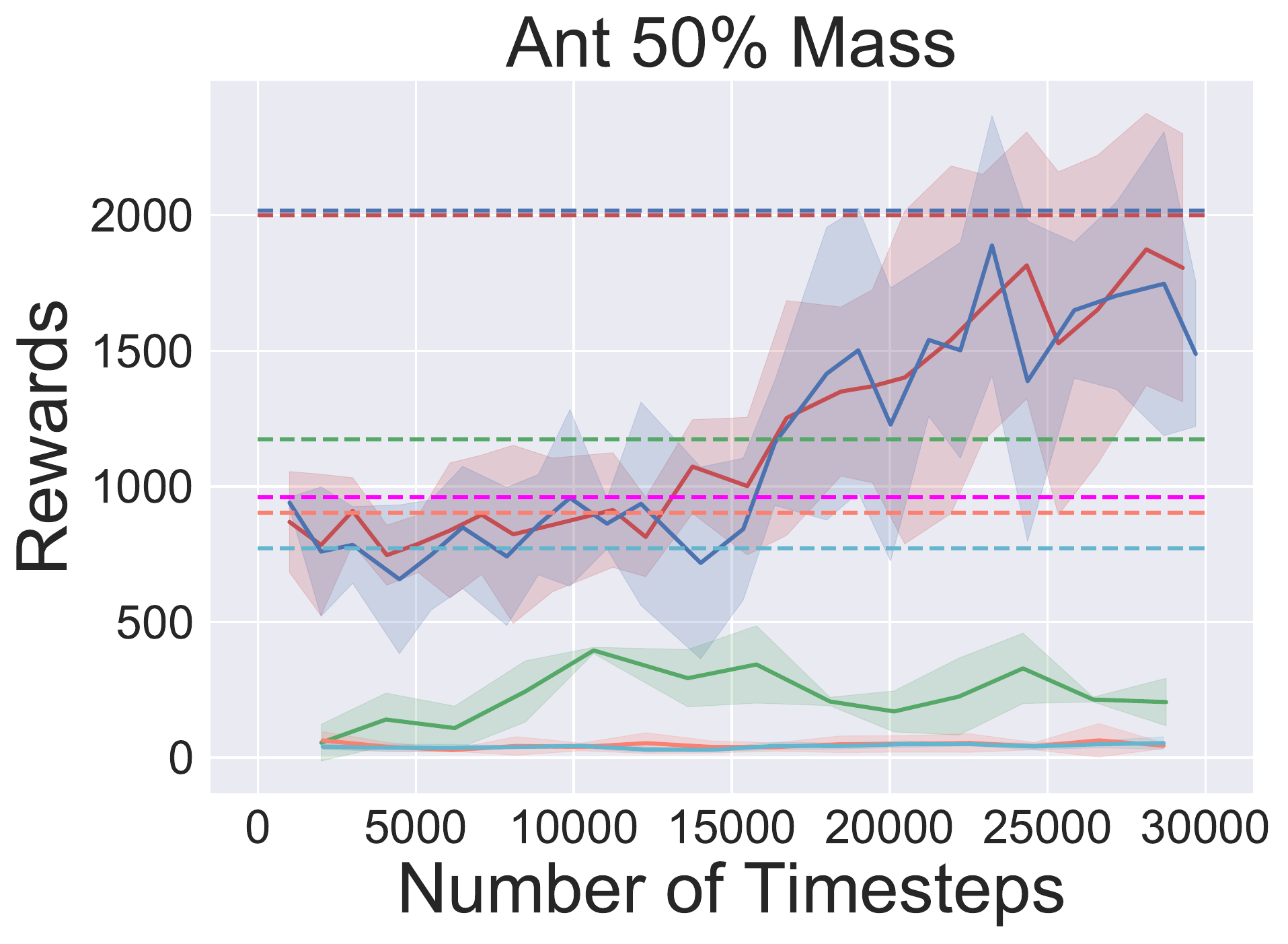}
	\includegraphics[width = 0.245\textwidth, trim=0cm 0cm 0cm 0cm, clip=true]{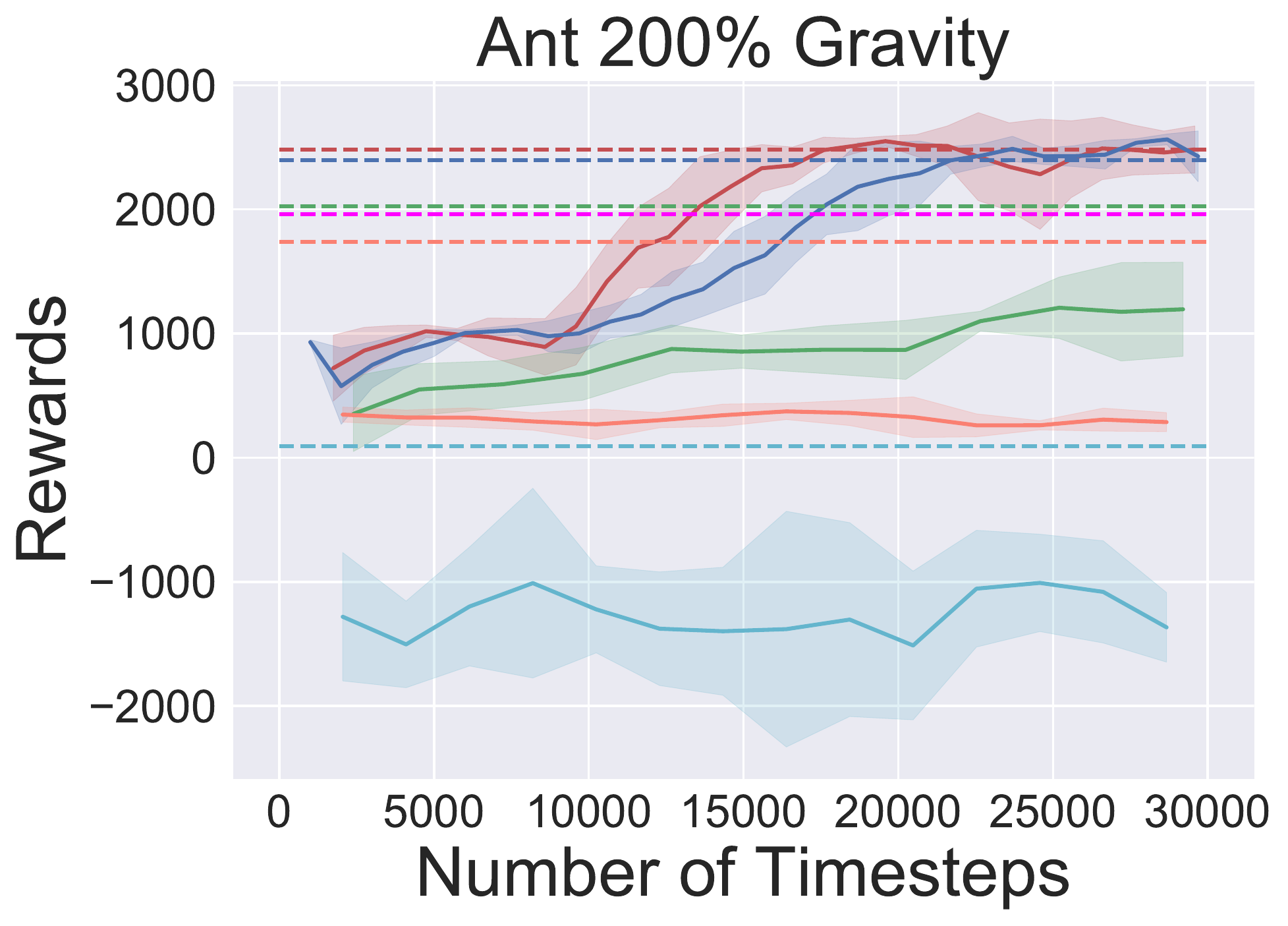}
	\includegraphics[width = 0.245\textwidth, trim=0cm 0cm 0cm 0cm, clip=true]{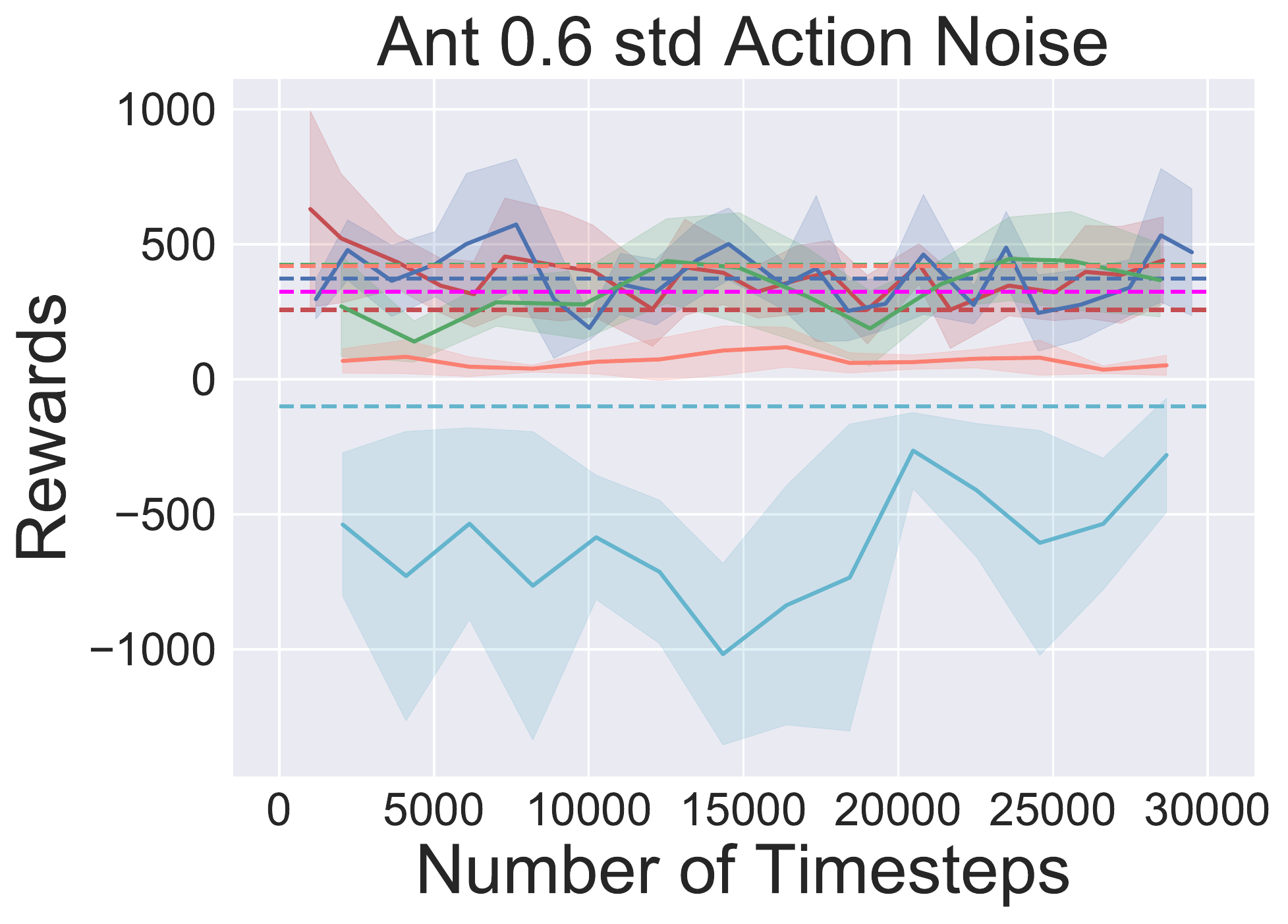}
	\includegraphics[width = 0.245\textwidth, trim=0cm 0cm 0cm 0cm, clip=true]{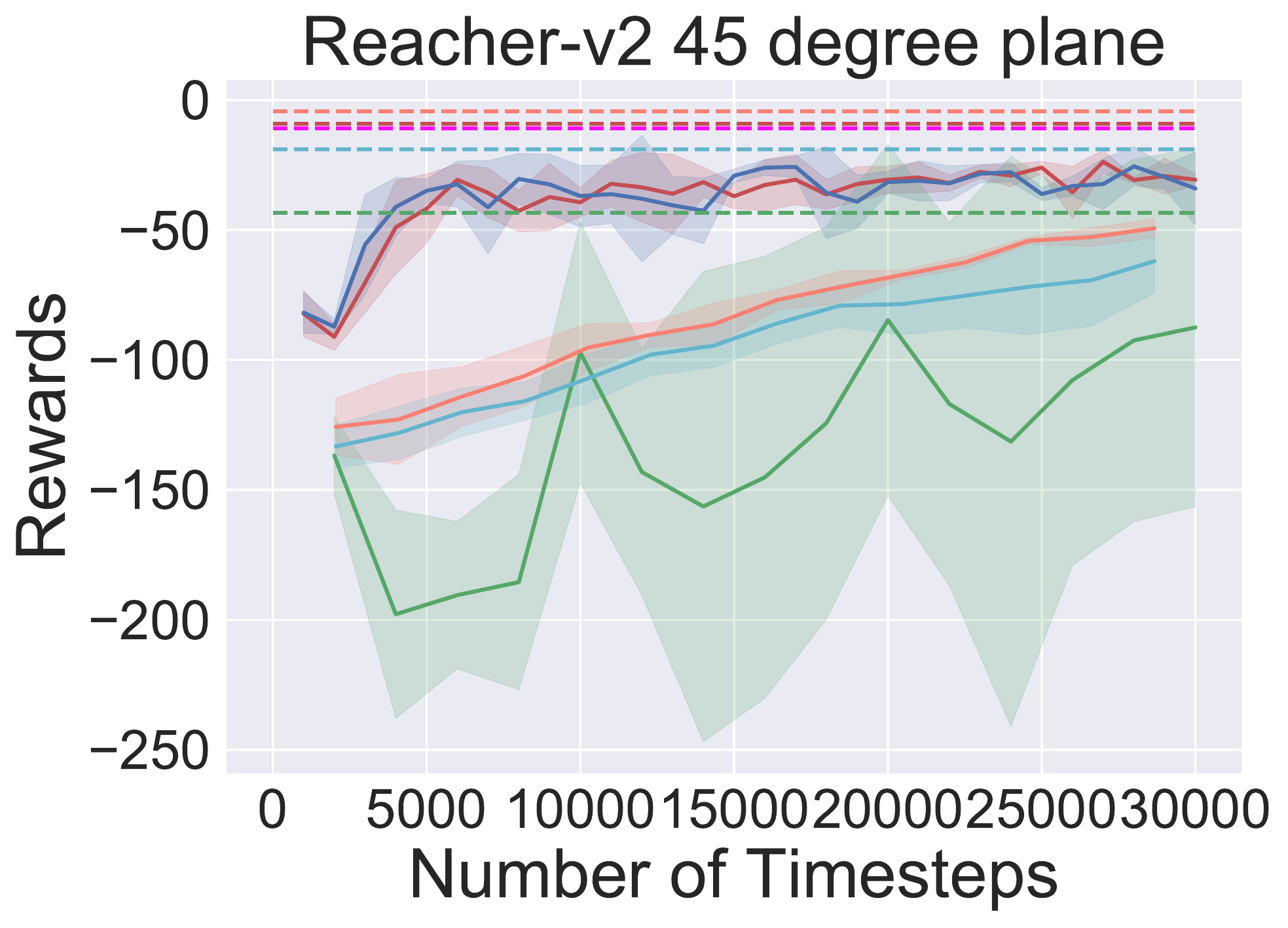}
	
	\centering
	\includegraphics[width = 1\textwidth, trim=0cm 0cm 0cm 0cm, clip=true]{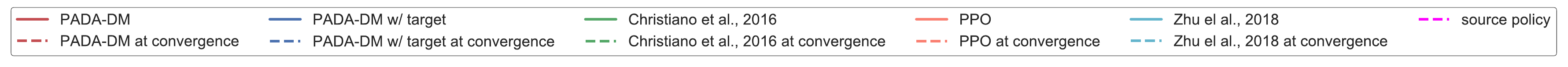}
  \caption{We plot the learning curves across 5 random seeds on a number of tasks. The title of each plot correspounds to the perturbation in the target domain, e.g., HalfCheetah Mass 150\% means in mass of the agent in the target domain is 150\% of that in the source domain.
  }
	\label{performance}
\end{figure*}

\section{Experiments} \label{sec:experiments}
In this section we compare our approach with the state-of-the-art methods for policy adaptation~\cite{christiano2016transfer,zhu2018reinforcement,ppo,finn2017model} and show that we achieve competitive results more efficiently. 
We also test the robustness of the approaches on multi-dimensional perturbations. We then compare to domain randomization and meta-learning approaches and show how they can be combined with our approach. We provide further experiments in Appendix \ref{appendix:sup_exp}. 

Following the same experiment setup as~\cite{christiano2016transfer}, We focus on standard OpenAI Gym~\cite{gym} and Mujoco \cite{mujoco} control environments such as HalfCheetah, Ant, and Reacher. We perturb the environments by changing their parameters such as mass, gravity, dimensions, motor noise, and friction. More details of task designs are in Appendix~\ref{appendix:env_des}. 

\subsection{Comparison with Existing Approaches}

We compare our methods (PADA-DM, PADA-DM with target policy) with the following state-of-the-art methods for policy adaptation. The names correspond to the learning curves shown in Figure~\ref{performance}. 

\textbf{Christiano et al., 2016:}  \cite{christiano2016transfer} uses a pre-trained policy $\pi\sups$ and source dynamics $f\sups$, to learn an inverse dynamics model $\phi \colon \mathcal{A} \times \mathcal{S} \times \mathcal{S} \rightarrow [0,1]$, where $\phi(a|s,s')$ is the probability of taking action $a$ that leads to $s'$ from $s$.\footnote{In general an inverse model is ill-defined without first specifying a policy, i.e., via Bayes rule, $\phi(a|s,s') \propto P(a, s, s') = f\supt(s'|s,a) \pi(a|s)$. Hence one needs to first specify $\pi$ in order to justify the existence of an inverse model $\phi(a|s,s')$.} That is, $\pi\supt(s) = \phi(s, f\sups(s,\pi\sups(s)))$.

\textbf{Zhu et al., 2018:} 
\cite{zhu2018reinforcement} 
proposed an approach for training policies in the target domain with a new reward 
$\lambda R(s_h) + (1-\lambda)R_{gail}(s_h,a_h)$, $\lambda \in [0,1]$. Here $R_{gail}$ is from the discriminator from GAIL. Note that this baseline has access to true reward signals while ours do not. 

For additional baselines, we also show the performance of directly running Proximal Policy Optimization ({\bf PPO}) \cite{ppo} in the target environment, as well as directly using the \textbf{source policy} in the perturbed environment without adapation. 

\begin{figure}[H]

	\centering
	\includegraphics[width = 0.23\textwidth, trim=0cm 0cm 0cm 0cm, clip=true]{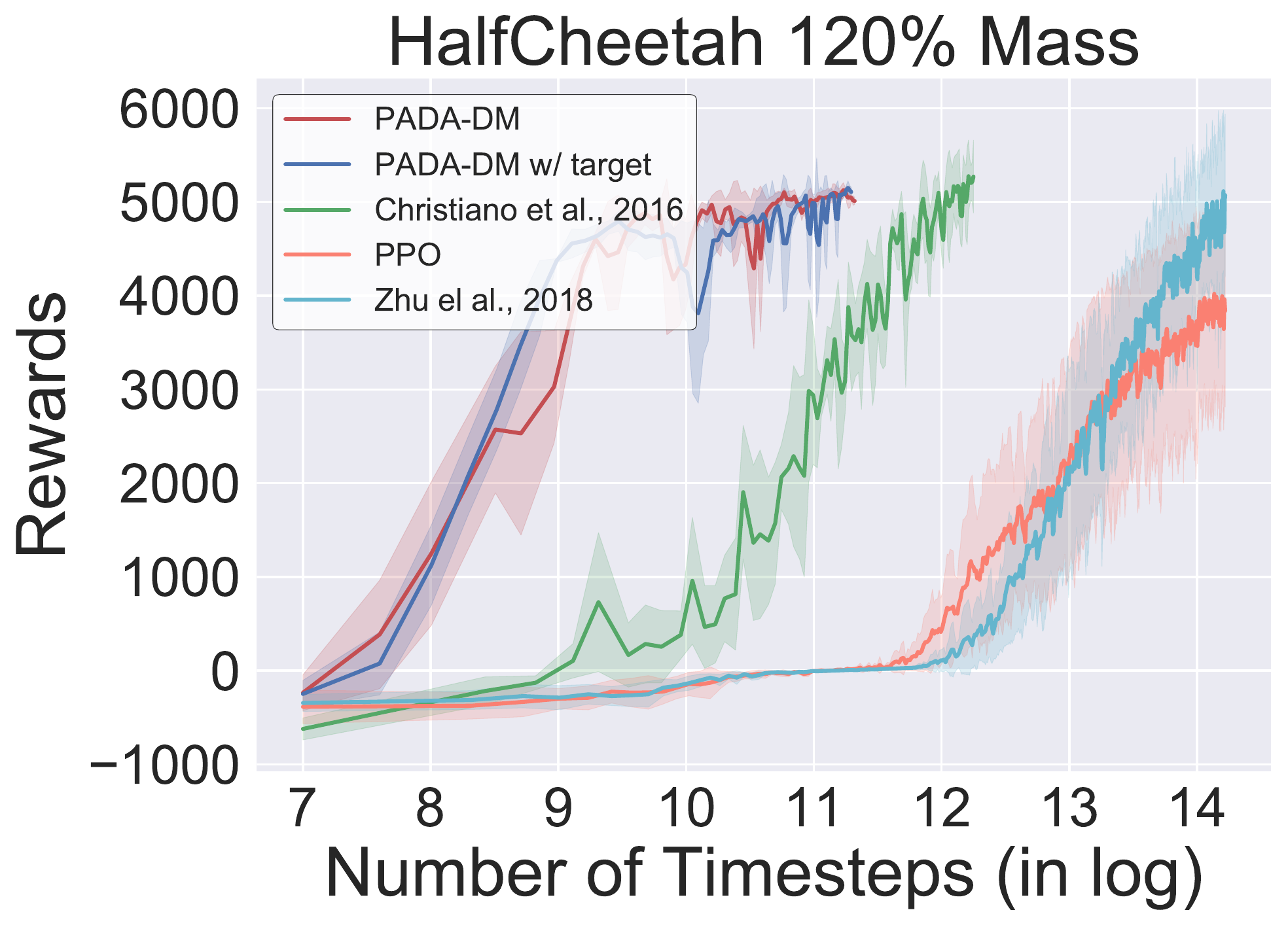}
	\includegraphics[width = 0.23\textwidth, trim=0cm 0cm 0cm 0cm, clip=true]{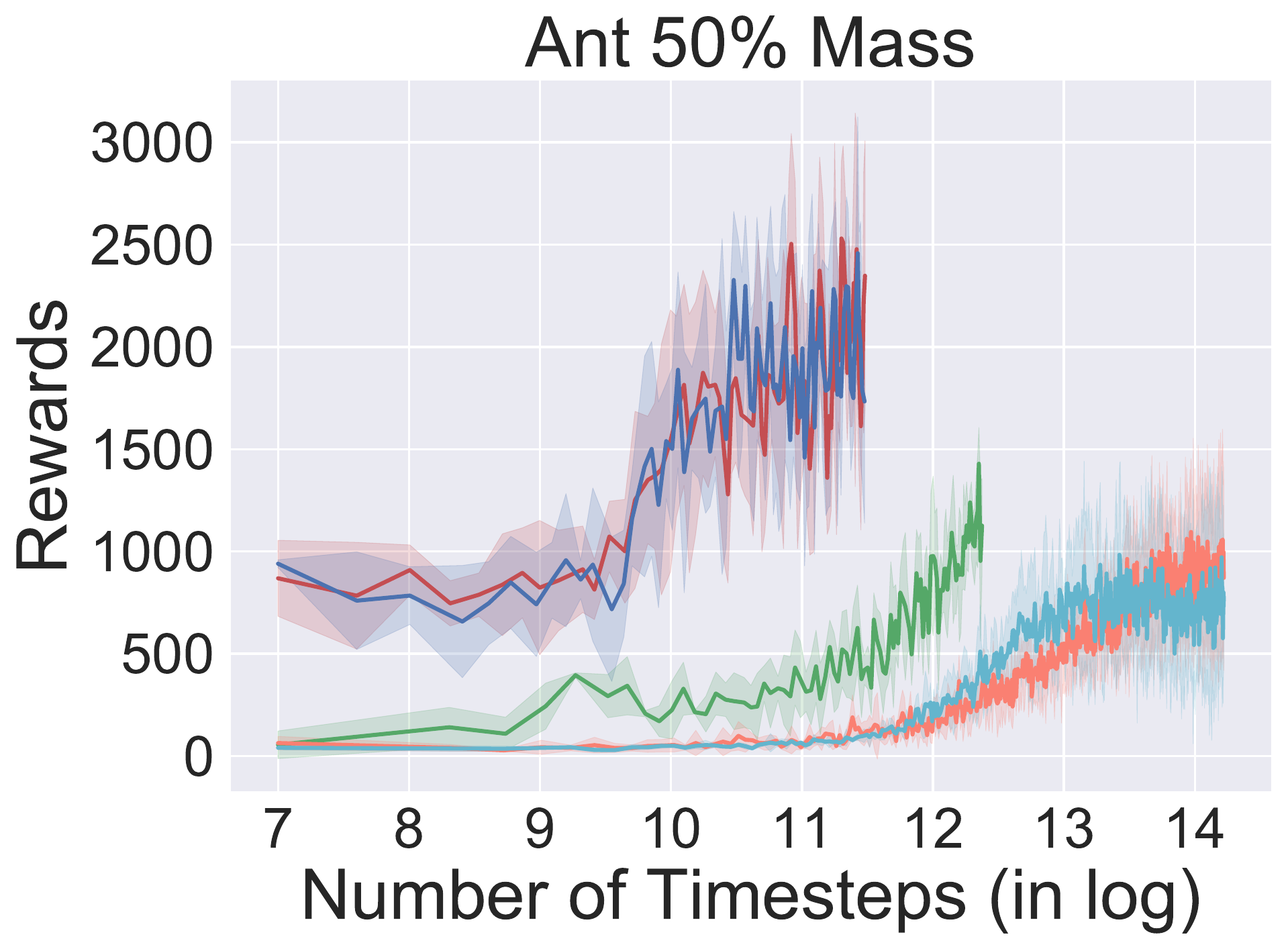}
	
	\caption{The long-term learning curves. The x-axis is the number of timesteps in natural logarithm scale. }
	\label{long_plots}
\end{figure}

\textbf{Results.} Figure~\ref{performance} demonstrates the sample efficiency of our methods compared to the other methods and baselines. Both PADA-DM and PADA-DM with target policy converge within 10k to 50k training samples in the target environments. In contrast, \cite{christiano2016transfer} requires 5 times more samples than our methods on average, and \cite{zhu2018reinforcement} and PPO require about 30 times more. At convergence, our methods obtain the highest episodic rewards in 7 out of 8 tasks above among the policy adaptation methods. The baseline performance of PPO is better than the policy adaptation methods in HalfCheetah and Reacher (recall PPO uses true reward signals), but it takes significantly longer as shown in Fig.~\ref{performance}. Note that in the Ant environment, even at convergence our methods outperform PPO as well.  

The only task where our methods failed to achieve top performance is Ant-v2 0.6 std motor noise. In this environment, the action noise causes high divergence between the target and source environments, making it hard to efficiently model the domain divergence. All the adaptation methods deliver bad performance in this case, indicating the difficulty of the task. 

We observe that the learning curves of PADA-DM and PADA-DM with target policy are similar across all tasks without sacrificing efficiency or performance. The target policy can be directly used without any MPC step.

To further illustrate the sample efficiency of our method, we compare the long-term learning curves in Fig. \ref{long_plots}. We plot the learning curves up to convergence of each method. We further include a long-term version of Fig \ref{performance} and the hyperparameters in the Appendix.

\subsection{Performance on Multi-Dimensional Perturbations}

We further evaluate the robustness of our methods by perturbing multiple dimensions of the target environment (Fig. \ref{vector}). Note that online adaptation is particularly useful for multiple-dimension perturbations, because they generate an exponentially large space of source environments that are hard to sample offline. In Fig. \ref{vector}(b), we show that even when perturbing 15 different degrees of freedom of the target environment, our adaptation method can still achieve competitive performance at a much faster rate than all the other methods. We record the details of the configurations of the target environments in Appendix \ref{app-multi-design}.
\vspace*{-\baselineskip}

\begin{figure}[H]
	\begin{tabular}{@{}l@{}}
		\subfigure[]{\includegraphics[width = 0.23\textwidth]{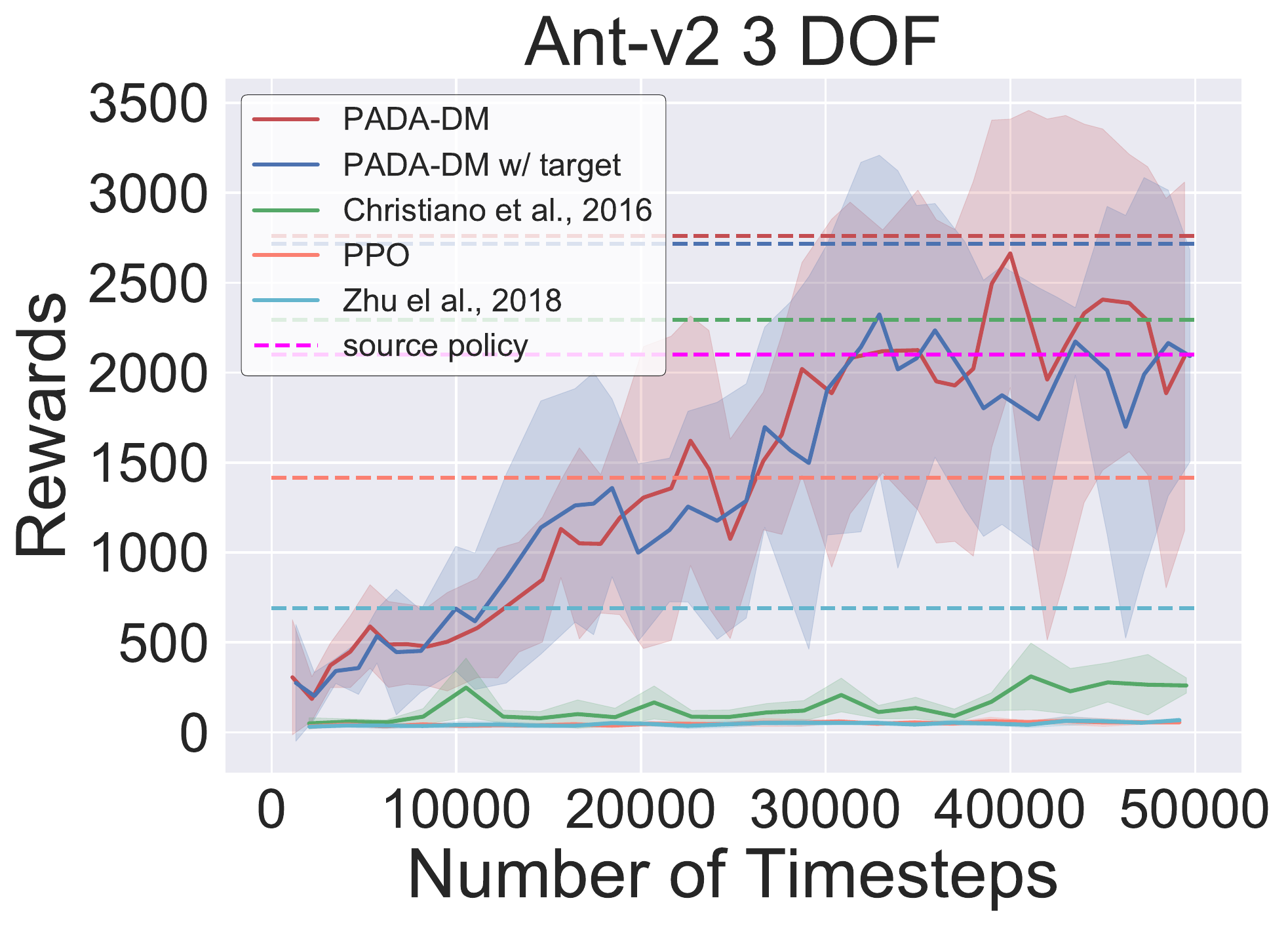}}
		\subfigure[]{\includegraphics[width = 0.23\textwidth]{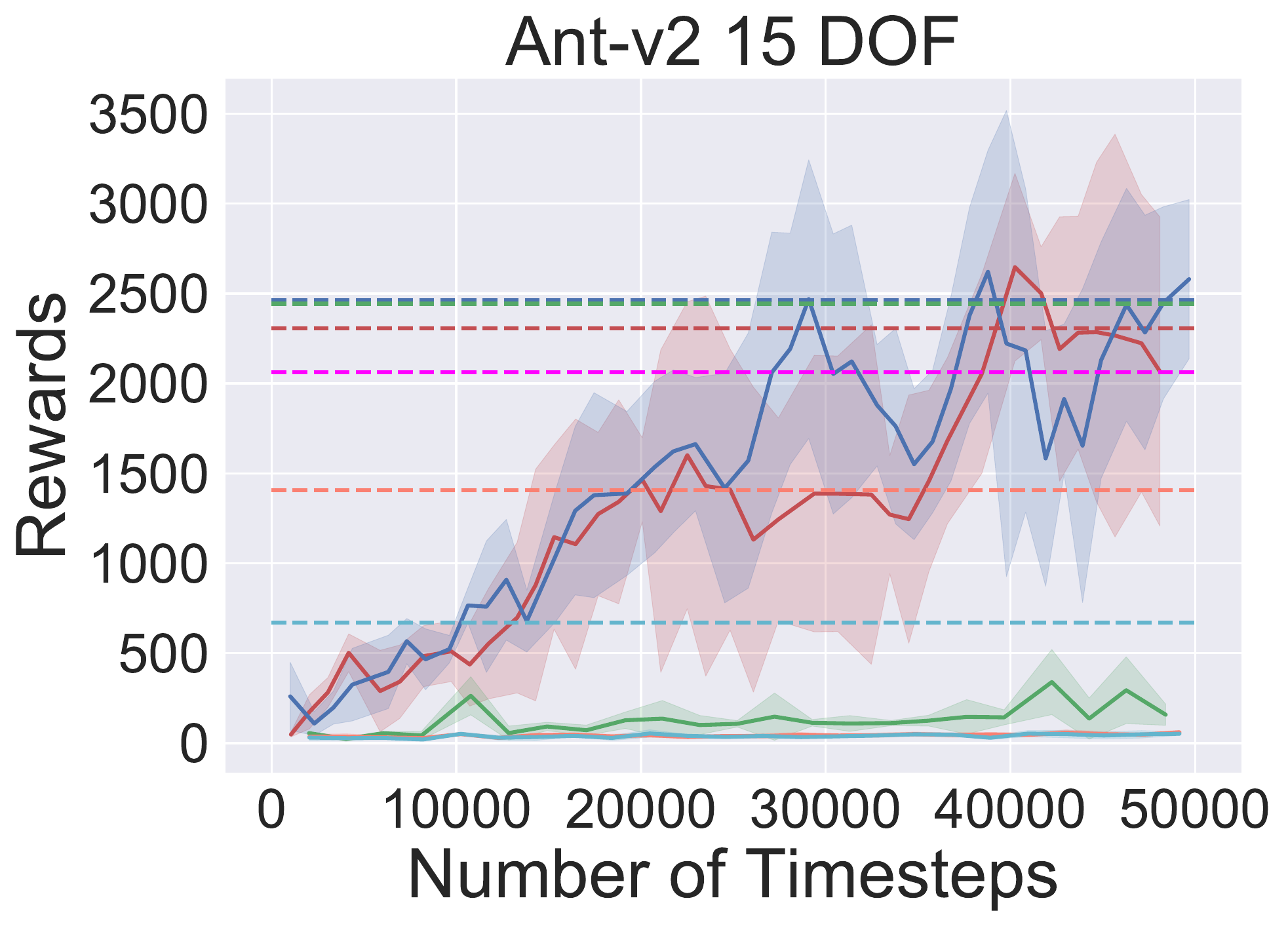}}
	\end{tabular}
	\caption{Learning curves for: changing 3 (a) and 15 (b) configurations in the target environment.}
	\label{vector}
\end{figure}
\subsection{Comparison with Domain Randomization and Meta-Learning} \label{exp:drmaml}
We now compare with domain randomization and meta-learning approaches and show how they can be combined with our methods. 
\begin{table}[h]
\small
\centering
\begin{tabular}{|c|c|c|}
\hline
& \begin{tabular}{@{}c@{}}single \\ source domain\end{tabular} & \begin{tabular}{@{}c@{}}randomized \\ source domains\end{tabular}      \\ \hline
\begin{tabular}{@{}c@{}}no adaptation in\\  target domain\end{tabular}  & source policy        & DR       \\ \hline
\begin{tabular}{@{}c@{}}adaptation in\\  target domain\end{tabular}     & PADA-DM              & \begin{tabular}{@{}c@{}}PADA-DM w/ DR; \\ MAML\end{tabular}  \\ \hline
\end{tabular}
\caption{Relationship of 5 methods in our ablation study.}
\label{trel}
\end{table}

\textbf{Domain randomization (DR)}~\cite{peng2018sim}: During the training in the source domain, at the beginning of each episode, we randomly sample a new configuration for the source domain within a specified distribution. The total number of training samples here is the same as that for training the source policy. The policy outputed from DR is used in the target environment without further adaptation.  

\textbf{MAML:} We adopt the RL version of MAML \cite{finn2017model} that meta-learns a model-free policy over a distribution of source environments and performs few-shot adaptation on the target environment.

We compare the following methods: (1) source policy trained in a fixed source environment, (2) domain randomization, (3) PADA-DM, (4) PADA-DM with DR (using domain randomization's output as $\pi\sups$ for PADA), 
and (5) MAML. 
Table \ref{trel} shows how these methods relate to each other. 

For the first four methods, we train the source policy with 2m samples and perform adaptation with 80k samples.  For MAML, we use 500 batches of meta-training (400m samples), and adapt 10 times with 80k samples in the target domain. We perform 100 trajectories across 5 random seeds in the target domain for each method and compare the episodic reward in Figure \ref{domainrand}. 
We first observe that when the target domains lie in the distribution of domain randomization ($70\% - 130\%$), domain randomization outperforms source policy significantly, but does not help when the target lies far away from the distribution, which is the most notable shortcoming of domain randomization. Note that using domain randomization in conjunction with our adaptation method usually yields the best results. Domain randomization often provides robustness within the environments'  distribution, and online adaptation in real target environment using our approach further ensures robustness to out-of-distribution environments. We also observe that our method provides the most stable performance  given the smallest test variances. We include additional experiments and detailed numbers of the performances of all methods (mean and standard deviations) in Appendix \ref{app_exp_dr}.

\begin{figure}[H]
	\begin{tabular}{@{}l@{}}
	\subfigure[]{\includegraphics[width = 0.22\textwidth]{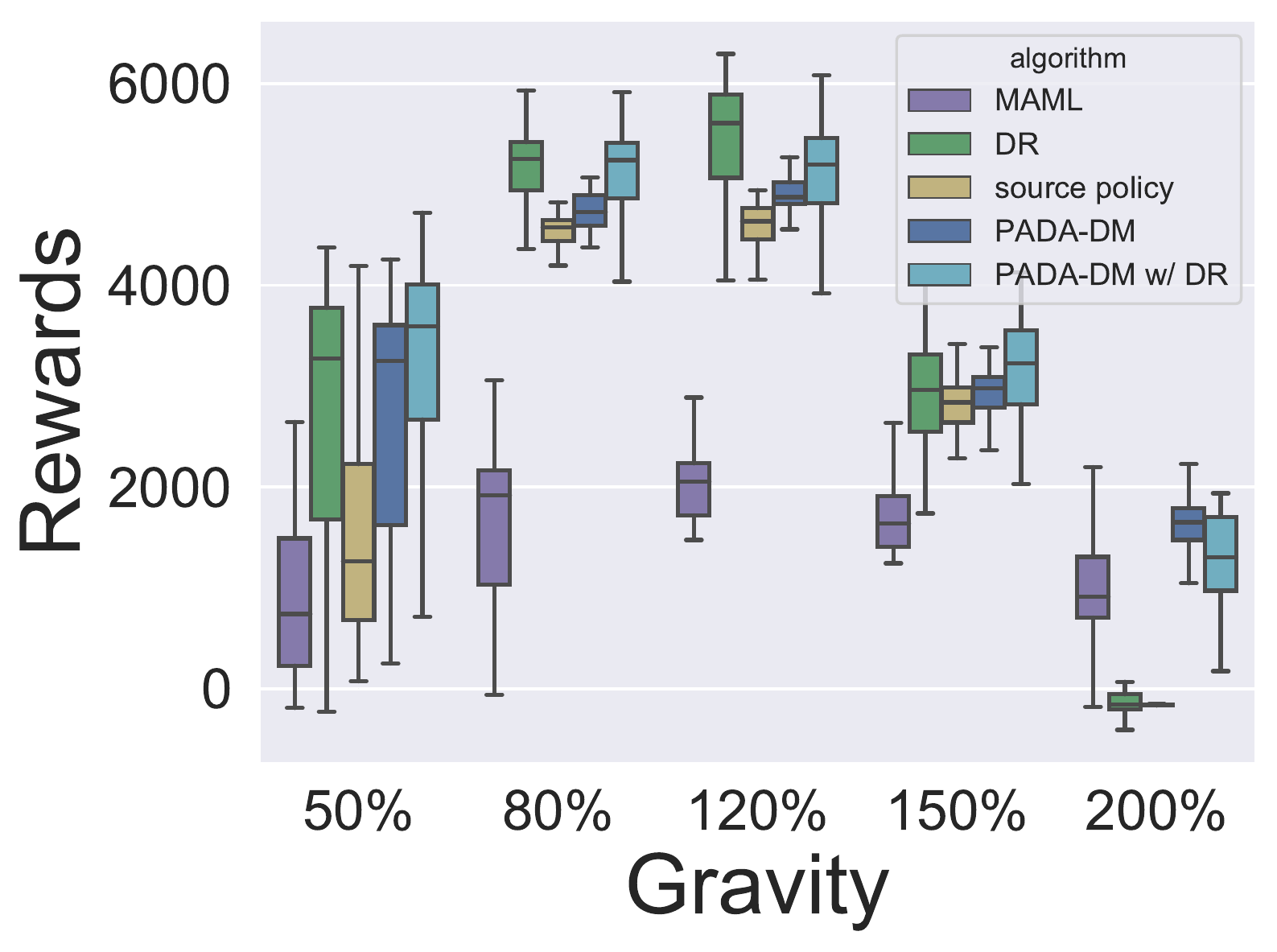}}
	\subfigure[]{\includegraphics[width = 0.22\textwidth]{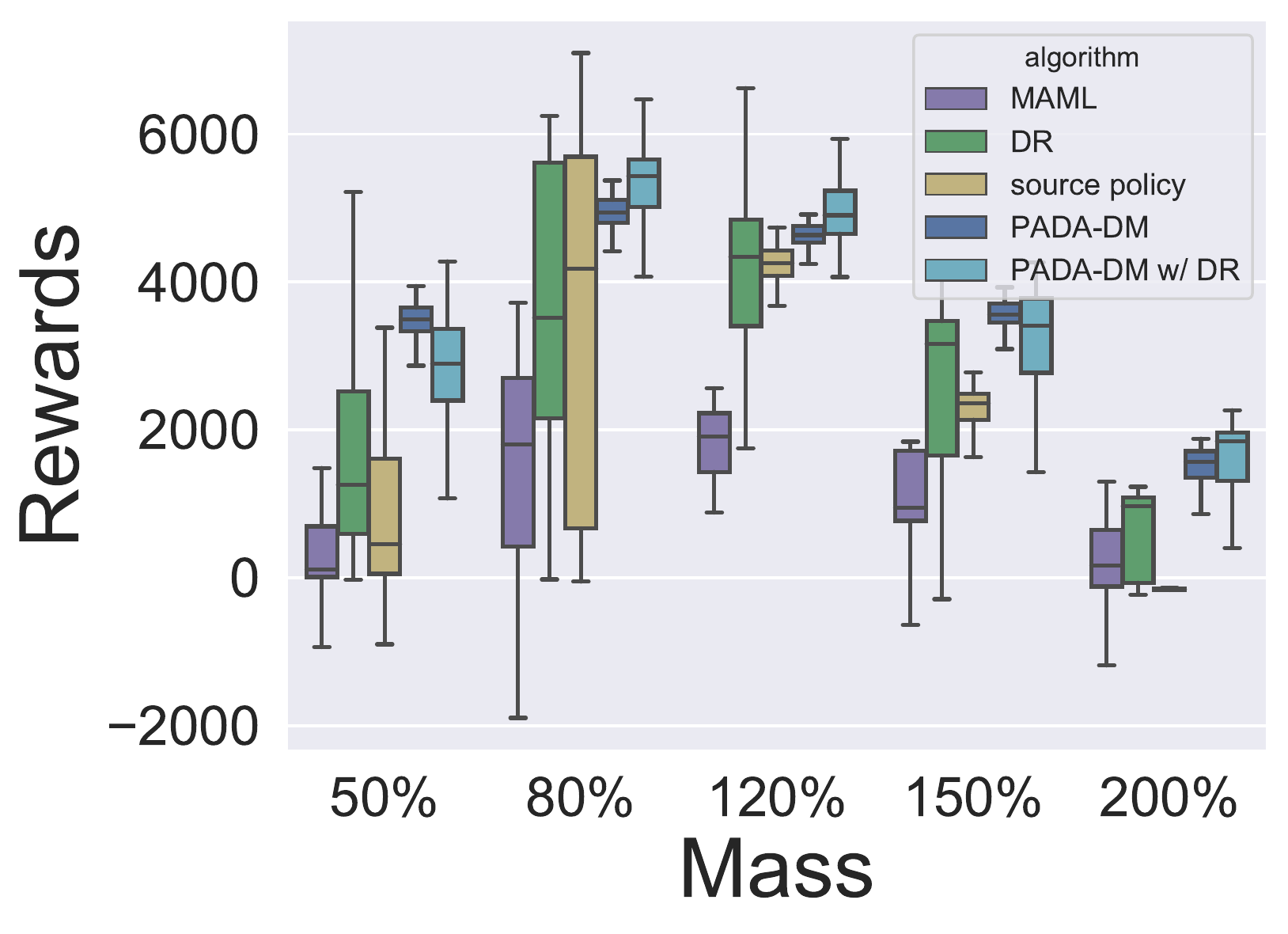}}
    \end{tabular}
	\caption{Ablation experiments using domain randomization and meta-learning. (a) Varying gravity. (b) Varying mass.} 
	\label{domainrand}
\end{figure}
\section{Conclusion}
We proposed a novel policy adaptation algorithm that combines techniques from model-based RL and no-regret online learning. We theoretically proved that our methods generate trajectories in the target environment that converge to those in the source environment. We established the rate of convergence of the algorithms. We have shown that our algorithm achieves competitive performance across a diverse set of continuous control tasks with better sample efficiency. A natural extension is to use our approach on simulation-to-real problems in combination with domain randomization and meta-learning. 

As our experiments indicated that the combination of domain randomization and our online adaptation approach together often yields good results,  for future work, we plan to investigate general theoretical framework for combining domain randomization and online adaptive control techniques.

\bibliography{refs}
\bibliographystyle{icml2020}

\newpage
\onecolumn
\appendix

\section{Detailed Analysis of Algorithm~\ref{alg:alg_theory} (Proof of Theorem~\ref{them:main})}
\label{app:proof_of_main_theorem}

In this section, we provide detailed proof for Theorem~\ref{them:main}.

Consider a Markov Chain $p: \calS \to \Delta(\calS)$ over horizon $H$. Denote $d_{p}$ as the induced state distribution under $p$ and $\rho_{p}$ as the induced state trajectory distribution under $p$.  

\begin{lemma}
Consider two Markov Chains $p_i:\mathcal{S}\to \Delta(\mathcal{S})$ with $i\in \{1,2\}$. If 
\begin{align*}
    \mathbb{E}_{s\sim d_{p_1}} \left[\| p_1(\cdot|s) - p_2(\cdot|s)\|\right] \leq \delta,
\end{align*} then for trajectory distributions, we have:
\begin{align*}
    \|\rho_{p_1} - \rho_{p_2}\| \leq O(H\delta).
\end{align*}
\label{lemma:divergence}
\end{lemma}
The above lemma implies that if the Markov chain $p_2$ can predict the Markov chain $p_1$ \emph{under the state distribution induced by $p_1$}, then we can guarantee that the state-wise trajectory distributions from $p_1$ and $p_2$ are close as well. 
\begin{proof}
Denote $\rho_{p,s_1,\dots s_h}$ as the trajectory distribution induced by $p$ conditioned on the first $h$ many states are equal to $s_1,\dots, s_h$. Denote $p(\cdot |s_0)$ as the initial state distribution for $s_1$ with $s_0$ being a faked state. 
By definition, we have:
\begin{align*}
&\| \rho_{p_1} - \rho_{p_2} \|  =  \sum_{\tau} \left\lvert \rho_{p_1}(\tau) - \rho_{p_2}(\tau)  \right\rvert  \\
& = \sum_{s_1,\dots s_H} \left\lvert \prod_{h=1}^H p_1(s_h|s_{h-1}) - \prod_{h=1}^H p_2(s_h | s_{h-1}) \right\rvert \\
& =  \sum_{s_1,\dots s_H} \left\lvert  \prod_{h=1}^H p_1(s_h|s_{h-1}) - p_1(s_1|s_{0})\prod_{h=2}^H p_2(s_h | s_{h-1}) + p_1(s_1|s_{0})\prod_{h=2}^H p_2(s_h | s_{h-1})  - \prod_{h=1}^H p_2(s_h|s_{h-1})\right\rvert   \\
& \leq \sum_{s_1} p_1(s_1|s_0)  \sum_{s_2,\dots, s_{H}}\left\lvert  \prod_{h=2}^H p_1(s_h|s_{h-1}) - \prod_{h=2}^H p_2(s_h|s_{h-1})  \right\rvert + \sum_{s_1} \left\lvert  p_1(s_1|s_0) - p_2(s_1|s_0)\right\rvert \left(\sum_{s_2,\dots,s_H} \prod_{h=2}^H p_2(s_h|s_{h-1})\right)\\
& = \Exp_{s_1\sim p_1} \left[  \| \rho_{p_1, s_1} - \rho_{p_2, s_1}\|  \right]  + \| p_1(\cdot |s_0) - p_2(\cdot | s_0)\| \\
& \leq \Exp_{s_1,s_2\sim p_1} \left[ \| \rho_{p_1,s_1,s_2} - \rho_{p_2,s_1,s_2} \| \right] + \| p_1(\cdot |s_0) - p_2(\cdot | s_0)\| + \Exp_{s_1\sim d_{\pi_1; 1}}[ \| p_1(\cdot|s_1) - p_2(\cdot|s_1) \|].
\end{align*}  Recursively applying the same operation on $ \| \rho_{p_1, s_1} - \rho_{p_2, s_1}\|$ to time step $H$, we have:
\begin{align*}
&\| \rho_{p_1} - \rho_{p_2} \| \leq \sum_{h=1}^H  \Exp_{s_h\sim d_{p_1;h}} \left[ \| p_1(\cdot|s_h) - p_2(\cdot |s_h) \|\right] \leq H \delta,
\end{align*} where we recall $d_{\pi} = \sum_{h=1}^H d_{\pi;h} / H$ by definition. Extension to continuous state can be achieved by simply replaying summation by integration. 

\end{proof}

The next lemma shows that by leveraging the no-regret property of FTL, we can learn a locally accurate model. 
\begin{lemma}[Local Accuracy of the Learned Model]
\label{lemma:local_model}
Denote the sequence of models learned in Alg.~\ref{alg:alg_theory} as $\{\hat{f}_1, \dots, \hat{f}_T\}$, there exists a model $\hat{f} \in \{\hat{f}_1,\dots, \hat{f}_T\}$ such that:
\begin{align*}
    \mathbb{E}_{s\sim d_{\pi_{\hat{f}}}}\left[\mathbb{E}_{a\sim U(\mathcal{A}\supt)}\left[ D_{KL}(f\supt(\cdot|s,a), \hat{f}(\cdot|s,a)) \right]\right] \leq  O(1/T),
\end{align*} where $\pi_{{f}}(s)\triangleq \argmin_{a\in\mathcal{A}\supt} \| {f}(\cdot|s,a) - f\sups(\cdot | s, \pi\sups(s))$ for all $s\in\calS$ for any $f$.
\end{lemma}

\begin{proof}
Denote loss function $\ell_e(f)$ as:
\begin{align*}
    \ell_e(f) \triangleq  \mathbb{E}_{s\sim d_{\pi\supt_e} ,a\sim U(\mathcal{A}\supt)}\left[\mathbb{E}_{s'\sim f\supt_{s,a}}\left[ -\log( f(s' |s,a)) \right]\right].
\end{align*} 
Since Alg.~\ref{alg:alg_theory} is equivalent to running FTL on the sequence of strongly convex loss functions $\{\ell_e(f)\}_{e=1}^T$, we have \cite{shalev2012online}: 
\begin{align*}
    \sum_{e=1}^T \ell_e(\hat{f}_e) \leq \min_{f \in\mathcal{F}} \sum_{e=1}^T \ell_e(f) + O\left(\log{T}\right).
\end{align*}
Add $\sum_{e=1}^T \mathbb{E}_{s\sim d_{\pi\supt_e} ,a\sim U(\mathcal{A}\supt)}[\mathbb{E}_{s'\sim f\supt_{s,a}}\log f\supt(s'|s,a)]$ on both sides of the above inequality and using the definition of KL divergence, we have:
\begin{align*}
    &\sum_{e=1}^T \mathbb{E}_{s\sim d_{\pi\supt_e} ,a\sim U(\mathcal{A}\supt)}\left[ D_{KL}(f\supt(\cdot|s,a), \hat{f}_e(\cdot|s,a)) \right] \\
    & ~~ \leq \min_{f\in\mathcal{F}}\sum_{e=1}^T \mathbb{E}_{s\sim d_{\pi\supt_e} ,a\sim U(\mathcal{A}\supt)}\left[ D_{KL}(f\supt(\cdot|s,a), f(\cdot|s,a)) \right] + O(\log(T)) = O(\log(T)),
\end{align*} where the last equality comes from the realizability assumption $f\supt\in\mathcal{F}$.

Using the fact that the minimum among a sequence is less than the average of the sequence, we arrive:
\begin{align*}
    & \min_{\hat{f}\in \{\hat{f}_e\}_{e=1}^T}\mathbb{E}_{s\sim d_{\pi_{\hat{f}}},a\sim U(\mathcal{A}\supt)}\left[ D_{KL}(f\supt(\cdot|s,a), \hat{f}(\cdot|s, a)) \right]  \leq \tilde{O}(1/T).
\end{align*}
\end{proof}
The above lemma indicates that as $T\to \infty$,  we will learn a model $\hat{f}$ which is close to the target true model $f\supt$ under the state distribution induced by $\pi_{\hat{f}}$. But it does not state the difference between the behavior generated by $\pi_{\hat{f}}$ at the target domain and the behavior generated by $\pi\sups$ at the source domain. The next lemma uses the definition $\pi_{\hat{f}}$ to show that when executing $\pi_{\hat{f}}$ in the target domain $\calM\supt$, $\pi_{\hat{f}}$ can actually generates behavior that is similar to the behavior generated by $\pi\sups$ in the source domain $\calM\sups$. 

\begin{lemma}[The Behavior of $\pi_{\hat{f}}$] Denote $d_{\pi_{\hat{f}}}$ as the state distribution induced by $\pi_{\hat{f}}$ induced at $\calM\supt$ (target domain). 
\begin{align*}
&\mathbb{E}_{s\sim d_{\pi_{\hat{f}}}} \| f\supt(\cdot|s, \pi_{\hat{f}}(s)) - f\sups(\cdot|s,\pi\sups(s))\| \leq  O(|\mathcal{A}\supt |/\sqrt{T}) + \mathbb{E}_{s\sim d_{\pi_{\hat{f}}}}\left[\epsilon_{s, \pi\sups(s) }  \right],
\end{align*} where we recall the definition of $\epsilon$ in Assumption~\ref{assum:similarity}.
\label{lemma:behavior}
\end{lemma}
\begin{proof}
Consider the Markov chain that is defined with respect to $f\supt$ and $\pi_{\hat{f}}$, i.e., $f\supt(s' |s, \pi_{\hat{f}}(s))$. Denote the state distribution induced by $f\supt(s'|s,\pi_{\hat{f}}(s) )$ at the target domain as $d_{\pi_{\hat{f}}}$. Let us bound ${\Exp}_{s\sim d_{\pi_{\hat{f}}}} \| f\supt(\cdot|s, \pi_{\hat{f}}(s)) - f\sups(\cdot|s,\pi\sups(s))\|$.
\begin{align*}
&\mathbb{E}_{s\sim d_{\pi_{\hat{f}}}} \| f\supt(\cdot|s, \pi_{\hat{f}}(s)) - f\sups(\cdot|s,\pi\sups(s))\| \\
& \leq \mathbb{E}_{s\sim d_{\pi_{\hat{f}}}}\|f\supt(\cdot|s, \pi_{\hat{f}}(s)) - \hat{f}(\cdot|s, \pi_{\hat{f}}(s))  \| + \mathbb{E}_{s\sim d_{\pi_{\hat{f}}}}\|\hat{f}(\cdot|s, \pi_{\hat{f}}(s)) -  f\sups(\cdot|s,\pi\sups(s))  \|\\
& \leq \sqrt{\mathbb{E}_{s\sim d_{\pi_{\hat{f}}}} D_{KL}( f\supt(\cdot|s,\pi_{\hat{f}}(s)), \hat{f}(\cdot|s,\pi_{\hat{f}}(s)) )} + \mathbb{E}_{s\sim d_{\pi_{\hat{f}}}}\|\hat{f}(\cdot|s, \pi_{\hat{f}}(s)) - {f}\sups(\cdot|s, \pi\sups(s))  \|  \\
& \leq O(|\mathcal{A}\supt|/\sqrt{T}) + \mathbb{E}_{s\sim d_{\pi_{\hat{f}}}}\|\hat{f}(\cdot|s, \pi_{\hat{f}}(s)) -  f\sups(\cdot|s,\pi\sups(s))  \| \\
& \leq O(|\mathcal{A}\supt|/\sqrt{T}) + \mathbb{E}_{s\sim d_{\pi_{\hat{f}}}} \|\hat{f}(\cdot|s, \pi_{f\supt}(s)) - f\sups(\cdot|s,\pi\sups(s))\| \\
& \leq O(|\mathcal{A}\supt|/\sqrt{T}) + \mathbb{E}_{s\sim d_{\pi_{\hat{f}}}}  \|\hat{f}(\cdot|s, \pi_{f\supt}(s)) - f\supt(\cdot| s, \pi_{f\supt}(s))\|  \\
& \qquad + \mathbb{E}_{s\sim d_{\pi_{\hat{f}}}} \| f\supt(\cdot|s, \pi_{f\supt}(s)) - f\sups(\cdot|s,\pi\sups(s))\| \\
& \leq O(|\mathcal{A}\supt|/\sqrt{T}) + \sqrt{  \mathbb{E}_{s\sim d_{\pi_{\hat{f}}}}  D_{KL}(f\supt(\cdot|s, \pi_{f\supt}(s)), \hat{f}(\cdot|s, \pi_{f\supt}(s)))  } +  \mathbb{E}_{s\sim d_{\pi_{\hat{f}}}}\left[ \epsilon_{s, \pi\sups(s)} \right]  \\
& = O(|\mathcal{A}\supt |/\sqrt{T}) + \mathbb{E}_{s\sim d_{\pi_{\hat{f}}}}\left[ \epsilon_{s, \pi\sups(s)} \right],
\end{align*} where the first inequality uses triangle inequality, the second inequality uses Pinsker's inequality, the third inequality uses the local accuracy of the learned model $\hat{f}$ (Lemma~\ref{lemma:local_model}), the fourth inequality uses the definition that $\pi_{\hat{f}}$, and the fifth inequality uses triangle inequality again, and the sixth inequality uses Pinsker's inequality again with the definition of adaptive ability together with the definition of $\pi_{f\supt}(s) \triangleq \argmin_{a\sim \calA\supt} \| f\supt(\cdot|s,a) - f\sups(\cdot|s, \pi\sups(s)) \|$.
\end{proof}

\begin{proof}[Proof of Theorem~\ref{them:main}]
Use Lemma~\ref{lemma:divergence} and Lemma~\ref{lemma:behavior} and consider $f\sups(\cdot | s, \pi\sups(s))$ as a Markov chain, we have that:
\begin{align*}
    \| \rho^t_{\pi_{\hat{f}}} - \rho^s_{\pi_s} \| \leq O(H|\mathcal{A}\supt| / \sqrt{T}) + H\epsilon,
\end{align*} where we denote $\epsilon := \mathbb{E}_{s\sim d_{\pi_{\hat{f}}}}\left[ \epsilon_{s, \pi\sups(s)} \right]$
\end{proof}

This concludes the proof of Theorem~\ref{them:main}.

\subsection{Extension to Continuous Action Space (Proof of Corollary~\ref{cor:continuous_action})}
\label{app:analysis_continuous}

For simplicity, we consider $\mathcal{A}\supt = [0,1]$.\footnote{We can always normalize action to $[0,1]$. 
} We consider Lipschitz continuous transition dynamics with and only with respect to actions, i.e., 
\begin{align}
    \| f(\cdot | s, a) - f(\cdot|s,a')\| \leq L |a - a'|,
    \label{eq:lipschitz}
\end{align} where $L$ is a Lipschitz constant. We emphasize here that we only assume Lipschitz continuity with respect to action in $\calM\supt$. Hence this is a much weaker assumption than the common Lipschitz continuity assumption used in RL community, which requires Lipschitz continuity in both action and state spaces.  We also assume that our function class $\mathcal{F}$ only contains function approximators that are Lipschitz continuous with respect to action $a$ (e.g., feedforward fully connected ReLu network is Lipschitz continuous).

\begin{proof}[Proof of Corollary~\ref{cor:continuous_action}]
For analysis purpose, let us discretize the action space into bins with size $\delta \in (0,1)$. Denote the discrete action set $\bar{\mathcal{A}}\supt = \{0.5\delta, 1.5\delta, 2.5\delta, \dots, 1-0.5\delta\}$ (here we assume $1/\delta = \mathbb{N}^+$). Here $|\bar{\mathcal{A}}\supt| = 1/\delta$.

Now consider the following quantity:
\begin{align*}
    \mathbb{E}_{s\sim d_{\pi_{\hat{f}}}}\|f\supt(\cdot|s, \hat{a}) - \hat{f}(\cdot|s, \hat{a})\|,
\end{align*} for any $\hat{a}$. Without loss of generality, we assume $\hat{a}\in [0,\delta]$. Via Pinsker's inequality and Lemma~\ref{lemma:local_model}, we have:
\begin{align*}
    \mathbb{E}_{s\sim d_{\pi_{\hat{f}}}}\mathbb{E}_{a\sim U([0,1])}\|f\supt(\cdot|s, {a}) - \hat{f}(\cdot|s, {a})\| \leq O(1/\sqrt{T}),
\end{align*} which implies that:
\begin{align*}
  \mathbb{E}_{s\sim d_{\pi_{\hat{f}}}}\mathbb{E}_{a\sim U([0,\delta])}\|f\supt(\cdot|s, {a}) - \hat{f}(\cdot|s, {a})\| \leq O(1/(\delta\sqrt{T})) .
\end{align*} We proceed as follows:
\begin{align*}
  &\mathbb{E}_{s\sim d_{\pi_{\hat{f}}}}\mathbb{E}_{a\sim U([0,\delta])}\|f\supt(\cdot|s, {a}) - \hat{f}(\cdot|s, {a})\| \\
  & = \mathbb{E}_{s\sim d_{\pi_{\hat{f}}}}\mathbb{E}_{a\sim U([0,\delta])}\|f\supt(\cdot|s, \hat{a} + a - \hat{a}) - \hat{f}(\cdot|s, \hat{a} + a - \hat{a})\| \\
  & \geq \mathbb{E}_{s\sim d_{\pi_{\hat{f}}}}\mathbb{E}_{a\sim U([0,\delta])} \left(\| f\supt(\cdot|s, \hat{a}) - \hat{f}(\cdot|s, \hat{a})\| - 2L|a - \hat{a}|    \right) \\
  & =  \mathbb{E}_{s\sim d_{\pi_{\hat{f}}}}\|f\supt(\cdot|s, \hat{a}) - \hat{f}(\cdot|s, \hat{a})\| - 2L\mathbb{E}_{a\sim U([0,\delta])}|a-\hat{a}| \\
  & \geq \mathbb{E}_{s\sim d_{\pi_{\hat{f}}}}\|f\supt(\cdot|s, \hat{a}) - \hat{f}(\cdot|s, \hat{a})\| - 2L\mathbb{E}_{a\sim U([0,\delta])}\delta \\
  & = \mathbb{E}_{s\sim d_{\pi_{\hat{f}}}}\|f\supt(\cdot|s, \hat{a}) - \hat{f}(\cdot|s, \hat{a})\| - 2L\delta,
\end{align*} where the first inequality uses the fact that $\hat{a}\in [0,\delta]$ and the Lipschitz conditions on both $f\supt$ and $\hat{f}\in\mathcal{F}$, the second inequality uses the fact that $|a-\hat{a}| \leq \delta$ for any $a\in [0,\delta]$ as $\hat{a}\in [0,\delta]$.

Hence, we have:
\begin{align*}
   &\mathbb{E}_{s\sim d_{\pi_{\hat{f}}}} \| f\supt(\cdot|s, \hat{a}) - \hat{f}(\cdot|s, \hat{a})\| \leq 2L\delta + O(1/(\delta\sqrt{T})) = O(T^{-1/4}), \forall \hat{a} \in\calA\supt,
\end{align*} where in the last step we set $\delta = \Theta(T^{-1/4})$. 

Now, we can simply repeat the process we have for proving Lemma~\ref{lemma:behavior}, we will have:
\begin{align*}
    &\mathbb{E}_{s\sim d_{\pi_{\hat{f}}}} \| f\supt(\cdot|s, \pi_{\hat{f}}(s)) - f\sups(\cdot|s,\pi\sups(s))\| \leq  O(T^{-1/4}) + \epsilon.
\end{align*}

Combine with Lemma~\ref{lemma:divergence}, we prove the corollary for continuous action setting.
\end{proof}

For general $n$-dim action space, our  bound will scale in the order $O(\sqrt{n} T^{-1/(2n+2)})+\epsilon$.  The proof of the $n$-dim result is similar to the proof of the $1$-d case and is included below for completeness:
\begin{proof}[Proof for $n$-dim Action Space]
For $n$-dimensional action space, we have:
\begin{align*}
  \mathbb{E}_{s\sim d_{\pi_{\hat{f}}}}\mathbb{E}_{a\sim U([0,\delta]^n)}\|f\supt(\cdot|s, {a}) - \hat{f}(\cdot|s, {a})\| \leq O(1/(\delta^n\sqrt{T})) .
\end{align*}Using  the Lipschitz property: 
\begin{align*}
  &\mathbb{E}_{s\sim d_{\pi_{\hat{f}}}}\mathbb{E}_{a\sim U([0,\delta]^n)}\|f\supt(\cdot|s, {a}) - \hat{f}(\cdot|s, {a})\| \\
  & = \mathbb{E}_{s\sim d_{\pi_{\hat{f}}}}\mathbb{E}_{a\sim U([0,\delta]^n)}\|f\supt(\cdot|s, \hat{a} + a - \hat{a}) - \hat{f}(\cdot|s, \hat{a} + a - \hat{a})\| \\
  & \geq \mathbb{E}_{s\sim d_{\pi_{\hat{f}}}}\mathbb{E}_{a\sim U([0,\delta]^n)} \left(\| f\supt(\cdot|s, \hat{a}) - \hat{f}(\cdot|s, \hat{a})\| - 2L\|a - \hat{a}\|    \right) \\
  & =  \mathbb{E}_{s\sim d_{\pi_{\hat{f}}}}\|f\supt(\cdot|s, \hat{a}) - \hat{f}(\cdot|s, \hat{a})\| - 2L\mathbb{E}_{a\sim U([0,\delta]^n)}\|a-\hat{a}\| \\
  & \geq \mathbb{E}_{s\sim d_{\pi_{\hat{f}}}}\|f\supt(\cdot|s, \hat{a}) - \hat{f}(\cdot|s, \hat{a})\| - 2L\mathbb{E}_{a\sim U([0,\delta]^n)}\sqrt{n}\delta \\
  & = \mathbb{E}_{s\sim d_{\pi_{\hat{f}}}}\|f\supt(\cdot|s, \hat{a}) - \hat{f}(\cdot|s, \hat{a})\| - 2L\sqrt{n}\delta,
\end{align*}
Combining with above leads to
\begin{align*}
   &\mathbb{E}_{s\sim d_{\pi_{\hat{f}}}} \| f\supt(\cdot|s, \hat{a}) - \hat{f}(\cdot|s,\hat{a})\| \leq 2\sqrt{n}L\delta + O(1/(\delta^n\sqrt{T})) = O(\sqrt{n}T^{-1/(2n+2)})
\end{align*}
where in the last step we set $\delta = \Theta(({\frac{T}{n}})^{\frac{-1}{2(n+1)}})$. Finally we have:
\begin{align*}
    &\mathbb{E}_{s\sim d_{\pi_{\hat{f}}}} \| f\supt(\cdot|s, \pi_{\hat{f}}(s)) - f\sups(\cdot|s,\pi\sups(s))\| \leq  O(\sqrt{n}T^{-1/(2n+2)}) + \epsilon.
\end{align*}

\end{proof}

\section{Detailed description of our experiments} \label{appendix:exp_des}

\subsection{Environment descriptions.} \label{appendix:env_des}
For each of the four environments (HalfCheetah, Ant, Reacher) we tested on, we include a detailed numerical description of the environments in table \ref{env}.

\subsection{Descriptions of the perturbations}
\label{appendix:pert_des}
\subsubsection{Changing gravity.}
We change the gravity in the whole Mujoco locomotion task by a max range from $50\%$ to $200 \%$ of the normal gravity. The normal gravity is set to $0$ on $x$ and $y$ axis and $-9.81$ on $z$ axis.

\subsubsection{Changing mass.} We change the mass of the agent in the Mujoco locomotion task by a max range from $50\%$ to $200 \%$ of its original mass. Note that most agents are composed of links with independent masses, so besides changing the agent's mass as a whole, we also design experiments that change the mass of each individual links of the agent respectively.

\subsubsection{Changing plane orientation.} In the Reacher task, we tilt the platform of the Reacher so that it forms an angle of 45 degree of the original plane. 
\subsubsection{Changing arm length.} In the Reacher task, we change the first link of the Reacher arm (which is composed of two parts) into one tenth of its original length. 
\subsubsection{Changing Friction.} We change the frictional coefficient in the environment on an uniform scale. We found if friction is the only change in the target domain, then the adaptation task is relatively simple, so we incorporate it as one of the changes in the Multi-dimensional perturbation task.
\subsubsection{Motor noise.} This task tries to mimic the motor malfunction or inaccuracy in the real world. After the agent outputs an action, we add on a noise from a normal distribution with mean 0 and a fixed standard deviation from 0.2 to 1.
\subsubsection{Multiple DOFs of changes.} \label{app-multi-design} In the 3DOF task, we set the gravity to 0.8, mass to 1.2 and friction coefficient to 0.9. In the 15DOF task, we uniformly sample a coefficient from 0.9 to 1.1 for each of the following configuration: gravity, friction coefficient and the mass of each joint of the agent. We record these changes and apply for all comparing methods.  
\begin{figure}[H]
	\centering
	\includegraphics[width = 0.23\textwidth, trim=0cm 0cm 0cm 0cm, clip=true]{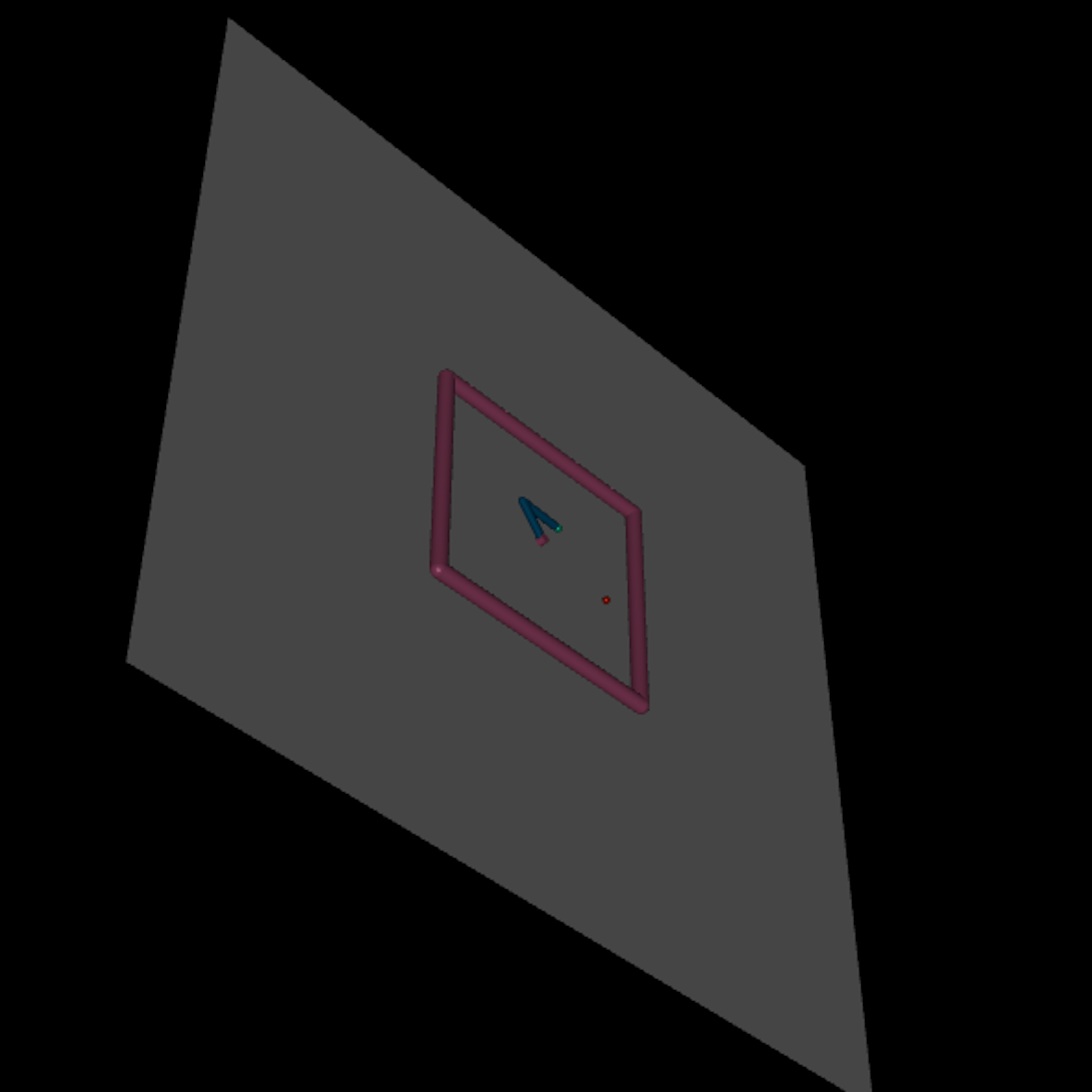}
	\includegraphics[width = 0.23\textwidth, trim=0cm 0cm 0cm 0cm, clip=true]{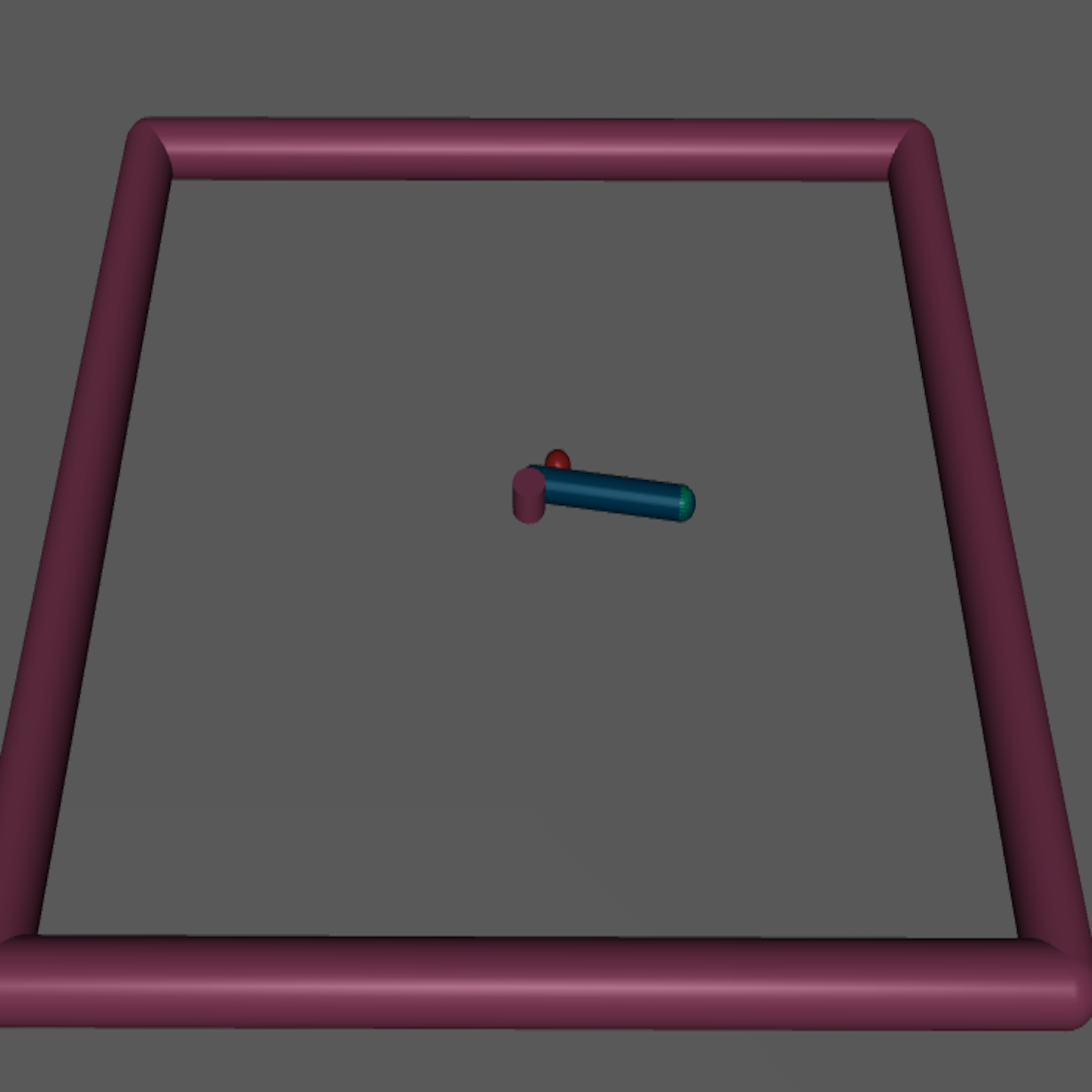}
	\caption{Visual illustration of modified Reacher environment. \textbf{left:} 45 degrees of tilted plane. \textbf{right:} Reacher with 10\% length of its first arm.}
\end{figure}

\begin{table*}[t]
\centering
\begin{tabular}{|c|cccc|}
\hline
\begin{tabular}[c]{@{}c@{}}Environment \\ name\end{tabular} & \begin{tabular}[c]{@{}c@{}}Full state \\ space size\end{tabular} & \begin{tabular}[c]{@{}c@{}}Model agnostic \\ state size\tablefootnote{This means the number of states that won't be passed as inputs to models or policies, e.g., the current coordinate or location of the agent.} \end{tabular} & \begin{tabular}[c]{@{}c@{}}Action space\\ size\end{tabular} & \begin{tabular}[c]{@{}c@{}}Reward \\ function\end{tabular}                                                                                \\ \hline
HalfCheetah & 18  & 1 & 6  & forward reward - 0.1 $\times$ control cost  \\
Ant         & 29 & 2 & 8  & 
\begin{tabular}[c]{@{}c@{}}  forward reward - 0.5 $\times$ control cost - \\ 
0.0005 $\times$ contact cost + survive reward\end{tabular} \\

Reacher & 16 & 4 & 2 & forward distance - control cost                                              \\ \hline
\end{tabular}
\caption{Description of the OpenAI gym environments. Note that to enforce safety of the agent in the target environment, we make a modification on HalfCheetah environment such that the episode will be terminated when the cheetah falls off.}
\label{env}
\end{table*}

\subsection{Hyperparameters} \label{app_hyper}
\begin{table}[h]
\begin{tabular}{|p{0.15\linewidth}|p{0.1\linewidth}|p{0.18\linewidth}|p{0.18\linewidth}|p{0.15\linewidth}|p{0.1\linewidth}|l|}
\hline
                          & source               & PADA-DM              & PADA-DM with target  & Christiano et al., 2016 & Zhu et al., 2018     & PPO                  \\ \hline
\# timesteps              & 2e6                  & 8e4   (5e4,8e4,12e4,15e4)               & 8e4     (5e4,8e4,1.2e5,1.5e5)             & 2e5  (1e5,2e5,4e5)                    & 2e6                  & 2e6                  \\ \hline
learning rate 
(with linear decay)       & 7e-4                 & 5e-3                 & 5e-3                 & 5e-3                    & 7e-4                 & 7e-4                 \\ \hline
soft update rate          &                      &                      & every 3000 timesteps (3000,5000,10000) &                &                      &                      \\ \hline
explore rate $\epsilon$   &                      & 0.01 (0.01,0.02)                & 0.01 (0.01,0.02)                 & \textbf{}               &                      &                      \\ \hline
reward tradeoff $\lambda$ &                      &                      &                      & \textbf{}               & 0.5                  &                      \\ \hline
\end{tabular}
\caption{Final hyperparameters we use for our methods and baselines in the experiments. The values in the brackets are the value we considered.}
\end{table}

\section{Implementation details} \label{app:prac_alg}
\subsection{Pretraining of the source dynamics model} \label{app:pretrain}
In section \ref{model_param}, one assumption we make is that we have a pretrained model $\hat{f}\sups$ that well approximates the source dynamics $f\sups$. In practice, we pretrain the model $\hat{f}\sups$ with the $(s,a,s')$ triplets generated during the training of $\pi\sups$. The model $\hat{f}\sups$ is a two-layer neural network with hidden size 128 and ReLU nonlinearity, trained with MSE loss since we assume deterministic transitions. Using these existing samples has two major advantages: first is that we don't need further interaction with the source environment and second is that the trained model $\hat{f}\sups$ especially well approximates the transitions around the actions taken by $\pi\sups$, which is important to our algorithm. 

Remark that if we already have the ground truth source dynamics $f\sups$, which is a mild assumption while using a simulator, we can also directly use $f\sups$ to replace $\hat{f}\sups$. During our experiments, we observe that whether using $f\sups$ or $\hat{f}\sups$ won't affect the performance of our method.

\subsection{Cross Entropy Method} \label{app:cem}
Here we provide a pseudocode of the Cross Entropy Method that we used in our method, as in Alg. \ref{app:alg:cem}. In our implementation, we use $T=10$ iterations, $N=200$ actions sampled per iteration, $K = 40$ elites (elite ratio 0.4) and $\sigma_0 = 0.5$. We use the output of target policy as the initial mean, and when we don't have the target policy, we use $\pi\sups(s)$ as the initial mean. To avoid invalid actions, for each action $a_i$ we sample, we clip the action if certain dimension is out of the bound of $\calA\supt$ (usually [-1,1] on each dimension).

\begin{algorithm}[t] 
\caption{Cross Entropy Method}
\label{app:alg:cem}
\begin{algorithmic}[1]
\REQUIRE Initial mean $\mu_0$, initial standard deviation $\sigma_0$, action space $\calA\supt$, current model $\delta_\theta$, current state $s$, number of iterations T, sample size N, elite size K.
\STATE $\Sigma_0 \gets  I_{|\calA\supt|}(\sigma_0^2)$
\FOR{t = 1, \ldots {\text{T}}}
\STATE Sample $\{a_i\}_{i=1}^{N} \sim \mathcal{N}(\mu_{t-1},\Sigma_{t-1})$
\STATE $\{a_i\}_{i=1}^{N} \gets clip(a_i, \min(\calA\supt), \max(\calA\supt))$ 
\STATE Sort $\{a_i\}$ with descending $\|\delta_{\theta}(s,a_i)\|_2^2$ and pick the first K actions $\{a_j\}_{j=1}^K$

\STATE $\mu_t \gets \frac{1}{K}\Sigma_{j=1}^{K}a_j$ 
\STATE$\Sigma_t \gets \frac{1}{K}\Sigma_{j=1}^{K} (a_j - \mu_t)(a_j - \mu_t)^T$
\ENDFOR 
\STATE \textbf{Output}: $\mu_T$
\end{algorithmic}
\end{algorithm}

\section{Supplemental experiments} \label{appendix:sup_exp}
\subsection{Accuracy of Deviation Model} \label{appdendix:exp:acc}
We evaluate the performance of the deviation model by comparing its output with the actual deviation (i.e., $\Delta^{\pi\sups}$) in the target and source environment states. We compare the performance between linear and nonlinear deviation models. We include linear models as they are convenient if we want to use optimal control algorithms. The nonlinear model is the same model we use for our PADA-DM algorithm, which has two 128-neuron layers with ReLU \cite{nair2010rectified} nonlinearity. Both of the deviation models are tested on the same initial state distribution after training on 80k samples. In \ref{accuracy}(left), we plot the output of the deviation model against the ground truth deviation. We test on 50 trajectories and each data point refers to the average L2 state distance along one trajectory.  In \ref{accuracy}(right), we plot the ground truth deviation, and the outputs of the linear and nonlinear deviation models over time, on 10 test trajectories.

\begin{figure}[H]
	\centering
	\includegraphics[width = 0.35\textwidth, trim=0cm 0cm 0cm 0cm, clip=true]{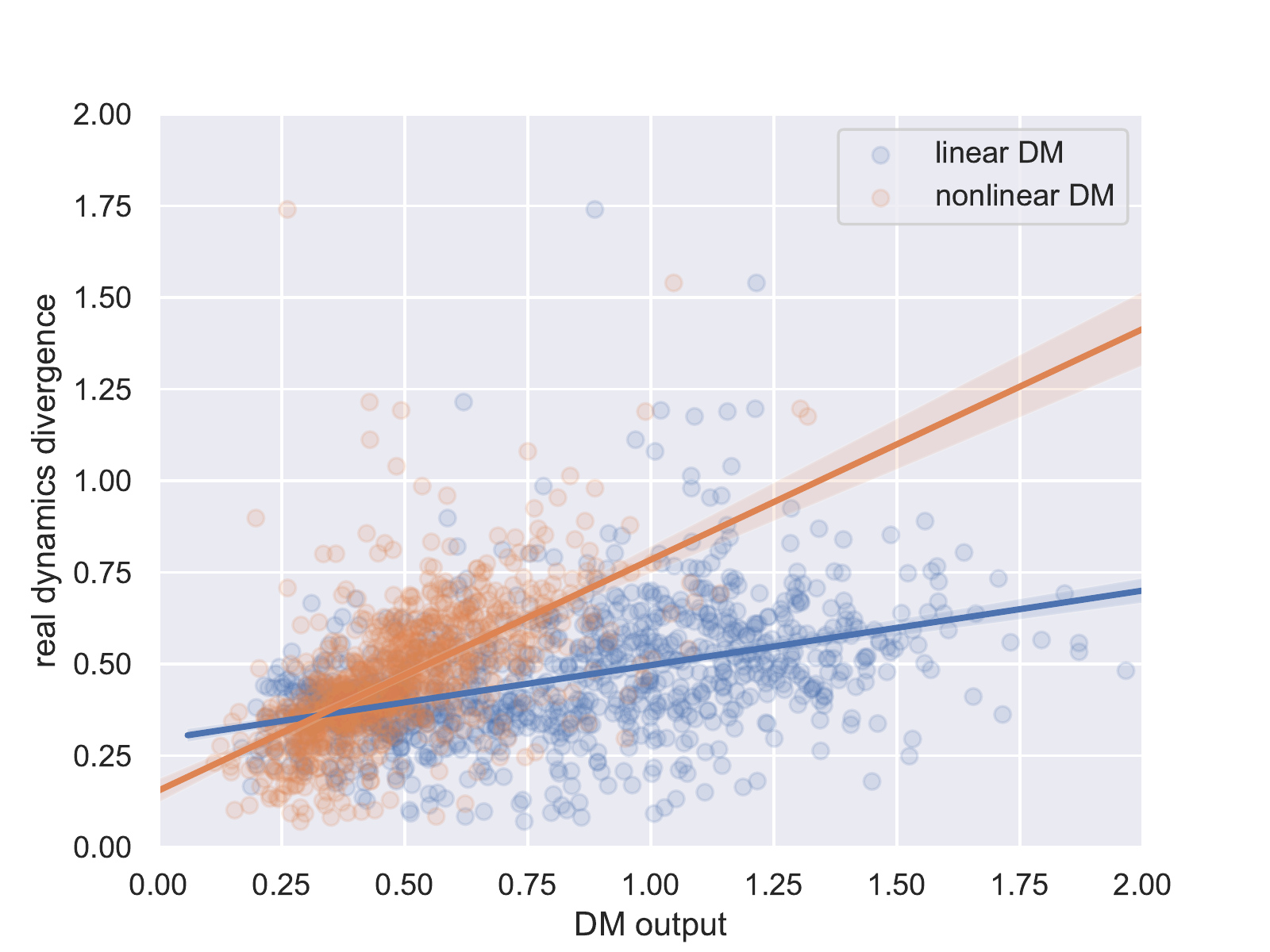}
	\includegraphics[width = 0.35\textwidth, trim=0cm 0cm 0cm 0cm, clip=true]{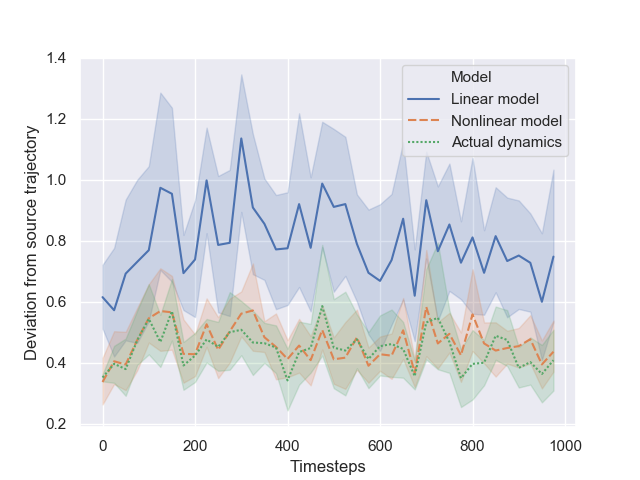}
	\caption{Comparing the performances of the linear and nonlinear deviation models. \textbf{left:} This plot depicts the correlation between the predicted deviation and the actual deviation. The nonlinear deviation model is more accurate since its slope is closer to 1. \textbf{right:} This plot shows the predicted and actual deviation over the course of 10 trajectories. Here again, the nonlinear model (orange) curve lies very close to the actual deviation curve (green).}
	\label{accuracy}
\end{figure}

\subsection{Long-term learning curves}\label{app_exp_long}
In this section we show a more comprehensive long-term learning curve in Fig. \ref{app_exp_long_plot}. Each task here corresponds to the task in Fig. \ref{performance}. Note that here again the x-axis is in natural logarithm scale. 
\begin{figure}[H]
    \centering
    \includegraphics[width = 0.245\textwidth, trim=0cm 0cm 0cm 0cm, clip=true]{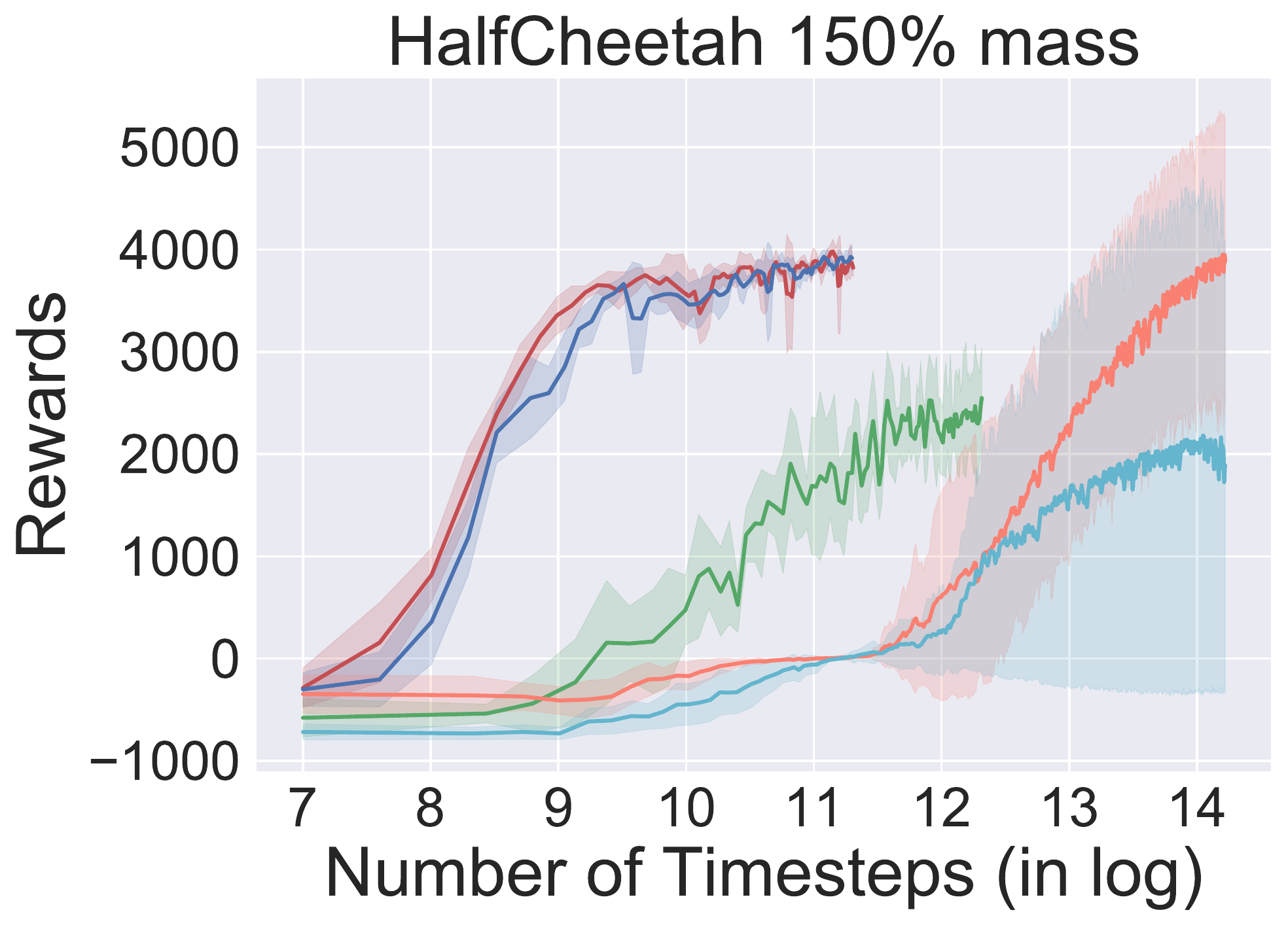}
	\includegraphics[width = 0.245\textwidth, trim=0cm 0cm 0cm 0cm, clip=true]{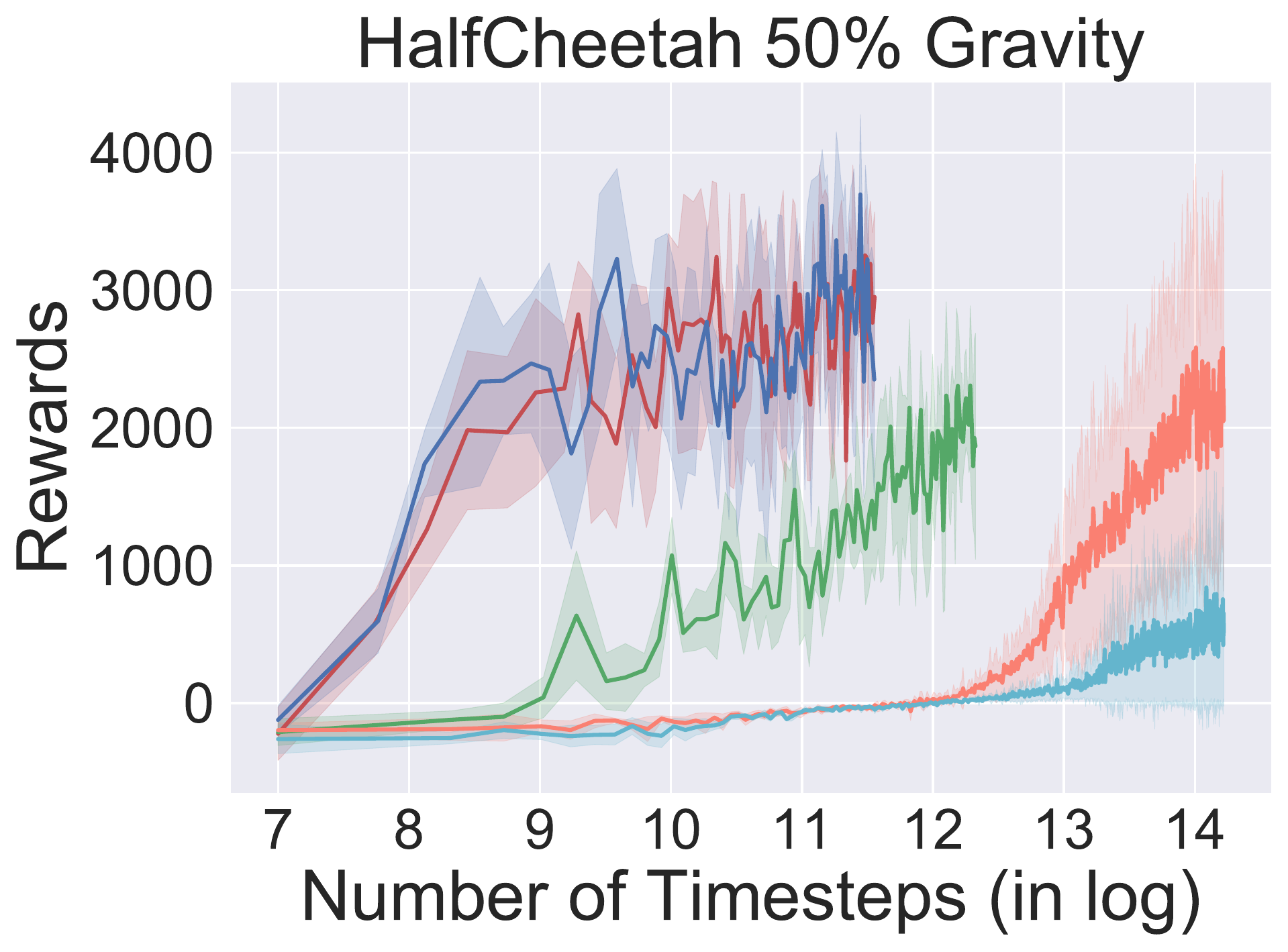}
	\includegraphics[width = 0.245\textwidth, trim=0cm 0cm 0cm 0cm, clip=true]{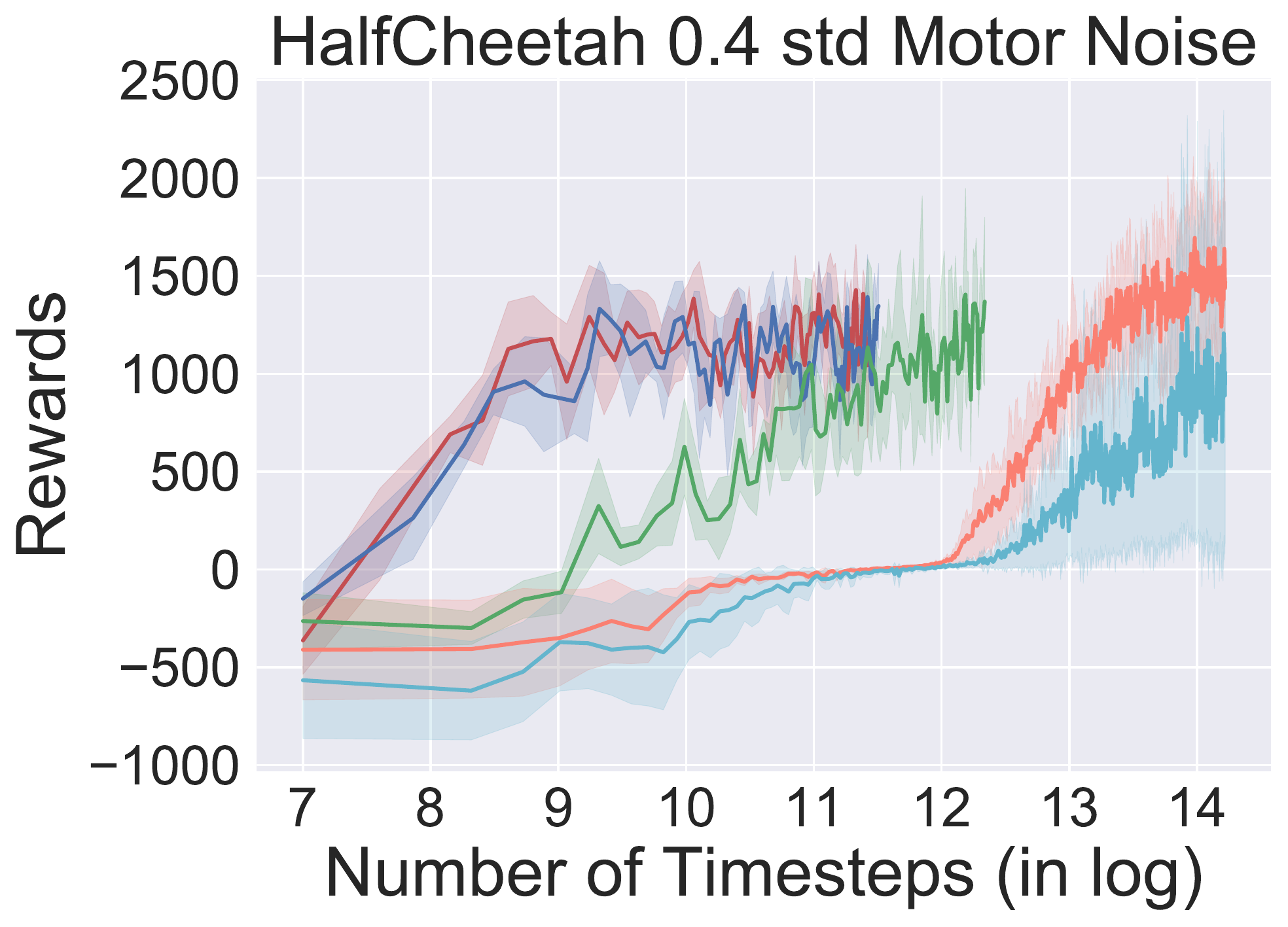}
	\includegraphics[width = 0.245\textwidth, trim=0cm 0cm 0cm 0cm, clip=true]{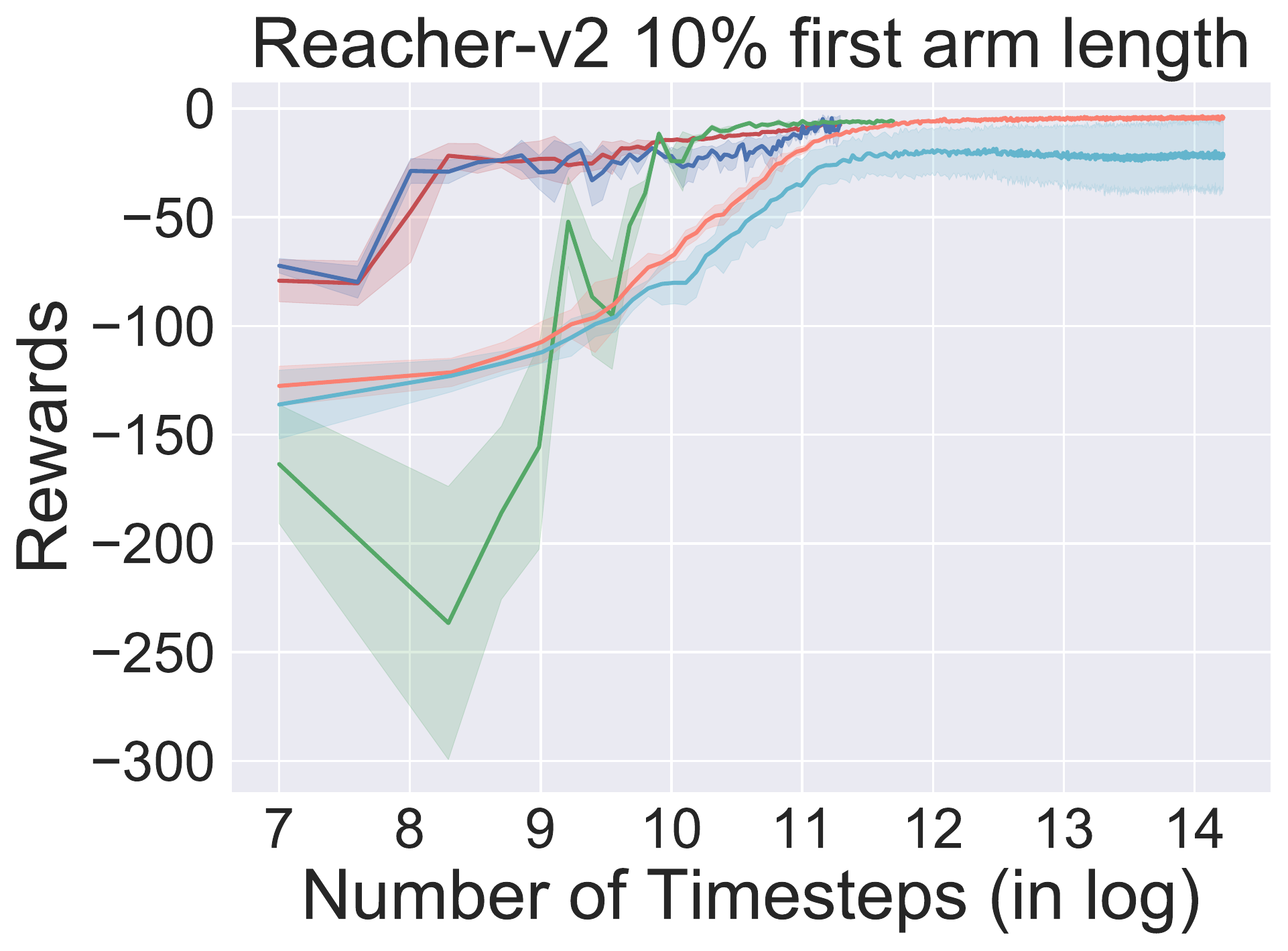}

	\centering
	\includegraphics[width = 0.245\textwidth, trim=0cm 0cm 0cm 0cm, clip=true]{results/cem/AntDM-v2_5_None_long.pdf}
	\includegraphics[width = 0.245\textwidth, trim=0cm 0cm 0cm 0cm, clip=true]{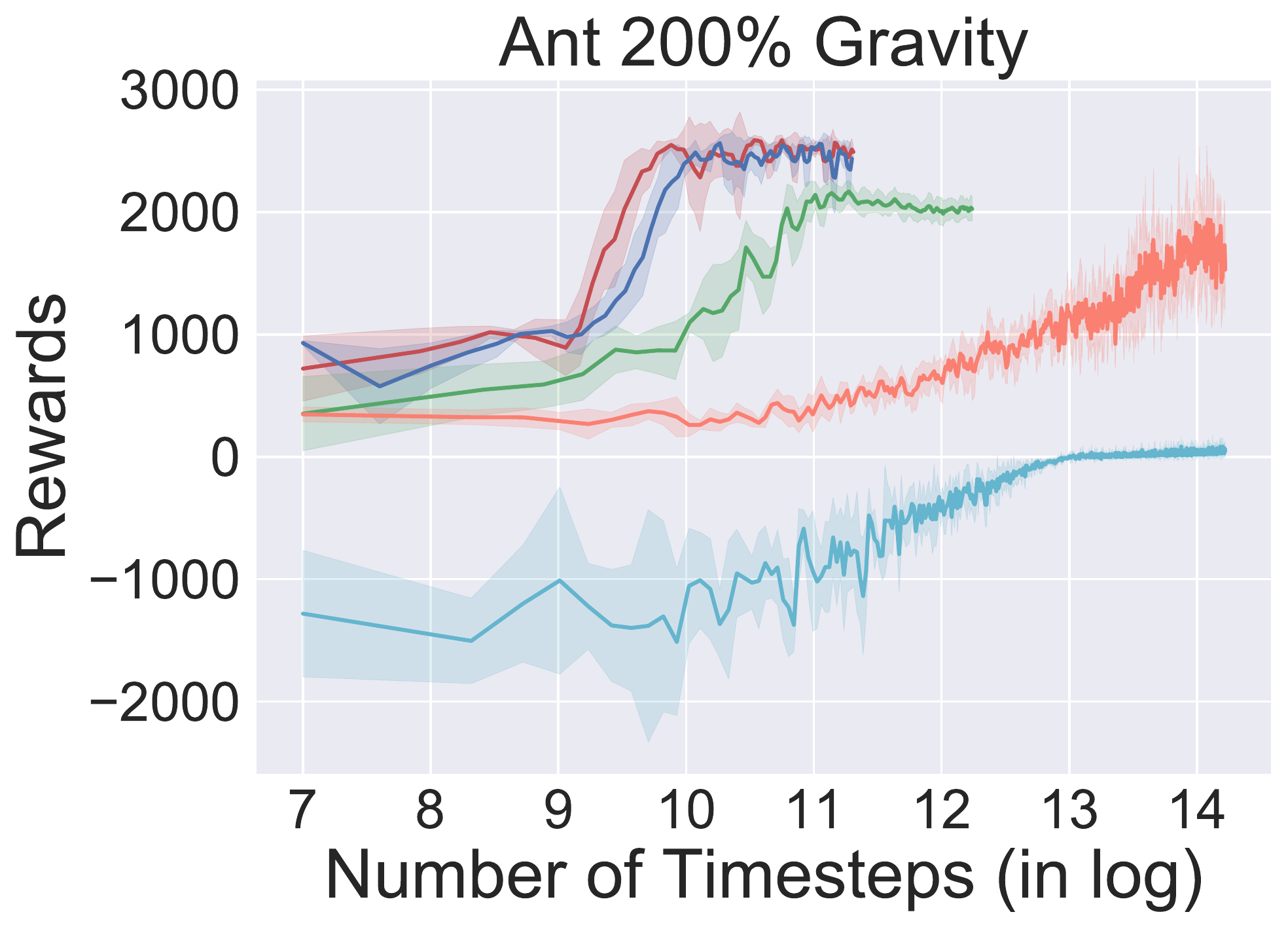}
	\includegraphics[width = 0.245\textwidth, trim=0cm 0cm 0cm 0cm, clip=true]{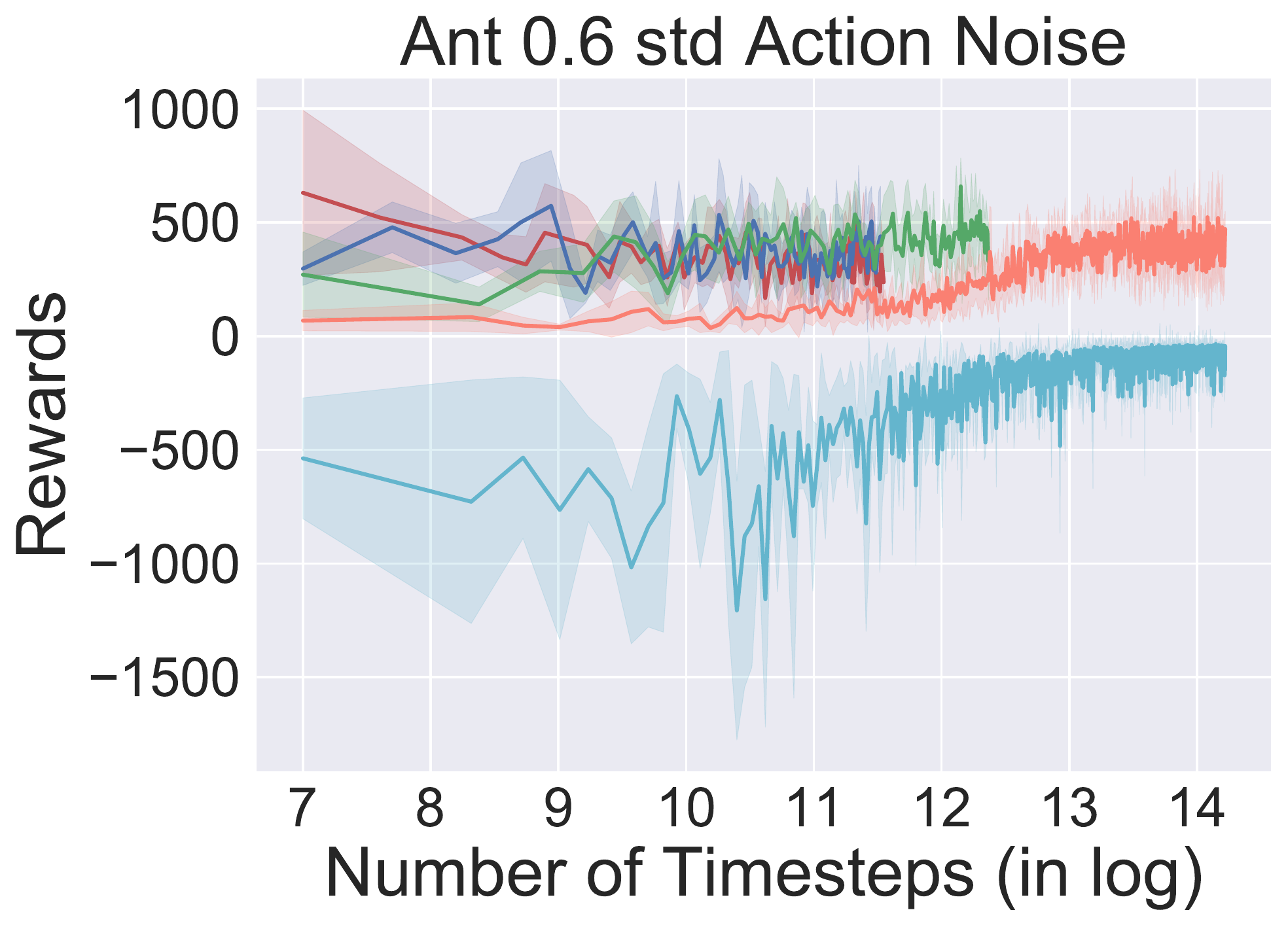}
	\includegraphics[width = 0.245\textwidth, trim=0cm 0cm 0cm 0cm, clip=true]{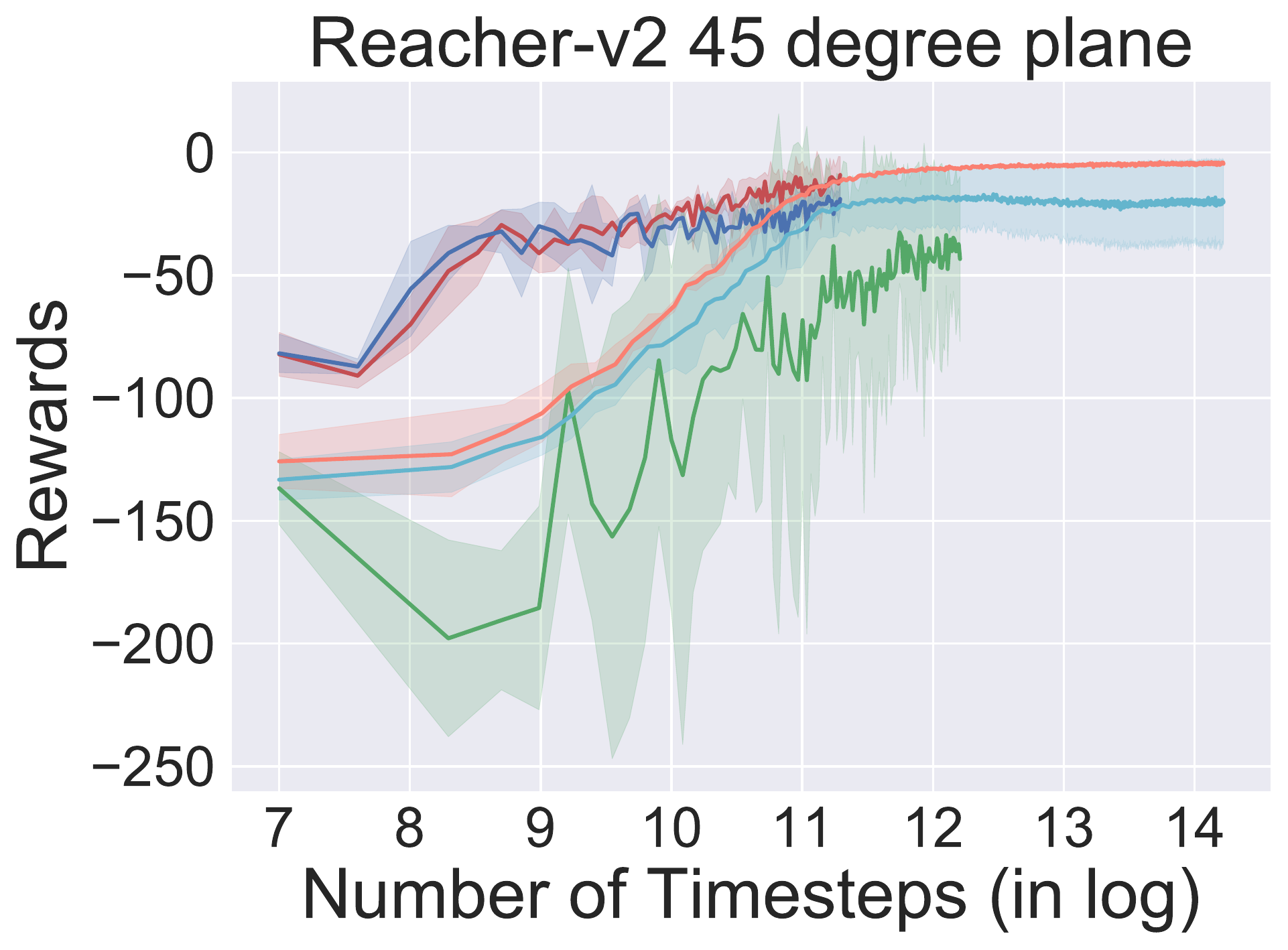}
	
	\centering
	\includegraphics[width = 1\textwidth, trim=0cm 0cm 0cm 0cm, clip=true]{results/lgd.pdf}
  \caption{We plot the learning curves across 5 random seeds of our adaptation method (using PADA-DM and PADA-DM with target), and the baseline methods on a number of tasks including perturbation of the mass, gravity, motor noise for the locomotion environments (HalfCheetah and Ant), and the plane angle and arm lengths for the navigation environment (Reacher). The title of each plot corresponds to the perturbation in the target domain, e.g., HalfCheetah Mass 150\% means in mass of the agent in the target domain is 150\% of that in the source domain. The shaded area denotes 1 standard deviation range from the mean.
  }
	\label{app_exp_long_plot}
\end{figure}

\subsection{Comparing using true reward}\label{app_rew}

To further verify the efficiency of our choice of reward (the deviations between two environments), we conduct an additional experiments where we use the ground truth reward for CEM. We fix all the hyperparameters and train an additional model to learn to ground truth reward. To ensure the fairness of the comparison, when we use the ground truth reward, we keep doing 1 step look-ahead during CEM. The results verify that the deviation serves as a better reward for conducting policy adaptation, where the ground truth reward leads to a local minimum that results in suboptimal performance.
\begin{figure}[h]
	\centering
	\includegraphics[width = 0.35\textwidth, trim=0cm 0cm 0cm 0cm, clip=true]{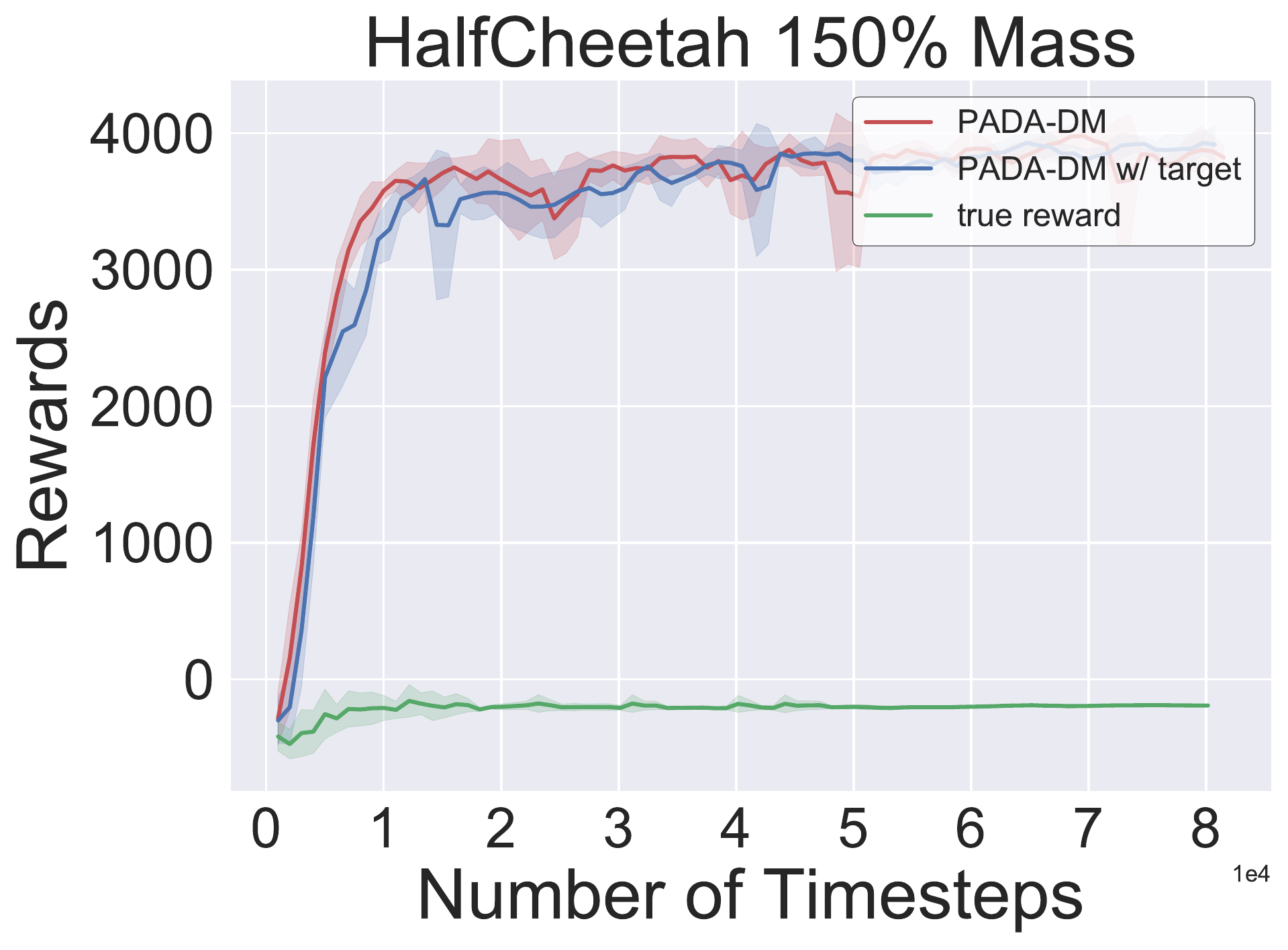}
	\caption{Comparing learning ground truth reward, learning deviations quickly adapts the policy in the target domain, where using ground truth reward may not necessarily leads to optimal performance in the target domain.}
	\label{true_rew}
\end{figure}

\subsection{Additional Experiments for Meta-Learning MAML}\label{app_exp_dr}
In this section we conduct an additional experiment to Section \ref{exp:drmaml}. In section \ref{exp:drmaml}, we use 80k samples of adaptation for our method and MAML to conduct a fair comparison. However, using so many samples for adaptation contradicts MAML's claim of few-shot adaptation and we also  observe that MAML’s test performance does not improve too much as we change the sample size in adaptation phase. Thus here we report the additional experiments to support this claim: we adopt the same experiment setting in Section \ref{exp:drmaml}, and this time we use 20k samples for MAML during adaptation. The test performance is recorded in Fig. ~\ref{app_domainrand}(a) and Fig.~\ref{app_domainrand_80}(a). Comparing with the original performance (Fig.~\ref{app_domainrand}(b) and Fig. ~\ref{app_domainrand_80}(b)), the test performance of MAML does not change that much as the number of adaptation samples decreases and our approach still outperforms MAML consistently. 

In addition, we record the mean and the standard deviation of the test performance of each method to deliver a more direct comparison in Table \ref{app_gravity_table} and Table \ref{app_mass_table}. As we can see, our approach outperforms other baselines most of the time. When the perturbation is small (e.g., the 120\% columns in both tables), DR also delivers very strong performances. However when perturbation is large (e.g., 200\% columns in both tables), DR fails to adapt completely, which indicates that DR has troubles to adapt to out-of-distribution new environments.

\begin{figure}[H]
\centering
	\begin{tabular}{@{}l@{}}
	\subfigure[]{\includegraphics[width = 0.4\textwidth]{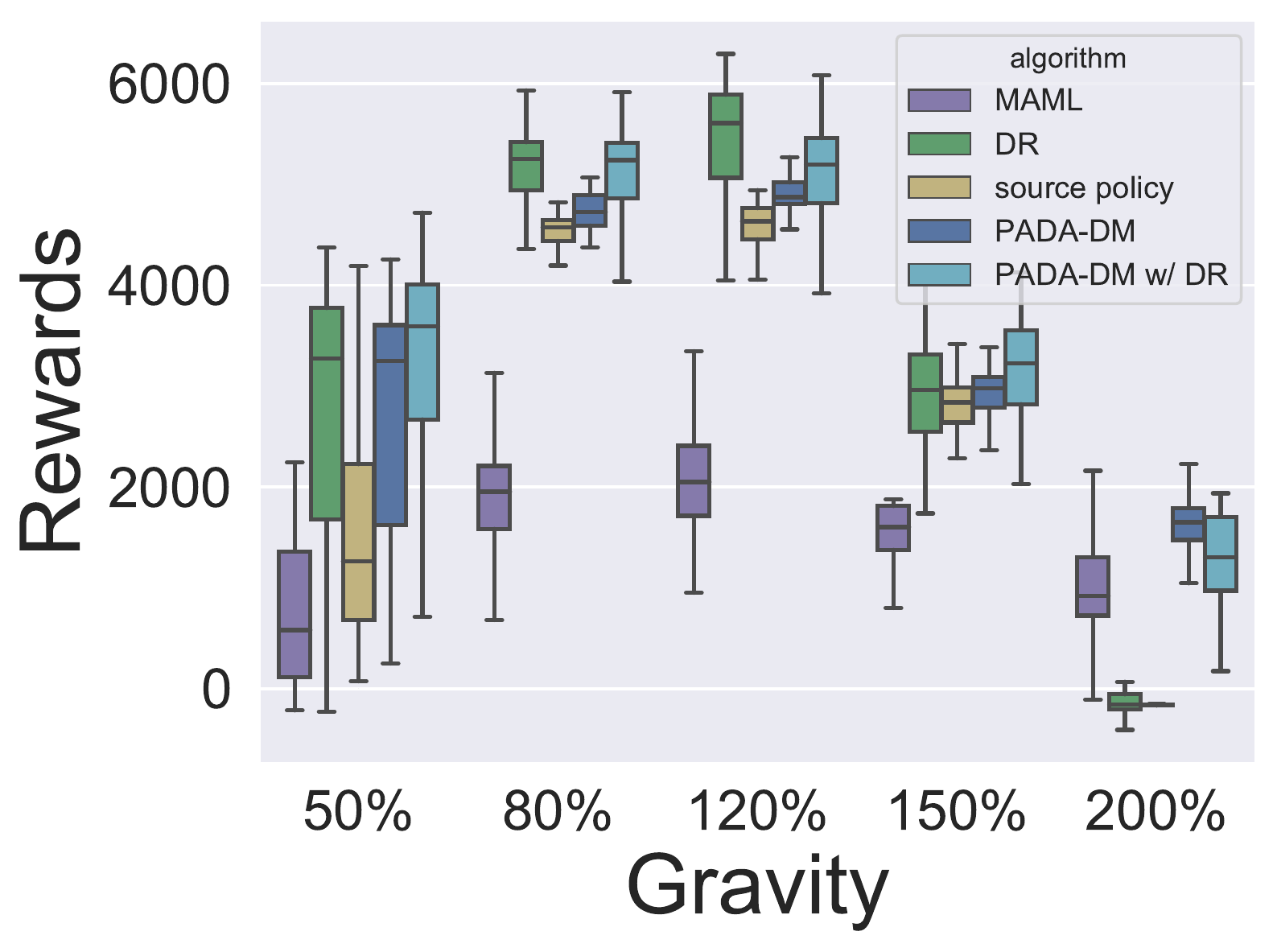} }
	\subfigure[]{\includegraphics[width = 0.4\textwidth]{results/cem/cheetah_gravity_bar.pdf}}
    \end{tabular}
	\caption{Ablation experiments using domain randomization and meta-learning. (a) MAML with 20k adaptation samples. (b) MAML with 80k adaptation samples. The boxplots show the median of the data while more statistics such as mean and standard deviation are shown in the following tables.}
	\label{app_domainrand}
\end{figure}
\begin{figure}[H]
\centering
	\begin{tabular}{@{}l@{}}
	\subfigure[]{\includegraphics[width = 0.4\textwidth]{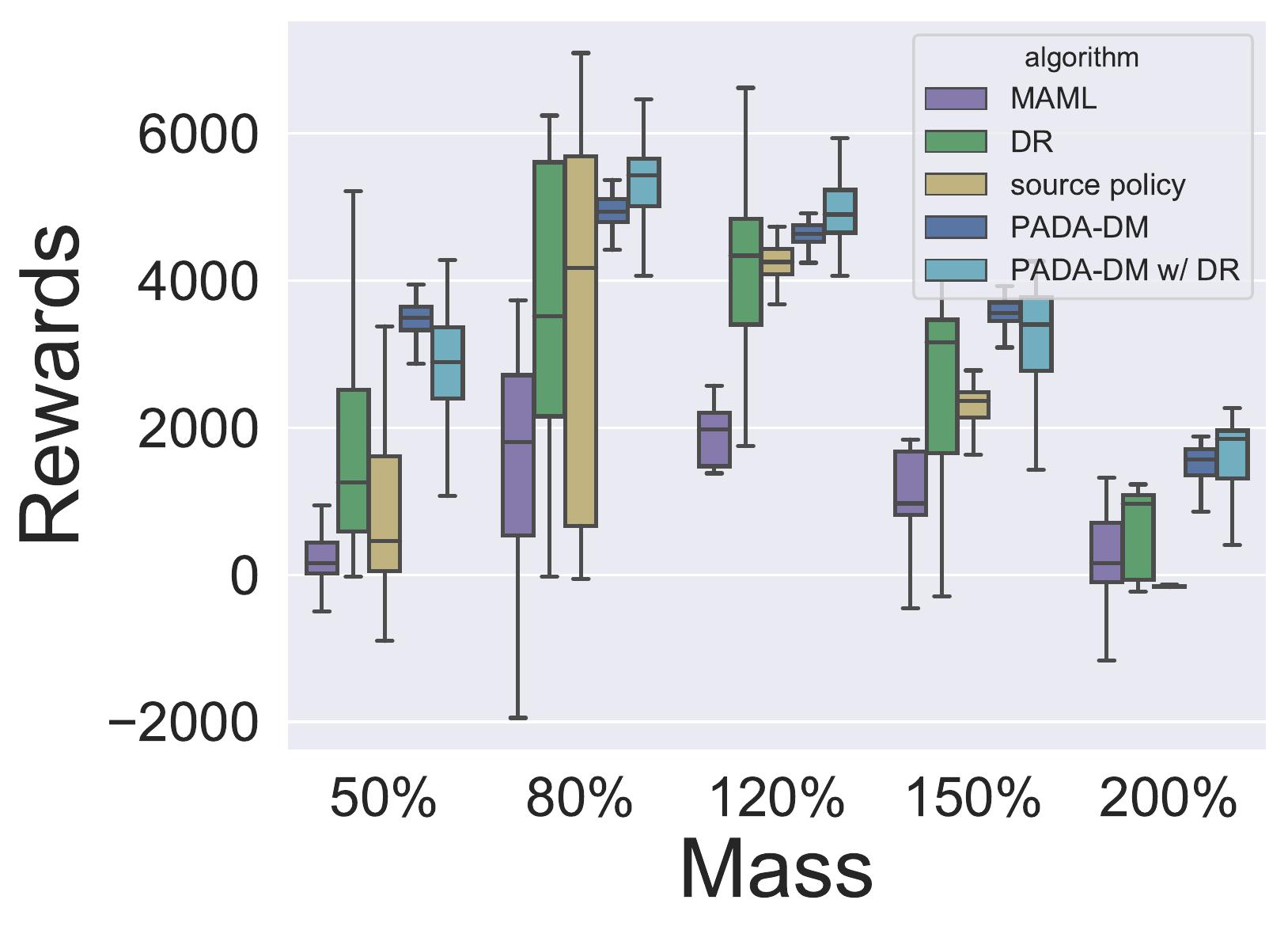} }
	\subfigure[]{\includegraphics[width = 0.4\textwidth]{results/cem/cheetah_mass_80_bar.pdf}}
    \end{tabular}
	\caption{Ablation experiments using domain randomization and meta-learning. (a) MAML with 20k adaptation samples. (b) MAML with 80k adaptation samples.}
	\label{app_domainrand_80}
\end{figure}

\begin{table}[H]
\centering
\begin{tabular}{|l|c|c|c|c|c|}
\hline
 &
  \multicolumn{5}{c|}{Gravity Perturbation} \\ \cline{2-6} 
 &
  50\% &
  80\% &
  120\% &
  150\% &
  200\% \\ \hline
Source policy &
  \begin{tabular}[c]{@{}c@{}}1549.74 \\ (1090.73)\end{tabular} &
  \begin{tabular}[c]{@{}c@{}}4444.86 \\ (637.77)\end{tabular} &
  \begin{tabular}[c]{@{}c@{}}4592.18 \\ (201.30)\end{tabular} &
  \begin{tabular}[c]{@{}c@{}}2543.55\\ (916.16)\end{tabular} &
  \begin{tabular}[c]{@{}c@{}}-156.51 \\ (25.93)\end{tabular} \\ \hline
Domain Randomization &
  \begin{tabular}[c]{@{}c@{}}2282.74 \\ (1563.70)\end{tabular} &
  \begin{tabular}[c]{@{}c@{}}4838.87 \\ (1134.98)\end{tabular} &
  \textbf{\begin{tabular}[c]{@{}c@{}}5236.46\\ (1179.85)\end{tabular}} &
  \begin{tabular}[c]{@{}c@{}}2896.03 \\ (554.00)\end{tabular} &
  \begin{tabular}[c]{@{}c@{}}-43.97 \\ (423.47)\end{tabular} \\ \hline
PADA-DM &
  \begin{tabular}[c]{@{}c@{}}2694.78 \\ (1166.88)\end{tabular} &
  \begin{tabular}[c]{@{}c@{}}4739.32\\ (279.06)\end{tabular} &
  \begin{tabular}[c]{@{}c@{}}4889.02\\ （164.38)\end{tabular} &
  \begin{tabular}[c]{@{}c@{}}2998.32\\ (266.75)\end{tabular} &
  \textbf{\begin{tabular}[c]{@{}c@{}}1531.23 \\ (400.96)\end{tabular}} \\ \hline
PADA-DM w/ DR &
  \textbf{\begin{tabular}[c]{@{}c@{}}3230.29 \\ (1280.54)\end{tabular}} &
  \textbf{\begin{tabular}[c]{@{}c@{}}5036.59\\ (657.98)\end{tabular}} &
  \begin{tabular}[c]{@{}c@{}}4934.04 \\ (720.34)\end{tabular} &
  \textbf{\begin{tabular}[c]{@{}c@{}}3200.73\\ (521.50)\end{tabular}} &
  \begin{tabular}[c]{@{}c@{}}1431.53\\ (496.11)\end{tabular} \\ \hline
MAML (20k) &
  \begin{tabular}[c]{@{}c@{}}854.24\\ (692.50)\end{tabular} &
  \begin{tabular}[c]{@{}c@{}}1810.51\\ (663.72)\end{tabular} &
  \begin{tabular}[c]{@{}c@{}}1895.86\\ (650.76)\end{tabular} &
  \begin{tabular}[c]{@{}c@{}}1575.06\\ (653.09)\end{tabular} &
  \begin{tabular}[c]{@{}c@{}}831.07\\ (717.78)\end{tabular} \\
   \hline
MAML (80k) &
  \begin{tabular}[c]{@{}c@{}}876.37\\ (711.98)\end{tabular} &
  \begin{tabular}[c]{@{}c@{}}1778.67\\ (669.86)\end{tabular} &
  \begin{tabular}[c]{@{}c@{}}1894.28\\ (644.55)\end{tabular} &
  \begin{tabular}[c]{@{}c@{}}1568.60\\ (646.79)\end{tabular} &
  \begin{tabular}[c]{@{}c@{}}823.13\\ (721.15)\end{tabular} \\
   \hline
\end{tabular}
\caption{Mean and standard deviation (in the brackets) of the episodic rewards of each method in the target environment with perturbed gravity across 100 trajectories of 5 random seeds (500 trajectories in total).}
\label{app_gravity_table}
\end{table}

\begin{table}[H]
\centering
\begin{tabular}{|l|c|c|c|c|c|}
\hline
& \multicolumn{5}{c|}{Mass Perturbation}            \\ \cline{2-6} 
& 50\%
& 80\%
& 120\%
& 150\%
& 200\%
\\ \hline
Source policy &
  \begin{tabular}[c]{@{}c@{}}921.13 \\ (1192.43)\end{tabular} &
  \begin{tabular}[c]{@{}c@{}}3343.05 \\ (2317.32)\end{tabular} &
  \begin{tabular}[c]{@{}c@{}}4166.10 \\ (494.94)\end{tabular} &
  \begin{tabular}[c]{@{}c@{}}2045.26\\ (665.30)\end{tabular} &
  \begin{tabular}[c]{@{}c@{}}-149.92 \\ (27.28)\end{tabular} \\ \hline
Domain Randomization &
  \begin{tabular}[c]{@{}c@{}}1665.05 \\ (1357.31)\end{tabular} &
  \begin{tabular}[c]{@{}c@{}}3823.45 \\ (1944.70)\end{tabular} &
  {\begin{tabular}[c]{@{}c@{}}3932.86\\ (1791.18)\end{tabular}} &
  \begin{tabular}[c]{@{}c@{}}2635.72 \\ (1105.15)\end{tabular} &
  \begin{tabular}[c]{@{}c@{}}944.50 \\ (1134.32)\end{tabular} \\ \hline
PADA-DM &
  \textbf{\begin{tabular}[c]{@{}c@{}}3271.52 \\ (752.62)\end{tabular}} &
  \begin{tabular}[c]{@{}c@{}}4914.67\\ (379.73)\end{tabular} &
  \begin{tabular}[c]{@{}c@{}}4584.95\\ （375.51)\end{tabular} &
  \textbf{\begin{tabular}[c]{@{}c@{}}3557.25\\ (183.46)\end{tabular}} &
  \begin{tabular}[c]{@{}c@{}}1398.88 \\ (500.09)\end{tabular} \\ \hline
PADA-DM w/ DR &
  \begin{tabular}[c]{@{}c@{}}2673.87 \\ (1009.75)\end{tabular} &
  \textbf{\begin{tabular}[c]{@{}c@{}}5348.98\\ (556.19)\end{tabular}} &
  \textbf{\begin{tabular}[c]{@{}c@{}}4854.30 \\ (591.50)\end{tabular}} &
  \begin{tabular}[c]{@{}c@{}}3276.70\\ (874.54)\end{tabular} &
  \textbf{\begin{tabular}[c]{@{}c@{}}1616.75\\ (490.53)\end{tabular}} \\ \hline
MAML (20k) &
  \begin{tabular}[c]{@{}c@{}}854.24\\ (692.50)\end{tabular} &
  \begin{tabular}[c]{@{}c@{}}1810.51\\ (663.72)\end{tabular} &
  \begin{tabular}[c]{@{}c@{}}1895.86\\ (650.76)\end{tabular} &
  \begin{tabular}[c]{@{}c@{}}1575.06\\ (653.09)\end{tabular} &
  \begin{tabular}[c]{@{}c@{}}831.07\\ (717.78)\end{tabular} \\
   \hline
MAML (80k) &
  \begin{tabular}[c]{@{}c@{}}876.37\\ (711.98)\end{tabular} &
  \begin{tabular}[c]{@{}c@{}}1778.67\\ (669.86)\end{tabular} &
  \begin{tabular}[c]{@{}c@{}}1894.28\\ (644.55)\end{tabular} &
  \begin{tabular}[c]{@{}c@{}}1568.60\\ (646.79)\end{tabular} &
  \begin{tabular}[c]{@{}c@{}}823.13\\ (721.15)\end{tabular} \\
   \hline
\end{tabular}

\caption{Mean and standard deviation (in the brackets) of the episodic rewards of each method in the target environment with perturbed mass across 100 trajectories of 5 random seeds (500 trajectories in total).
}
\label{app_mass_table}
\end{table}

\end{document}